\documentclass{article}

    \PassOptionsToPackage{numbers, compress}{natbib}

\usepackage[final]{neurips_2025}

\usepackage[utf8]{inputenc} 
\usepackage[T1]{fontenc}    
\usepackage{hyperref}       

\usepackage{url}            
\usepackage{booktabs}       
\usepackage{amsfonts}       
\usepackage{nicefrac}       
\usepackage{microtype}      
\usepackage{xcolor}         

\usepackage[utf8]{inputenc} 
\usepackage[T1]{fontenc}    
\usepackage{booktabs}       
\usepackage{amsfonts}       
\usepackage{nicefrac}       
\usepackage{microtype}      
\usepackage{xcolor}         
\usepackage{makecell}

\usepackage{mathtools, bm}
\usepackage{amssymb}
\usepackage{leftidx}
\usepackage{amsthm}

\usepackage{subcaption}

\usepackage{wrapfig}

\usepackage{tikz,pgfplots}
\usepackage{tikzscale}
\pgfplotsset{compat=1.18}
\usetikzlibrary{bayesnet}
\usetikzlibrary{shapes, arrows, calc, positioning, matrix}
\usepackage{tikz-cd}
\tikzcdset{scale cd/.style={every label/.append style={scale=#1},
    cells={nodes={scale=#1}}}}
\tikzstyle{do} = [obs, rectangle]

\usepackage{titletoc}

\usepackage{hyperref}       
\usepackage{url}            
\usepackage{xr}
\usepackage[capitalize]{cleveref}

\usepackage[page,header]{appendix}

\usepackage{graphicx}

\usepackage{bbm}

\usetikzlibrary{calc,intersections}

\usetikzlibrary{angles}

\usepackage{soul}

\tikzset{
    right angle quadrant/.code={
        \pgfmathsetmacro\quadranta{{1,1,-1,-1}[#1-1]}     
        \pgfmathsetmacro\quadrantb{{1,-1,-1,1}[#1-1]}},
    right angle quadrant=1, 
    right angle length/.code={\def\rightanglelength{#1}},   
    right angle length=2ex, 
    right angle symbol/.style n args={3}{
        insert path={
            let \p0 = ($(#1)!(#3)!(#2)$),     
                \p1 = ($(\p0)!\quadranta*\rightanglelength!(#3)$), 
                \p2 = ($(\p0)!\quadrantb*\rightanglelength!(#2)$), 
                \p3 = ($(\p1)+(\p2)-(\p0)$) in  
            (\p1) -- (\p3) -- (\p2)
        }
    }
}

\usepackage{thmtools, thm-restate}

\usepackage{etoolbox}

\usetikzlibrary{tikzmark, arrows.meta}

\usepackage{xparse}

\usepackage{adjustbox}

\usepackage{extarrows}

\usepackage{nomencl}

\usepackage{multirow}

\usepackage[thinc]{esdiff}

\usepackage{mathrsfs}

\declaretheorem[style=plain]{lemma}
\declaretheorem[style=plain]{proposition}

\declaretheorem[style=definition]{observation}

\declaretheorem[style=definition]{example}

\def \ProofSep {0\topsep}

\makeatletter
\newcommand*{\centerfloat}{%
  \parindent \z@
  \leftskip \z@ \@plus 1fil \@minus \textwidth
  \rightskip\leftskip
  \parfillskip \z@skip}
\makeatother

\crefname{section}{Sec.}{Secs.}
\Crefname{section}{Section}{Sections}
\Crefname{table}{Table}{Tables}
\crefname{table}{Tab.}{Tabs.}

\newcommand{\rnum}[1]{\expandafter{\romannumeral #1\relax}}
\newcommand{\RNum}[1]{\uppercase\expandafter{\romannumeral #1\relax}}

\definecolor{deep_blue}{HTML}{4c72b0}
\definecolor{deep_orange}{HTML}{dd8452}
\definecolor{deep_green}{HTML}{55a868}
\definecolor{deep_red}{HTML}{c44e52}
\definecolor{deep_purple}{HTML}{8172b3}
\definecolor{deep_brown}{HTML}{937860}
\definecolor{deep_pink}{HTML}{da8bc3}
\definecolor{deep_grey}{HTML}{8c8c8c}

\hypersetup{
    colorlinks=true,
    breaklinks=true,
    linkcolor=red,
    urlcolor=magenta,
    citecolor=green,
}

\dottedcontents{section}[2.3em]{}{2.3em}{5pt}

\makeatletter
\newcommand{\stored@titlecmd}{\@latex@warning@no@line{Title command not stored}}
\newcommand{\stored@authorcmd}{\@latex@warning@no@line{Author command not stored}}
\AfterEndPreamble{%
  \global\let\stored@titlecmd\@title
  \global\let\stored@authorcmd\@author
}

\newcommand{\appendixmaketitle}[1]{%
  {%
  \par 
  \begingroup 
    \let\thanks\@gobble 

    \thispagestyle{plain} 
    \vbox{%
      \hsize\textwidth
      \linewidth\hsize
      \vskip 0.1in 
      \@toptitlebar 
      \centering 
      {\Large\bf #1\par} 
      \@bottomtitlebar 

      \if@submission 
        \begin{tabular}[t]{c}\bf\rule{\z@}{24\p@}
          Anonymous Author(s) \\
        \end{tabular}%
      \else
        \def\And{
          \end{tabular}\hfil\linebreak[0]\hfil%
          \begin{tabular}[t]{c}\bf\rule{\z@}{24\p@}\ignorespaces%
        }
        \def\AND{%
          \end{tabular}\hfil\linebreak[4]\hfil%
          \begin{tabular}[t]{c}\bf\rule{\z@}{24\p@}\ignorespaces%
        }
        \begin{tabular}[t]{c}\bf\rule{\z@}{24\p@}\stored@authorcmd\end{tabular}%
      \fi 

      \vskip 0.3in \@minus 0.1in 
    }
  \endgroup 
  \par 
  \vskip \bigskipamount 
  } 
}
\makeatother

\newcolumntype{P}[1]{>{\centering\arraybackslash}m{#1}}

\AtEndEnvironment{restatable}{\vspace{-0.5em}}
\AtEndEnvironment{observation}{\vspace{-0.5em}}

\titlecontents{subsection}[4em] 
  {\vspace{-2pt}\footnotesize\itshape}              
  {\contentslabel{2em}}         
  {}                            
  {\hfill\contentspage}         
  []                

\makenomenclature

\renewcommand{\nompreamble}{The notation is largely borrowed from \cite{eoci}, with some overloading where necessary.\\}

\newcommand{\myRule}{
\newpage
}

\newif\ifrestatingtheorem
\restatingtheoremfalse

\usepackage{stackengine,xcolor}
\input pdf-trans
\newbox\qbox
\def\usecolor#1{\csname\string\color@#1\endcsname\space}
\newcommand\bordercolor[1]{\colsplit{1}{#1}}
\newcommand\fillcolor[1]{\colsplit{0}{#1}}
\newcommand\colsplit[2]{\colorlet{tmpcolor}{#2}\edef\tmp{\usecolor{tmpcolor}}%
  \def\tmpB{}\expandafter\colsplithelp\tmp\relax%
  \ifnum0=#1\relax\edef\fillcol{\tmpB}\else\edef\bordercol{\tmpC}\fi}
\def\colsplithelp#1#2 #3\relax{%
  \edef\tmpB{\tmpB#1#2 }%
  \ifnum `#1>`9\relax\def\tmpC{#3}\else\colsplithelp#3\relax\fi
}
\newcommand\outline[1]{\leavevmode%
  \def\maltext{#1}%
  \setbox\qbox=\hbox{\maltext}%
  \boxgs{Q q 2 Tr \thickness\space w \fillcol\space \bordercol\space}{}%
  \copy\qbox%
}
\newcommand\myMathbb[2][1]{%
  \stackengine{0pt}{\def\thickness{.15}\outline{${#2}$}}{\kern.1pt\outline{${#2}$}}{O}{l}{F}{F}{L}}
\bordercolor{black}
\fillcolor{white}
\def\thickness{.4}

\usepackage{amsmath,amsfonts,bm}

\newcommand{\captiona}{{\em (a)}}
\newcommand{\captionb}{{\em (b)}}

\def\eqref#1{equation~\ref{#1}}

\def\1{\bm{1}}

\def\vg{{\bm{g}}}

\def\mD{{\mathbf{D}}}

\def\mM{{\mathbf{M}}}

\def\mS{{\mathbf{S}}}

\def\mW{{\mathbf{W}}}

\def\mZ{{\mathbf{Z}}}

\DeclareMathAlphabet{\mathsfit}{\encodingdefault}{\sfdefault}{m}{sl}
\SetMathAlphabet{\mathsfit}{bold}{\encodingdefault}{\sfdefault}{bx}{n}

\def\tN{{\tilde{N}}}

\def\sG{{\mathbb{G}}}
\def\sH{{\mathcal{H}}}

\def\sP{{\mathcal{P}}}

\def\sX{{\mathbb{X}}}
\def\sY{{\mathbb{Y}}}

\newcommand{\R}{\mathbb{R}}

\newcommand{\Var}{\mathrm{Var}}

\DeclareMathOperator*{\argmin}{argmin}

\newcommand{\indep}{\perp \!\!\!\perp}
\newcommand{\nindep}{\not\!\perp\!\!\!\perp}
\newcommand{\perpAS}{\,\perp_{\mathrm{a.s.}}\,\!}

\newcommand\given{\;\middle\vert\;}

\newcommand{\M}{\mathfrak{M}}
\newcommand{\A}{\mathfrak{A}}

\NewDocumentCommand\intervention{mg}{
\ensuremath{
\operatorname{do}\pr*{ 
    #1 \IfNoValueTF{#2}{}{ \coloneqq #2 }
}
}
}

\NewDocumentCommand\doM{mg}{
\ensuremath{
 \M;
 \IfNoValueTF{#2}{
\intervention{ #1 } 
 }{ \intervention{ #1 }{ #2 } }
}
}

\NewDocumentCommand\doA{mg}{
\ensuremath{
 \A;
 \IfNoValueTF{#2}{
\intervention{ #1 } 
 }{ \intervention{ #1 }{ #2 } }
}
}

\newcommand{\counterfactualM}[3]{ \M_{ #1 };\intervention{ #2 }{ #3 } }
\newcommand{\counterfactualA}[3]{ \A_{ #1 };\intervention{ #2 }{ #3 } }

\newcommand{\myMethod}{IVL}

\newcommand*{\da}[1][]{\text{DA}_{#1}+}
\newcommand{\erm}{\text{ERM}}

\newcommand{\causal}{\text{CR}}

\newcommand{\iv}{\text{IV}}

\newcommand*{\ivla}[1][]{\text{\myMethod}_{\alpha}^{\text{#1}}}

\newcommand{\f}{\mathbf{f}}
\newcommand{\h}{\mathbf{h}}

\newcommand{\vtau}{\bm{\tau}}
\newcommand{\veps}{\bm{\epsilon}}

\newcommand{\vzeta}{\bm{\zeta}}

\def\vx{{\mathbf{x}}}
\def\vy{{\mathbf{y}}}
\def\vc{{\mathbf{c}}}
\def\vz{{\mathbf{z}}}
\def\vg{{\mathbf{g}}}
\def\vn{{\mathbf{n}}}
\def\vv{{\mathbf{v}}}
\def\vu{{\mathbf{u}}}

\newcommand{\hHat}{\hat{\h}}

\newcommand{\tC}{\tilde{C}}
\newcommand{\tZ}{\tilde{Z}}
\newcommand{\tG}{\tilde{G}}

\NewDocumentCommand\cov{mg}{
\ensuremath{
{\bm \Sigma}_{#1}^{\IfNoValueTF{#2}{}{#2}}
}
}

\newcommand{\T}{\mathbf{M}}
\newcommand{\E}{\mathbf{T}}
\newcommand{\K}{{\bm\Gamma}}

\newcommand{\mU}{\mathbf{U}}
\newcommand{\mV}{\mathbf{V}}

\newcommand{\I}{\mathbf{I}}
\newcommand{\In}[1]{\I_{#1}}

\NewDocumentCommand\0{mg}{
\ensuremath{
    \mathbf{0}_{ #1 \IfNoValueTF{#2}{}{\times #2}}
}
}

\def\sX{{\mathcal{X}}}
\def\sY{{\mathcal{Y}}}

\def\sG{{\mathcal{G}}}
\def\sD{{\mathcal{D}}}

\DeclarePairedDelimiter\abs{\lvert}{\rvert}
\DeclarePairedDelimiter{\pr}{(}{)}
\DeclarePairedDelimiter{\br}{\{}{\}}
\DeclarePairedDelimiter{\sqPr}{[}{]}
\DeclarePairedDelimiter\norm{\lVert}{\rVert}
\newcommand{\sqNorm}[1]{\norm*{#1}^2}

\newcommand{\EE}[3]{\mathbb{E}_{#1}^{#2}\sqPr*{#3}}
\newcommand{\VV}[3]{\Var_{#1}^{#2}\pr*{\;\! #3 \;\!}}

\newcommand{\Risk}[3]{R_{#1}^{#2}\pr*{#3}}

\newcommand{\loss}[1]{\ell\pr*{#1}}
\newcommand{\xor}[1]{ \texttt{xor}\pr*{#1} }
\newcommand{\tr}[1]{ \operatorname{tr}\pr*{#1} }

\newcommand{\Range}[1]{ \operatorname{range}\pr*{#1} }
\newcommand{\Null}[1]{ \operatorname{null}\pr*{#1} }

\NewDocumentCommand\CER{mg}{
\ensuremath{
    \operatorname{CER}_{ #1 }\IfNoValueTF{#2}{}{ \pr*{#2} }
}
}
\NewDocumentCommand\nCER{mg}{
\ensuremath{
    \operatorname{nCER}_{ #1 }\IfNoValueTF{#2}{}{ \pr*{#2} }
}
}

\newcommand{\residual}{\xi}

\newcommand{\func}[2]{#1\pr*{#2}}

\NewDocumentCommand\distribution{mg}{
\ensuremath{
    \mathbb{P}_{ {#1} }\IfNoValueTF{#2}{}{ ^{#2} }
}
}

\newcommand{\betaDistribution}[2]{ 
    {\myMathbb{\bm\beta}\pr*{#1, #2}} 
}
\newcommand{\bernoulliDistribution}[1]{ 
    {\mathbb{B}\pr*{#1}} 
}
\newcommand{\normalDistribution}[2]{ 
    {\myMathbb{\mathcal{N}}\pr*{#1, #2}} 
}

\newcommand{\paperTitle}{An Analysis of Causal Effect Estimation using \\ Outcome Invariant Data Augmentation}
\title{\paperTitle}

\author{%
  \href{mailto:uzair.akbar@gatech.edu?Subject=Your NeurIPS 2025 paper}{Uzair Akbar}\thanks{Part of work done while at Max Planck Institute for Intelligent Systems and TU Munich.
  }\\
  Georgia Tech
  \And
  Niki Kilbertus\\
  TU Munich\\
  Helmholtz AI
  \And
  Hao Shen\\
  TU Munich\\
  Fortiss GmbH
  \And
  Krikamol Muandet\\
  Rational Intelligence \\ CISPA
  \And
  ~Bo~Dai\\
  Georgia Tech\\
  Google DeepMind
}

\begin{document}

\maketitle

\setcounter{footnote}{0}
\begin{abstract}
The technique of data augmentation (DA) is often used in machine learning for regularization purposes to better generalize under i.i.d. settings. In this work, we present a unifying framework with topics in causal inference to make a case for the use of DA beyond just the i.i.d. setting, but for generalization across interventions as well. Specifically, we argue that when the outcome generating mechanism is invariant to our choice of DA, then such augmentations can effectively be thought of as interventions on the treatment generating mechanism itself. This can potentially help to reduce bias in causal effect estimation arising from hidden confounders. In the presence of such unobserved confounding we typically make use of instrumental variables (IVs)---sources of treatment randomization that are conditionally independent of the outcome. However, IVs may not be as readily available as DA for many applications, which is the main motivation behind this work. By appropriately regularizing IV based estimators, we introduce the concept of \emph{IV-like ({\myMethod})} regression for mitigating confounding bias and improving predictive performance across interventions even when certain IV properties are relaxed. Finally, we cast parameterized DA as an IVL regression problem and show that when used in composition can simulate a worst-case application of such DA, further improving performance on causal estimation and generalization tasks beyond what simple DA may offer. This is shown both theoretically for the population case and via simulation experiments for the finite sample case using a simple linear example. We also present real data experiments to support our case.
\end{abstract}
\section{Introduction}
\label{sec:introduction}
A classical problem in machine learning is that of regression---using i.i.d. samples from some fixed,
unknown distribution $\distribution{X, Y}$, we predict outcome $Y$ values for unlabeled treatment $X$ values. The use of \emph{regularization} techniques is crucial for this task to achieve good generalization from training to test data \cite{vapnik}. \emph{Data augmentation (DA)} \cite{da-reg,invariance_benefit} is one such method, where each sample is randomly perturbed multiple times to grow the dataset size. 
However, these regression models cannot generally be interpreted causally as the statistical relationship between $X$ and $Y$ may arise from shared common causes, known as \emph{confounders}, rather than from $X$ influencing $Y$. Removing such confounders requires independently assigning values of $X$ during data generation, known as an \textit{intervention} \cite{eoci,Pearl_2009}.

    
    Unfortunately, we seldom have access to the data generation process to be able to intervene on variables. A common workaround is to use auxiliary variables to correct for unobserved confounders \cite{backdoor,specIV,prox-causal-learning}. One such approach is that of \emph{instrumental variables (IVs)} that represent certain conditional independences in the system which can be used to identify the causal effect of $X$ on $Y$ \cite{iv_arthur,iv_rui,iv_niki}. Alas, IVs too are generally hard to find in many popular applications such as computer vision~and~natural~language processing, motivating the need for more accessible ways to mitigate~unobserved~confounding.

Recent work therefore seeks to leverage more commonly available auxiliary variables to reduce confounding-induced bias even when the causal effect itself cannot be identified \cite{causal-regularization2,deconfounding-causal-regularization,regularizing-towards-causal-invariance,anchor}. Collectively referred to as \emph{causal regularization}, these methods aim to learn predictors that generalize \emph{out-of-distribution (OOD)} by discouraging reliance on spurious (i.e., non-causal,) correlations. Since distribution shifts often correspond to interventions on parts of the data-generating process~\cite{icp,eoci},~models that fail under such shifts typically do so because they exploit confounded relationships~\cite{when-shift-happens}.~Tackling this root cause directly, causal regularization offers a principled approach for more robust prediction.

In the same vein, more ambitious works have also explored the use of common regularization techniques, such as $\ell_1$, $\ell_2$ \cite{causal-regularization} and the min-norm interpolator \cite{causal-interpolation-regularization}, for the same purpose of causal regularization. This is in contrast to the canonical use of such regularizers for estimation variance reduction and i.i.d. prediction generalization \cite{vapnik}. Other popular regularization methods, however, remain understudied in a similar context of un-identifiable causal effect estimation,~motivating~our~work.

\paragraph{Our contributions.}\label{sec:contribution}
    To this end, we provide a first analysis of DA for estimating un-identifiable causal effects using only observational data for $(X, Y)$. Our contributions, summarized in \cref{table:summary}, include:
        \hypertarget{contribution1}{(\rnum{1})} \textbf{DA as a soft intervention (\cref{sec:da-intervention}):} We show that DA can synthesize treatment interventions when the outcome function is invariant to DA, lowering bias in causal effect estimates when the intervention acts along spurious features.
    \hypertarget{contribution2}{(\rnum{2})} \textbf{Introducing IV-like regression (\cref{sec:uiv}):} Relaxing the properties of IVs, we introduce the concept of \emph{IV-like ({\myMethod})} variables. This generalization renders IV regression ineffective at identifying causal effects, but when regularized appropriately via our proposed \emph{{\myMethod} regression}, may still reduce confounding bias and improve prediction generalization across treatment interventions.
        \hypertarget{contribution3}{(\rnum{3})} \textbf{DA parameters as {\myMethod} (\cref{sec:da-uiv}):} By casting parameterized DA as {\myMethod}, we show that its composition DA+IVL with IVL regression further reduces confounding bias beyond just simple DA by essentially simulating a worst-case or adversarial~application~of~the~DA.


     We validate our approach with theoretical results in a linear setting for the infinite-sample case, and simulation and real-data experiments in the finite-sample case.
\section{Preliminaries}\label{sec:preliminaries}
    Consider treatment $X$ and outcome $Y$ taking values in $\sX\subseteq \R^m$ and $\sY\subseteq \R^l$ respectively. Given the set of functions $\sH\coloneqq \br*{ h:\sX\rightarrow\sY }$, the canonical setting described in the literature \cite{eoci, anchor, econAnalysis} deals with estimating the function $f\in \sH$ in the \emph{structural equation model (SEM)} $\M$ of the following form\footnote{Throughout this work we shall borrow and overload notation from \cite{eoci}. See Appendix for a list of symbols.}
    \begin{align}
    \label{eq:canonical-sem}
        X = \func{\tau}{ Y, Z, C, N_X } ,
        &&
        Y = \func{f}{ X } + \func{\epsilon}{ C } + N_Y ,
    \end{align}
    where $Z$, $C$, $N_X$, $N_Y$ are exogenous (and therefore mutually independent) random variables and the residual $\residual \coloneqq Y - \func{f}{ X } = \func{\epsilon}{C} + N_Y$ is assumed to be zero mean, i.e. $\EE{}{\M}{\residual} = 0$. Since $\M$ is potentially cyclic, a priori it may entail several or no distributions at all. However, here we make the assumption that for all $\pr*{ \vx_0, \vy_0 } \in \sX\times \sY$ the unique limits
    \begin{align*}
        \vx \coloneqq \lim_{t\rightarrow\infty} \vx_t = \lim_{t\rightarrow\infty} \func{\tau}{ \vy_{t-1}, \vz, \vc, \vn_X },
        &&
        \vy \coloneqq \lim_{t\rightarrow\infty} \vy_t = \lim_{t\rightarrow\infty} \func{f}{ \vx_{t-1} } + \func{\epsilon}{\vc} + \vn_Y
    \end{align*}
    exist for any $\pr*{ \vz, \vc, \vn_X, \vn_Y } \sim \distribution{Z, C, N_X, N_Y}{\M}$, meaning that the unique distribution entailed by $\M$ is in this equilibrium state. Of course, if $\M$ is acyclic, these limits always exist. Note that assuming the existence of such an equilibrium does not violate the classic \emph{independent causal mechanism (ICM)} principle \cite{eoci}; we defer interested readers to \cref{app:cyclic-sem} for further details on cyclic~SEMs~and~the~ICM.
    

\newlength{\figurepartwidth} 
\setlength{\figurepartwidth}{0.48\linewidth} 
\newlength{\tablepartwidth}
\setlength{\tablepartwidth}{0.48\linewidth} 
\begin{figure}[t] 
    \begin{minipage}[t]{\tablepartwidth} 
    \vspace{0pt}
        \centering 
        
        \captionof{table}{A picture summary of our contributions. $\rightarrow$ represents composition of operations or transformations, and $\Leftrightarrow$ represents equivalence.}\label{table:summary}
        {
        \begin{tabular}{ P{0.022\tablepartwidth} P{0.38\tablepartwidth}@{}c@{}P{0.4\tablepartwidth}}
        \toprule
            & Type of Data Augmentation      & & Topics in Causal Inference \\
        \midrule
        \tikzmarknode{toparrow}{}    & None; observational data  & $\leftarrow$ & Data generating structural model \\
            & $\downarrow$  & & $\downarrow$ \\
            & Outcome invariant~DA   & $\overset{\text{(\hyperlink{contribution1}{\rnum{1}})}}{\xLeftrightarrow{\hspace{0.5em}}}$ & Treatment (soft)~intervention \\
            & $\downarrow$  & & $\downarrow$ \\
        \tikzmarknode{bottomarrow}{}    & Worst-case or adversarial DA        & $\overset{\text{(\hyperlink{contribution3}{\rnum{3}})}}{\xLeftrightarrow{\hspace{0.5em}}}$ & \phantom{.}~Regularized~\phantom{.} \phantom{.}~\hfill~IV regression \hfill $\scriptstyle\text{(\hyperlink{contribution2}{\rnum{2}})}$ \\
        \bottomrule
        \end{tabular}
        }

        \begin{tikzpicture}[overlay, remember picture]
            \draw[-{Stealth[length=5pt, width=4pt]}, red, thick]
                ([yshift=1.75ex, xshift=1.15em]toparrow.north west) -- ([yshift=-0.5ex, xshift=1.15em]bottomarrow.south west)
                node[midway, sloped, below, font=\small, black, inner sep=2.5pt, rounded corners=1pt] {\rotatebox{180}{\parbox{3.25cm}{\centering Lower confounding bias in causal effect estimate}}};
        \end{tikzpicture}
    \end{minipage}%
    \hfill
    \begin{minipage}[t]{\figurepartwidth} 
    \vspace{5pt}
        \centering 
        
        \begin{subfigure}{.4\figurepartwidth} 
        \centering
        \adjustbox{width=.4\figurepartwidth}{
            \includegraphics{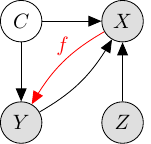}
        }
        \caption{Graph of $\M$.}
        \end{subfigure}
        \hfill
        \begin{subfigure}{.6\figurepartwidth} 
        \centering
        \adjustbox{width=.4\figurepartwidth}{
            \includegraphics{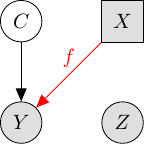}
        }
        \caption{Graph of $\doM{ X }$.}
        \end{subfigure}
        
        \captionof{figure}{Graph of $\M$ depicting an instrument $Z$ that satisfies treatment relevance, exclusion restriction, un-confoundedness and outcome relevance properties. An intervention on $X$ gives us the graph in {\captionb}. IV regression simulates such an intervention using only observational~data.}
        \label{fig:iv-reg}
    \end{minipage}%
\end{figure}

    Given a proper convex loss $\ell:\R^l\times\R^l\rightarrow\R_{+}$, \emph{empirical risk minimization (ERM)} uses a dataset $\sD\coloneqq \br*{ (\vx_i, \vy_i) }_{i=0}^n$ of $n$ samples from $\M$ to minimize an empirical version of the \emph{statistical risk}
    \begin{equation}\label{eq:erm}
        \Risk{\erm}{\M}{ h } \coloneqq \EE{}{\M}{ \loss{ Y, h\pr*{X} } } ,
    \end{equation}
    over $h \in \sH$. However, since the residual $\residual$ in \cref{eq:canonical-sem} is generally correlated with $X$,~i.e.,~$\EE{}{\M}{\residual \given X} \neq 0$, the ERM minimizer $\hat{h}^{\M}_{\erm}$ typically yields a biased estimate of $f$~\cite{Pearl_2009,eoci}. This bias arises
due to the exclusion of the (unobserved) common parent $C$ of $X$ and $Y$, i.e. a confounder, in the ERM objective (hence fittingly called the \emph{omitted-variable bias} \cite{omitted-variable-bias}) and/or the model is cyclic
(\emph{simultaneity bias}~\cite{econAnalysis,cyclic-sem-iv}, or \emph{reverse causality}~\cite{Pearl_2009} in the degenerate case). For simplicity we shall refer to either case by saying that $X$ and $Y$ are confounded and the resulting bias as the \emph{confounding bias}~\cite{Pearl_2009}.\footnote{Pearl~\cite[p.78,184]{Pearl_2009} similarly uses the term for any bias causing observational vs. interventional deviation; this also aligns with econometrics~\cite{trinity,econAnalysis}, where both are classified as sources of \emph{endogeneity} (i.e., $X\nindep \xi$).}

\subsection{Intervention for causal effect estimation}\label{sec:intervention-causal-estimation}
    We can make $X$ and the residual $\residual$ uncorrelated via an intervention\footnote{A \emph{soft} intervention replaces the mechanism $\tau$ in \cref{eq:canonical-sem} with an alternative $\tau^{\prime}$ \cite[p.~34]{eoci}. This may \emph{potentially} reduce confounding between $X$ and $Y$.} $\intervention{ X }{ X^{\prime} }$, where we explicitly set $X$ to some independently sampled $X^{\prime}$ in \cref{eq:canonical-sem} irrespective of its parents, resulting now in the new SEM $\doM{ X }{ X^{\prime} }$ or $\doM{ X }$ as a shorthand for when $X^{\prime} \sim \distribution{X}{\M}$. The distribution induced by this modified SEM is called an \emph{interventional distribution} (with respect to $\M$) under which the ERM objective from \cref{eq:erm} now defines the following \emph{causal risk (CR)} \cite{causal-regularization2,causal-interpolation-regularization,optical-device} as
    \begin{align}
    \label{eq:causal-risk}
        \Risk{ \causal }{ \M }{ h } \coloneqq \Risk{\erm}{\doM{X}}{h} = \Risk{\erm}{\doM{X}{X^{\prime}}}{h} , && \text{s.t.} && X^{\prime}\sim \distribution{X}{\M} .
    \end{align}
    Minimizing \cref{eq:causal-risk} is meaningful in two important cases where ERM fails: (\rnum{1}) \textbf{Causal effect estimation:} The minimizer $\hat{h}^{\M}_{\causal}$ of \cref{eq:causal-risk} gives us an unbiased estimate of the \emph{average treatment effect (ATE)} \cite{backdoor} $\EE{}{ \doM{X}{ \vx } }{ Y\given X = \vx } = \func{f}{\vx}$ that measures the causal influence of $X$ on $Y$. (\rnum{2}) \textbf{Robust prediction:}
    ATE based prediction of $Y$ values for unlabeled $X$ values is \emph{robust} in the sense that it can generalize across arbitrary OOD treatment interventions or shifts in the treatment distribution~\cite{causal-ood-generalizatoin}. 
    Consequently, the causal risk minimizer $\hat{h}^{\M}_{\causal}$ is also a robust predictor over the support of $\distribution{X}{\M}$. Specifically, $\hat{h}^{\M}_{\causal}$ minimizes the worst-case ERM objective over the set $\sP$ of all possible intervention distributions $\distribution{X^{\prime}}$ over the support of $\distribution{X}{\M}$ \cite{causal-ood-generalizatoin}, i.e. for $\sP\coloneqq \br*{\distribution{X^{\prime}} \given
    \func{\operatorname{supp}}{\distribution{X^{\prime}}}
    \subseteq
    \func{\operatorname{supp}}{\distribution{X}{\M}}
    }$,
    \begin{equation*}
        \hat{h}^{\M}_{\causal} \in \argmin_{h\in\sH} \max_{ \distribution{X^{\prime}} \in \sP } \Risk{\erm}{\doM{X}{X^{\prime}}}{h} .
    \end{equation*}
    To better isolate the estimation error due to confounding, we define~the~\emph{causal~excess~risk~(CER)}~\cite{causal-interpolation-regularization}
    \begin{equation*}
        \CER{\M}{h} \coloneqq \Risk{ \causal }{ \M }{ h } - \Risk{ \causal }{ \M }{ f } .
    \end{equation*}
    This removes the irreducible noise from \cref{eq:causal-risk} (see \cref{app:confounding-bias}) and directly measures how far a hypothesis $h$ deviates from the true causal function $f$ under interventions, so that $\CER{\M}{f} = 0$.
    
    Since interventions are often inaccessible for computing the risk in \cref{eq:causal-risk}, we usually rely on observational data/ distribution and additional variables to approximate them, as outlined~in~the~next~section.

\subsection{Instrumental variable regression}\label{sec:iv-regression}
    One way to get an unbiased estimate of $f$ from the observational distribution of $\M$ is to use so-called instrumental variables $Z$ with the properties \cite{Pearl_2009,eoci,iv_rui,iv_arthur,dual-iv} of: (\rnum{1}) {\bf Treatment Relevance:} $Z \nindep X$. (\rnum{2}) {\bf Exclusion Restriction:} $Z$ enters $Y$ only through $X$, i.e.~$Z\indep Y^{\doM{X}{\vx}}$.\footnote{Counterfactual definition of the exclusion restriction property~\cite[p. 248]{Pearl_2009}.}~(\rnum{3})~{\bf Un-confoundedness:} $Z\indep \residual$. (\rnum{4}) {\bf Outcome Relevance:} $Z$ carries information about $Y$, i.e. $Y\nindep Z$.
    
    Conditioning \cref{eq:canonical-sem} on $Z$ and using $\EE{}{}{ \residual \given Z } = \EE{}{}{ \residual }=0$ from the unconfoundedness property~gives
    \begin{align}\label{eq:iv-problem}
        \EE{}{\M}{ Y\given Z } = \EE{}{\M}{ f(X)\given Z } .
    \end{align}
    IV regression therefore entails solving \cref{eq:iv-problem} for $f$, which can be done by minimizing the risk \cite{dual-iv}
    \begin{equation}
    \label{eq:iv-loss}
        \Risk{\iv}{\M}{ h } \coloneqq \EE{}{\M}{ 
            \loss{ 
                Y, \EE{}{\M}{ h\pr*{ X } \given Z} 
            } 
        } .
    \end{equation}
     For linear $f(\cdot)\coloneqq\f^\top(\cdot), h(\cdot)\coloneqq\h^\top(\cdot)$ with $\f, \h \in \R^m$ and squared loss $\loss{ \vy, \vy^{\prime} } \coloneqq \sqNorm{ \vy - \vy^{\prime} }$,~this gives the two-stage-least-squares (2SLS) \cite{2SLS-3SLS} solution where the first stage regresses $X$ from $Z$,~and the second stage regresses $Y$ from predictions $\EE{}{}{ X\given Z }$ of the first stage to get~the~estimate~$\hat{h}_{\iv}^{\M}$.

\subsection{Data augmentation}\label{sec:data-augmentation}
    In this work we restrict ourselves to data augmentation with respect to which $f$ is invariant \cite{invariance_benefit,group-theory-da-jmlr}. The action of a group $\sG$ is a mapping $\delta:\sX\times\sG\rightarrow\sX$ which is compatible with the group operation. For convenience we shall write $\vg\vx \coloneqq \delta(\vx, \vg)$. We say that $f$ is \emph{invariant} under $\sG$ (or \emph{$\sG$-invariant}) if
    \begin{align*}
        f(\vg\vx) = f(\vx),\qquad\forall\; (\vg, \vx)\in \sG\times \sX.
    \end{align*}
    Less formally, we say that the map $\vg\vx$, henceforth assumed to be continuous in $\vx$, is a valid \emph{outcome-invariant} DA transformation parameterized by the vector $\vg\in\sG$. Let $\sG$ have a (unique) normalized Haar measure and $\distribution{G}$ be the corresponding distribution defined over it. For some $G\sim \distribution{G}$, the canonical application of DA seeks to minimize an empirical version of the following risk.
    \begin{align}\label{eq:data-augmentation-risk}
        \Risk{\da[G]\erm}{\M}{ h } \coloneqq \EE{}{\M}{ \loss{Y, \func{h}{GX} } } .
    \end{align}
    Note that it is sufficient to have some prior information about the symmetries of $f$ in order to be able to construct such a DA. For example, when classifying images of cats and dogs we already know that whatever the true labeling function may be, it would certainly be invariant to rotations on the images. $G$ would then represent the random rotation angle, whereas $G\vx$ would be the rotated image $\vx$.
    
    We wish to contrast the use of DA in this work with the canonical setting---to mitigate overfitting, DA is used to grow the sample size by generating multiple augmentations $(G\vx, \vy)$ for each data sample $(\vx, \vy)\sim \distribution{{X, Y}}{\M}$ \cite{invariance_benefit,group-theory-da-jmlr,savkin2020adversarial}. Such regularization, overfitting mitigation, estimation variance reduction, or i.i.d. prediction generalization is not the focus of this work and we intentionally provide \cref{eq:data-augmentation-risk} along with theoretical results that follow in the population case to emphasize that DA is not being~used~as~a~conventional~regularizer. Instead, our goal is to improve causal effect estimation and robust prediction by re-purposing DA to mitigate hidden confounding bias in the data.
    
\section{Faithfulness and Outcome Relevance in IVs}\label{sec:uiv}

The distribution $\distribution{X, Y, Z, C}{\M}$ is \emph{faithful} to the graph of $\M$ if it only exhibits independences implied by the graph \cite{eoci,pgms}.\footnote{Also known as \emph{stability} in some texts \cite[p. 48]{Pearl_2009}.} This standard assumption in IV settings renders outcome-relevance implicit and therefore rarely mentioned. In this section we discuss the case where only the first three IV properties are satisfied, i.e. outcome-relevance may not hold. Since such a $Z$ may not be a valid IV, therefore identifiability of ATE is not possible in general as the problem in \cref{eq:iv-problem} can now be misspecified, having multiple, potentially infinitely many solutions when $Y\indep Z$. Nevertheless, we shall refer to such a $Z$ as \emph{IV-like (IVL)} to emphasize that while $Z$ may not be an IV, it may still be `instrumental' for reducing confounding bias when estimating the ATE compared to the standard ERM baseline.
\paragraph{ERM regularized IV regression.}
    Despite problem misspecification for a {\myMethod} $Z$, the target function $f$ remains a minimizer for the IV risk in \cref{eq:iv-loss}. Albeit, potentially not unique---for example, a linear $h$ with squared loss leads to an under-determined problem in \cref{eq:iv-loss}. We therefore propose the following regularized version of the IV risk for such an IVL setting,
    \begin{align}\label{eq:uiv-alpha-loss}
        \Risk{\ivla}{\M}{ h } \coloneqq \Risk{\iv}{\M}{ h } + \alpha \Risk{\erm}{\M}{ h },
    \end{align}
    where $\alpha>0$ is the regularization parameter. The ERM risk as a penalty allows our estimations to have good predictive performance while the IV risk encourages solution search within the subspace where we know $f$ to be present. We refer to minimizing the risk in \cref{eq:uiv-alpha-loss} as {\it {\myMethod} regression}.
    
    Note that the motivation behind {\myMethod} regression is not the identifiability of $f$, but rather potentially better estimations of $f$ with lower confounding bias. The next section provides a concrete example.
    \begin{example}[a linear Gaussian IVL example]
\label{example:ivl}
For scalar $\sigma > 0$, non-zero matrices $\K, \E \in \R^{\ast\times m}$ and vectors $\vtau^\top, \f, \veps \in \R^m$ such that $\f^\top \vtau^\top \neq 1$ so that the following SEM $\M$ is solvable in $\pr*{X, Y}$\footnote{See \cref{app:cyclic-sem} and \cref{lemma:sem-solvability} for details on solving for and sampling of $\pr*{X, Y}$ in such linear, cyclic~SEMs.}
\begin{align*}
    X = \vtau^{\top} Y + \K^\top Z + \E^\top C + \sigma N_X ,
    &&
    Y =\f^\top X + \veps^\top C + \sigma N_Y ,
\end{align*}
where $Z, C, N_X, N_Y$ are conformable, centered Gaussian random vectors and $Z$ is {\myMethod}~w.r.t.~$\pr*{X, Y}$.\footnote{All examples assume correlated $X$ and residual $\residual$, i.e. $\EE{}{\M}{X \residual^\top} \neq {\0{}}$, as otherwise there is no confounding.}~
\end{example}
    Now, the task is to improve our estimation of $\f$ compared to standard ERM. We evaluate an estimate $\hHat^{\sD}$ using the $\CER{}$, which for a squared loss and covariance $\cov{X}{\M}$ in \cref{example:ivl}~simply~comes~out~to~be
    \begin{equation}
    \label{eq:cer}
        \CER{\M}{\hHat^{\sD}} = \sqNorm{\hHat^{\sD} - \f}_{\cov{X}{\M}} .
    \end{equation}
    Prior works use this form to quantify the error in ATE estimation \cite{causal-interpolation-regularization,causal-regularization2} or measure some notion of strength of confounding \cite{causal-regularization,dominik-optical-device-causal-reg2,optical-device}.
    Similarly, we use it to measure confounding bias of population estimates $\hHat^{\M}$ (\cref{app:confounding-bias}) and estimation error in finite sample experiments.~The~next~results~follow.
    
    \begin{restatable}[robust prediction with {\myMethod} regression]{theorem}{ivlRobustPrediction}
\label{theorem:ivl-robust-prediction}
For SEM $\M$ in \cref{example:ivl}, the following holds:
\begin{align*}
\hHat^{\M}_{\ivla} \in \argmin_{\h}\max_{\vzeta\in\mathcal{P}_{\alpha}} \Risk{\erm}{ \doM{ \K^{\top}\pr*{\cdot} }{ \vzeta } }{ \h } ,
&&
\text{s.t.}
&&
\mathcal{P}_{\alpha}\coloneqq \br*{ \vzeta \given \vzeta\vzeta^\top \preccurlyeq \left(\frac{1}{\alpha} + 1\right) \K^\top \cov{ Z }{\M} \K } .
\end{align*}
\end{restatable}
\textit{Proof.} See \cref{app:theorem-ivl-robust-prediction} for the proof. \qed
    
    \begin{restatable}[causal estimation with {\myMethod} regression]{theorem}{ivlCausalEstimation}
\label{theorem:ivl-causal-estimation}
    In SEM $\M$ of \cref{example:ivl}, for $\alpha < \infty$, we have
    \begin{align*}
        \CER{\M}{\hHat^{\M}_{\ivla}} \leq \CER{\M}{\hHat^{\M}_{\erm}} ,
        &&
        \text{equality iff}
        &&
        \EE{}{\M}{X \given Z } \perpAS \EE{}{\M}{X \given \residual } .
    \end{align*}
\end{restatable}
\textit{Proof.} See \cref{app:theorem-ivl-causal-estimation} for the proof. \qed

\Cref{theorem:ivl-robust-prediction} shows that {\myMethod} regression achieves optimal predictive performance across treatment interventions within the perturbation set $\sP_{\alpha}$ defined by $\alpha$. \Cref{theorem:ivl-causal-estimation} further states that this strictly reduces confounding bias in ATE estimates iff the perturbations align with spurious features of $X$, as indicated by the equality condition (also necessary for identifiability in linear IV settings \cite{causal-identifiability-ate,causal-ood-generalizatoin}).

\section{Causal Effect Estimation using Data Augmentation}\label{sec:causal-data-augmentation}
\newlength{\figurepartwidthNew} 
\setlength{\figurepartwidthNew}{0.53\linewidth} 

\newlength{\tablepartwidthNew}
\setlength{\tablepartwidthNew}{0.43\linewidth} 


\begin{figure}[t] 

    \begin{minipage}[t]{\figurepartwidthNew} 
        \centering 
        
        \begin{subfigure}{.465\figurepartwidthNew} 
        \centering
        \adjustbox{width=.5\figurepartwidthNew}{
            \includegraphics{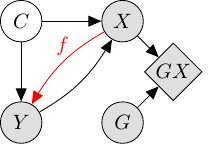}
        }
        \caption{Graph of $\A$ post DA.}
        \label{fig:da-graph}
        \end{subfigure}
        \hfill
        \begin{subfigure}{.535\figurepartwidthNew} 
        \centering
        \adjustbox{width=.3625\figurepartwidthNew}{
            \includegraphics{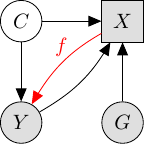}
        }
        \caption{Graph of $\doA{ \tau }{ G\tau }$}
        \label{fig:intervention-da-equivalence}
        \end{subfigure}

        \setcounter{figure}{1}
        
        \captionof{figure}{The observational distribution of $(GX, Y, G, C)$ and $(X, Y, G, C)$ for graphs {\captiona} and {\captionb} respectively are the same. The former applies DA on $X$, whereas the later applies a (soft) intervention on $X$. Furthermore, for the graph in {\captionb}, $G$ is \myMethod.}
    \end{minipage}%
    \hfill 
    \begin{minipage}[t]{\tablepartwidthNew} 
        \centering 
        
        
        \begin{subfigure}{\tablepartwidthNew} 
        \centering
        \adjustbox{width=.578\tablepartwidthNew, trim=0pt 3pt 0pt 0pt, clip}{
            \includegraphics{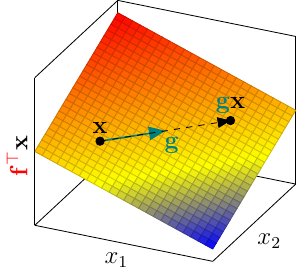}
        }
        \end{subfigure}
        
        \captionof{figure}{The ground truth function $\mathbf{f}$ in \cref{example:da}. The DA applied here corresponds to randomly translating the data samples along their level-set by adding random noise sampled from the null-space of $\mathbf{f}$.}
        \label{fig:da-cartoon}
    \end{minipage}%
\end{figure}
We dedicate this section to the main topic and point of this work---discussing the potential of data augmentation for improving predictive performance across interventions and reducing confounding bias in ATE estimates. To that effect, for the rest of this work we shall consider the following SEM $\A$
\begin{align}\label{eq:problem-setup-sem}
    X = \func{\tau}{Y, C, N_X} ,
    &&
    Y = \func{f}{X} + \func{\epsilon}{C} + N_Y ,
\end{align}
which is assumed to have a unique stationary distribution with exogenous $C, N_X, N_Y$ and the residual $\residual \coloneqq Y - \func{f}{X}$ is zero-mean, i.e. $\EE{}{}{ \residual } = 0$. We also have access to DA transformations $GX$ of $X$ parameterized by $G\sim \distribution{G}{{\A}}$ such as described in \cref{sec:data-augmentation}. \Cref{fig:da-graph} shows the graph of $\A$ post~DA.

Given samples for only $\pr*{X, Y}$ and some valid DA parameterized by $G$, the task is to improve predictive performance across interventions and reduce confounding bias in ATE estimates. We now make two observations in the following sections and state the respective results that follow~thereof.

\subsection{Data augmentation as a soft intervention}
\label{sec:da-intervention}
Consider a (soft) intervention on $\A$ where we substitute the mechanism $\tau$ of $X$ with $G\tau$. With some abuse of notation, we shall represent this SEM by $ \doA{\tau}{G\tau} $ the graph of which is shown in \cref{fig:intervention-da-equivalence}. Note that this SEM also has a unique stationary distribution (proof in \cref{app:proposition-da-intervention-distribution-existence}). Comparing the DA mechanism in $\A$ (\cref{fig:da-graph}) and the intervention $ \doA{\tau}{G\tau} $ (\cref{fig:intervention-da-equivalence}), we see:

\begin{observation}[soft intervention with DA]
\label{observation:da-as-intervention}
Distributions $\distribution{GX, Y, G, C}{\A}$ and $\distribution{X, Y, G, C}{ \doA{\tau}{G\tau} }$ are identical.
\end{observation}
\vspace{-0.5\baselineskip}
We can hence treat samples generated from $\A$ via DA as if they were instead generated from $ \doA{\tau}{G\tau} $ by intervening on $X$. This allows us to re-write the DA+ERM risk from \cref{eq:data-augmentation-risk} as,
\begin{align*}
    \Risk{\da[G]\erm}{\A}{ h } = \Risk{\erm}{ \doA{\tau}{G\tau} }{ h },
\end{align*}
to emphasize that DA is equivalent to a (soft) intervention and as such can be used to reduce confounding bias when estimating $f$, as we will show with the following example.

\begin{example}[a linear Gaussian DA example]
\label{example:da}
For scalars $\kappa, \sigma > 0$, non-zero matrices $\K, \E \in \R^{\ast\times m}$ and vectors $\vtau^\top, \f, \veps \in \R^m$ such that $\f^\top \vtau^\top \neq \kappa^{-1}$ so that the following SEM $\A$ is solvable in $(X, Y)$
\begin{align*}
    X = \kappa\cdot \vtau^\top Y + \E^\top C + \sigma N_X ,
    &&
    Y = \f^\top X + \kappa\cdot \veps^\top C + \sigma N_Y ,
    &&
    GX \coloneqq X + \gamma\cdot \K^\top G ,
\end{align*}
where $G, C, N_X, N_Y$ are conformable, centered Gaussian random vectors, $\kappa$ determines how much $\pr*{X, Y}$ are confounded and $\Range{\K^\top}\subseteq\Null{\f^\top}$ so that $GX$ is a valid outcome invariant DA transformation of $X$ parameterized by $G$ with \emph{strength} $\gamma > 0$. This transformation can be viewed as translating $X$ along its level-set as shown in \cref{fig:da-cartoon} and represents our prior knowledge about the symmetries of $\f$ for the purposes of this example.
\end{example}

\begin{restatable}[causal estimation with DA+ERM]{theorem}{daCausalEstimation}
\label{theorem:da-causal-estimation}
For SEM $\A$ in \cref{example:da}, the following holds:
\begin{align*}
    \CER{\A}{ \hHat^{\A}_{\da[G]\erm} } \leq \CER{\A}{ \hHat^{\A}_{\erm} } ,
    &&
    \text{equality iff}
    &&    
    \EE{}{\A}{ GX \given G } \perpAS \EE{}{\A}{ X \given \residual } .
\end{align*}
\end{restatable}
\textit{Proof.} See \cref{app:theorem-da-causal-estimation} for the proof. \qed

That is, DA strictly reduces confounding bias in ATE estimate iff the induced intervention perturbs $X$ along spurious features. Importantly, \cref{theorem:da-causal-estimation} suggests that lower confounding~bias~is~not a `free lunch' with outcome invariance of DA and practitioners may need domain knowledge to construct DA that targets spurious features. Fortunately however, \cref{theorem:da-causal-estimation} also suggests~that~with~outcome invariance, DA should not perform worse than ERM. We say that DA+ERM \emph{dominates} ERM on causal estimation \cite[p. 48]{domination}. Practitioners may therefore be advised to generously use such DA, as it achieves regularization in the worst case, and mitigates confounding bias as a `bonus' in the best case.

\subsection{Worst-case data augmentation with {\myMethod} regression}
\label{sec:da-uiv}
We once again point our attention to the graph of $ \doA{\tau}{G\tau} $ from \cref{fig:intervention-da-equivalence} to observe that:

\begin{observation}[IV-like DA parameters]
\label{observation:da-as-ivl}
    In SEM $ \doA{\tau}{G\tau} $, the DA parameters $G$ are {\myMethod}.
\end{observation}

In light of this we can now re-write the IV and {\myMethod} risks for $ \doA{\tau}{G\tau} $ to respectively read
\begin{align*}
    \Risk{\da[G]\iv}{\A}{ h } = \Risk{\iv}{ \doA{\tau}{G\tau} }{ h },  
    &&  
    \Risk{\da[G]\ivla}{\A}{ h } = \Risk{\ivla}{ \doA{\tau}{G\tau} }{ h }.
\end{align*}

\begin{restatable}[worst-case DA with DA+{\myMethod} regression]{corollary}{worstCaseDA}
\label{corollary:worst-case-da}
For SEM $\A$ in \cref{example:da}, it holds that
\begin{align*}
\hHat^{\A}_{\da[G]\ivla} \in \argmin_{\h}\max_{\vg\in\sG_{\alpha}} \Risk{\da[\vg]\erm}{\A}{ \h } ,
&&
\text{s.t.}
&&
\sG_{\alpha}\coloneqq \br*{ 
    \vg \given  \K^\top \vg \vg^\top \K \preccurlyeq \pr*{ \frac{1}{\alpha} + 1 } \K^\top \cov{ G }{\A} \K 
} .
\end{align*}
\end{restatable}
\textit{Proof.} The result follows from \cref{observation:da-as-intervention}, \cref{observation:da-as-ivl} and \cref{theorem:ivl-robust-prediction}. \qed

\begin{restatable}[causal estimation with DA+{\myMethod} regression]{corollary}{daivlCausalEstimation}
\label{corollary:da-ivl-causal-estimation}
For $\alpha, \gamma < \infty$ in SEM $\A$ from \cref{example:da},
\begin{align*}
    \CER{\A}{ \hHat^{\A}_{\da[G]\ivla} } \leq \CER{\A}{ \hHat^{\A}_{\da[G]\erm} } ,
    &&
    \text{equality iff}
    &&
    \EE{}{\A}{ GX \given G } \perpAS \EE{}{\A}{ X \given \residual } .
\end{align*}
\end{restatable}
\textit{Proof.} The result follows directly from \cref{theorem:ivl-causal-estimation} and \cref{observation:da-as-ivl}. \qed

Using DA parameters as IVL therefore simulates a worst-case, or adversarial application of DA within a set of transforms $\sG_{\alpha}$. Of course \cref{corollary:worst-case-da} can also be viewed as a predictor that generalizes to treatment interventions encoded by $\sG_{\alpha}$. As is intuitive, such a worst-case intervention improves our ATE estimation so long as the features of $X$ intervened along include some that are spurious (\cref{corollary:da-ivl-causal-estimation}). DA and IVL regression may therefore be used in composition if the application can benefit from regularization and/ or better prediction generalization across DA-induced interventions, with a `bonus' of lower confounding bias if the DA also augments any spurious features of $X$.

\section{Related Work}\label{sec:related-work}
{
\newlength{\clipLeftSweep}
\newlength{\clipRightSweep}
\newlength{\clipBottomSweep}
\newlength{\clipTopSweep}

\setlength{\clipLeftSweep}{1.75cm}
\setlength{\clipRightSweep}{0.375cm}
\setlength{\clipBottomSweep}{0.35cm}
\setlength{\clipTopSweep}{0.3cm}

\setlength{\fboxsep}{0pt}

\begin{figure*}[t]
\centerfloat
\begin{subfigure}{0.35194\linewidth}
\centering
    \raisebox{0.15ex}{\includegraphics[width=\linewidth,trim={0.3cm \clipBottomSweep{} \clipRightSweep{} \clipTopSweep{}},clip]{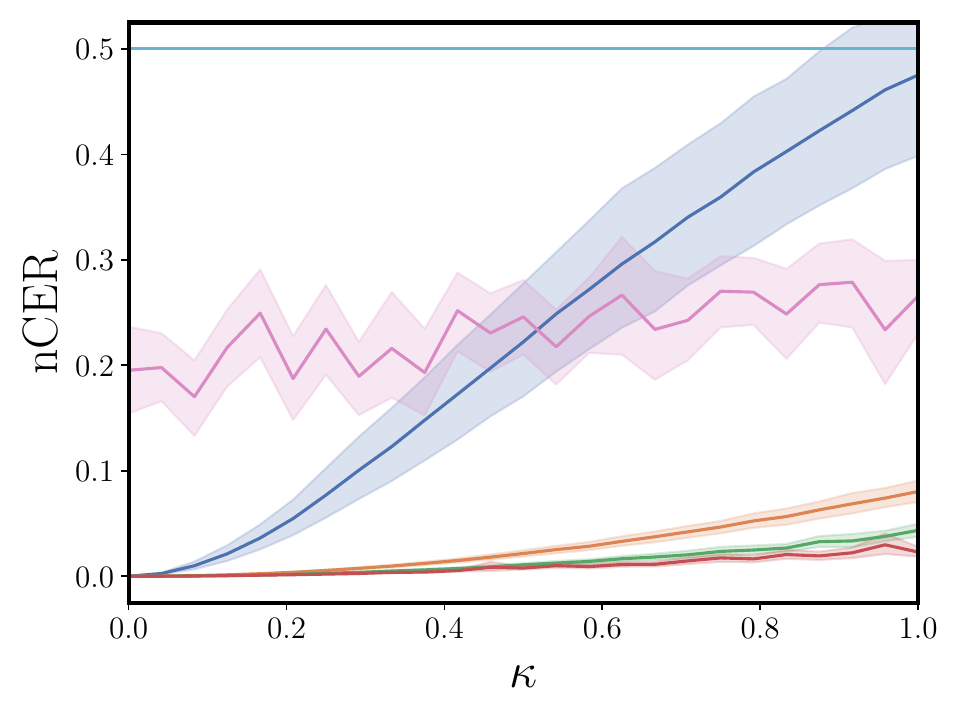}}
    \caption{$\gamma = 1$, $\kappa\in[0, 1]$}
    \label{fig:kappa-sweep}
\end{subfigure}
\hfill
\begin{subfigure}{0.31903\linewidth}
\centering
    \includegraphics[width=\linewidth,trim={\clipLeftSweep{} \clipBottomSweep{} \clipRightSweep{} \clipTopSweep{}},clip]{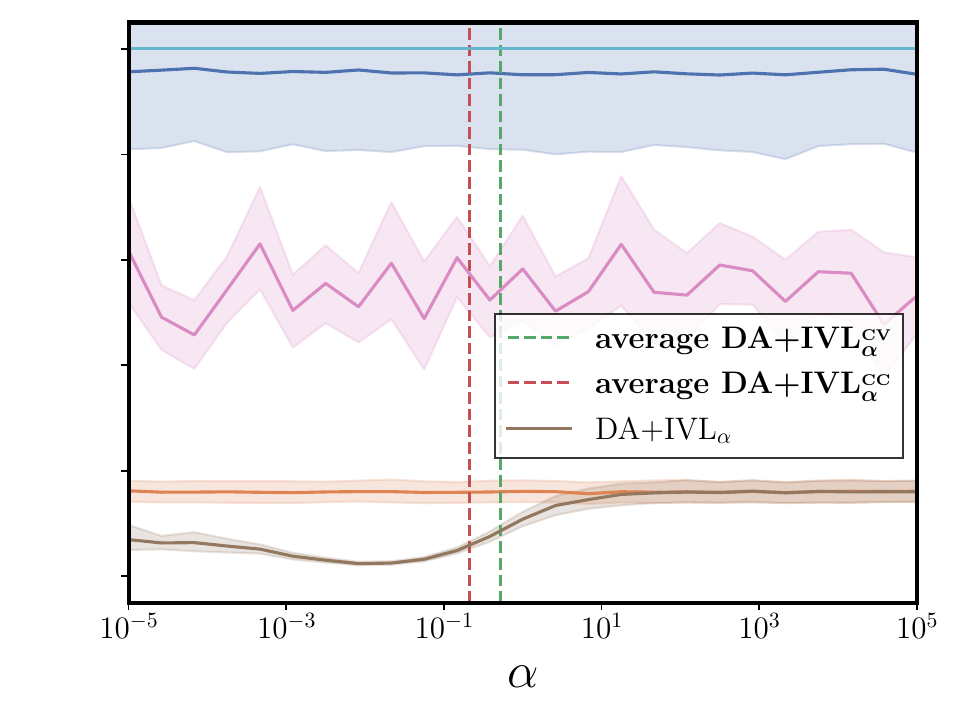}
    \caption{$\gamma, \kappa = 1$, $\alpha\in[10^{-5}, 10^5]$}
    \label{fig:alpha-sweep}
\end{subfigure}
\hfill
\begin{subfigure}{0.31903\linewidth}
\centering
    \includegraphics[width=\linewidth,trim={\clipLeftSweep{} \clipBottomSweep{} \clipRightSweep{} \clipTopSweep{}},clip]{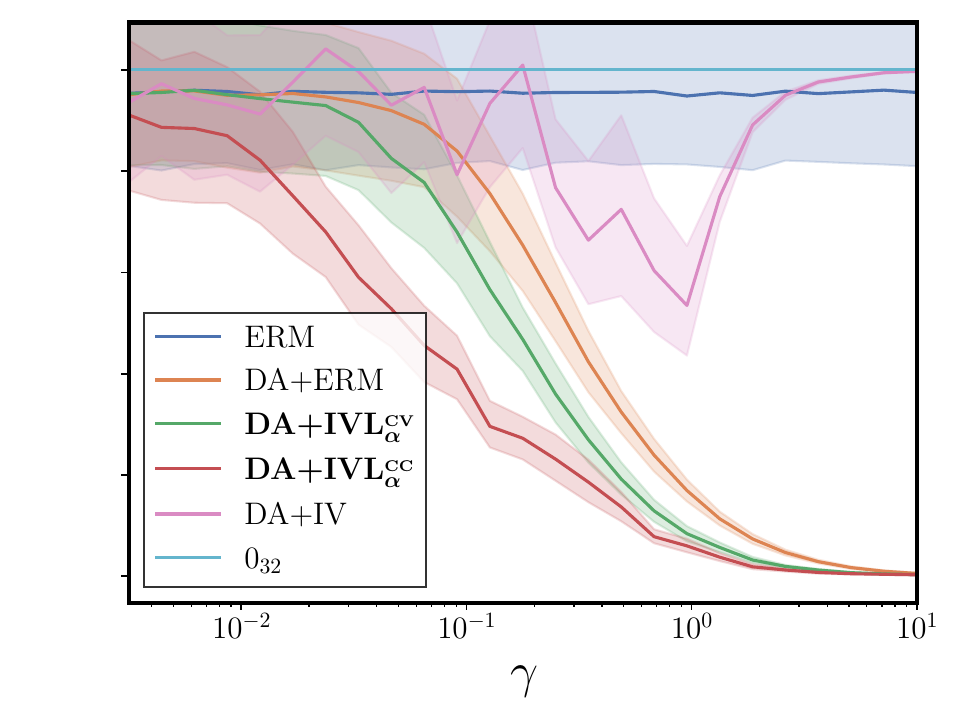}
    \caption{$\gamma \in [10^{-2.5}, 10]$, $\kappa=1$}
    \label{fig:gamma-sweep}
\end{subfigure}
\caption{Simulation experiment for a linear Gaussian SEM. $\kappa$ represents the amount of confounding, $\gamma$ is the strength of DA and $\alpha$ is the {\myMethod} regularization parameter. Each data-point represents the average ${\nCER{}}$ over $25$ trials with a $95\%$ confidence interval (CI).}
\label{fig:simulation-experiment}
\end{figure*}

}

\textbf{Causal regularization} is perhaps the most appropriate classification for this work. These methods aim for more robust prediction by mitigating the upstream problem of confounding bias in a more accessible way than is required for full identification. This is done, for example, by relaxing properties of auxiliary variables \cite{causal-regularization2,deconfounding-causal-regularization,regularizing-towards-causal-invariance,anchor}, as we have done via our IVL approach. Most relevant, however, are methods that re-purpose common regularizers, canonically used for estimation variance reduction and i.i.d. prediction generalization, for confounding bias mitigation. Of note is \cite{causal-regularization}, where a certain linear modeling assumption allows the estimation of $\sqNorm{\f}$ from observational $(X, Y)$ data, which is then used to develop a cross-validation scheme for $\ell_1,\ell_2$ regularization. \cite{causal-interpolation-regularization} conducted a similar theoretical analysis for the min-norm interpolator. To the best of our knowledge, we are the first to study the same for DA---re-purposing yet another ubiquitous regularizer to mitigate confounding~bias. 

\textbf{Domain generalization (DG)} \cite{krikamol-domain-generalization} methods aim for prediction generalization to unseen test domains via \emph{robust optimization (RO)} \cite{RO} over a perturbation set $\sP$ of possible test domains $\rho\in\sP$ as
\begin{equation*}
\Risk{\text{{RO}}}{\sP}{ h } \coloneqq \max_{\rho\in\sP} \Risk{\erm}{\rho}{ h } ,
\end{equation*}
Since generalizing to arbitrary test domains is impossible, the choice of perturbation set encodes one's assumptions about which test domains might be encountered. Instead of making such assumptions a priori, it is often assumed to have access to data from multiple training domains which can inform one's choice of perturbation set. This setting is explored in group distributionally robust optimization (DRO) \cite{dro}. Variations have been used to mitigate confounding bias and subsequently generalize to treatment interventions when used with interventional data \cite{icp,causal_invariance_nonlinear}, confounder information (i.e. entire graph) \cite{riskextrap,15,cocycle} or some proxy thereof in the form of environments \cite{arjovsky2020invariant,ac-vocal,owb-comain-rl,riskextrap}. We, however, do not assume access to any of these and instead synthesize interventions via DA.


\textbf{Counterfactual DA} strategies have been the primary lens for causal analyses of DA \cite{ilse2020selecting,notJustPrettyPictures,causalDA4LLM,mocoda,caiac,causal-matching,contrastive}. These aim for prediction robustness to treatment interventions via DA simulated \emph{counterfactuals}.\footnote{Representing an SEM with exogenous noise distribution conditioned on some variable $Y=\vy$ by $\A_{Y=\vy}$, the counterfactual SEM $\counterfactualA{Y=\vy}{X}{\vx}$ is an intervention $\intervention{X}{\vx}$ on $\A_{Y=\vy}$. The resulting \emph{counterfactual distribution} then captures questions like: ``After observing $Y=\vy$, what would have been~had~$X=\vx$~been~true.''} As with counterfactual reasoning more broadly, this requires strong assumptions---such as access to the full SEM \cite{notJustPrettyPictures,causalDA4LLM}, auxiliary variables \cite{ilse2020selecting,causalDA4LLM,causal-matching,contrastive}, or causal graphs \cite{mocoda,caiac}. By contrast, we show that outcome invariance of DA suffices for treatment intervention robustness without invoking counterfactuals. Moreover, prior work has largely overlooked causal effect estimation, often assuming reverse-causal settings where the ATE becomes trivial \cite{ilse2020selecting,causalDA4LLM,notJustPrettyPictures}. Ours is the first framework to study ATE estimation via DA with minimal structural assumptions.

\textbf{Invariant prediction} based methods aim to make predictions based on statistical relationships that remain stable across all domains in $\sP$. A common assumption, for instance, is that $\distribution{Y\mid X}$ is invariant across $\sP$, with only the marginal $\distribution{X}$ being allowed to vary. Invariance is also closely linked to causal discovery---following the classic ICM principle~\cite{eoci}, causal mechanisms remain stable under interventions on inputs \cite{causal-ood-generalizatoin,when-shift-happens}. This connection has inspired approaches that enforce invariance conditions to recover causal structures \cite{icp,causal_invariance_nonlinear}. IV regression can also be viewed as one such method, where the goal is to learn predictors whose residuals are invariant to the instruments \cite{iv_rui,iv_arthur,dual-iv,deep-iv-arthur,specIV}. More broadly, the principle of invariance, whether motivated by causality or otherwise,~has~proven~useful for improving prediction generalization across heterogeneous settings~\cite{anchor,arjovsky2020invariant,bo,regularizing-towards-causal-invariance,7DA,fanny,zz1-droid,rice,krikamol-domain-generalization}.~

\section{Experiments}\label{sec:experiments}
We began by presenting results in the infinite-sample setting to emphasize that mitigating confounding bias is fundamentally not a sample size issue, i.e., not solvable through traditional regularization alone. In this section, we turn to the finite-sample regime and empirically evaluate the effectiveness of DA in reducing hidden confounding bias. Importantly, we do not use DA for its conventional purpose of augmenting data to improve i.i.d. generalization or reduce estimation variance. Throughout all experiments, we therefore fix the number of samples in the augmented dataset to match that of the original dataset since our focus lies squarely on robust prediction via confounding bias mitigation.

Finding baselines for evaluating our results is however a challenge---the problem of mitigating confounding bias given only observational $(X, Y)$ data and symmetry knowledge via DA is quite underexplored. Nevertheless, for the sake of completeness we make an effort to re-purpose existing methods from domain generalization, invariance learning and causal inference literature to be used as baselines. These methods often require access to additional variables (e.g. IVs, confounders, domains/environments, etc.), and to maintain fairness we will replace these with DA parameters $G$. Such a comparison is conceptually valid since by virtue of being DG methods, they are essentially solving a robust loss of a similar form as in \cref{corollary:worst-case-da}, giving us meaningful baselines for DA+IVL.

In addition to standard ERM, DA and IV regression, our baselines include DRO \cite{dro}, invariant risk minimization (IRM) \cite{arjovsky2020invariant}, invariant causal prediction (ICP) \cite{icp}, regularization with invariance on causal essential set (RICE) \cite{rice}, variance risk extrapolation (V-REx) and minimax risk extrapolation (MM-REx) \cite{riskextrap}. We also include the causal regularization method by Kania and Wit \cite{causal-regularization2} and the $\ell_1, \ell_2$ approaches by Janzing \cite{causal-regularization}. We discretize $G$ if the method accepts only discrete variables. For {\myMethod} regression, we select the regularization parameter $\alpha$ in a variety of ways, including vanilla cross validation (CV), level-based CV (LCV) and confounder correction (CC) as described in \cref{app:implementation-details}. Other implementation details are provided in \cref{app:experiments}, and the code to reproduce our results is publicly released at \url{https://github.com/uzairakbar/causal-data-augmentation}.

To make CER based evaluation more interpretable for our experiments, we propose the normalization
\begin{align*}
    \nCER{\M}{h} \coloneqq \frac{
    \CER{\M}{h}
    }{
    \CER{\M}{h} + \CER{\M}{ h_0 }
    } \in \sqPr*{0, 1} ,
    &&
    \func{h_0}{ \cdot } \coloneqq \EE{}{\doM{X}}{Y} ,
\end{align*}
where $h_0$ represents the null treatment effect, i.e. when $X$ has no causal influence on $Y$, then $\EE{}{\doM{X}}{Y \given X} = \EE{}{\doM{X}}{Y}$. The normalized CER (nCER) can be considered a generalization of the metrics used by \cite{causal-regularization,optical-device,dominik-optical-device-causal-reg2} in linear settings and similarly has the interesting property that it is $0$ for the ground-truth causal solution $h=f\neq h_0$ but $1$ if there is pure confounding~for~$h\neq f = h_0$. Janzing argues in \cite{optical-device,dominik-optical-device-causal-reg2} that using an Euclidean norm instead of the weighted norm in \cref{eq:cer} is more relevant for causal settings, which also motivates our choice when evaluating results of the simulation and optical-device experiments described below. Conceptually, this is equivalent to evaluation based on the causal risk in \cref{eq:causal-risk} under the interventional distribution $X^{\prime}\sim \normalDistribution{\0{m}}{\mathbf{I}_{m}}$.

\subsection{Simulation experiment}\label{sec:sim-exp}
        For the finite sample results of the linear SEM $\A$ from  \cref{example:da}, by taking $m=32$, $k=31$ (dimension of $G$), $\sigma=0.1$ and fixing $\vtau^\top = {\0{m}}$,\footnote{Simulation results are similar under a cyclic setting with a non-trivial $\vtau$, and discussed under \cref{app:experiments}.} we sample a new $\f, \veps$ and $\E \in \R^{m\times m}$ from a standard normal distribution for each of the $32$ experiments for every combination of $\kappa$ and $\gamma$. Each time we construct a $\K \coloneqq \mV_0$ with $k$ rows as orthonormal basis of $\Null{ \f }$, such that the SVD of $\f$ is
        \begin{equation*}
            \f = 
            \begin{bmatrix}
                \vu & \mU_0
            \end{bmatrix}
            \begin{bmatrix}
                \lambda & \0{1}{(m-1)}
                \\
                \0{(m - 1)}{1} & \0{(m - 1)}{(m - 1)}
            \end{bmatrix}
            \begin{bmatrix}
                \vv^\top
                \\
                \mV_0^\top
            \end{bmatrix}.
        \end{equation*}
        Although this construction of $\K$ relies on direct knowledge of $\f$, which is of course unavailable in practice, we include it here purely for illustrative purposes. We treat access to $\K$ as having prior knowledge about the structural symmetries of $\f$, noting that this information alone is insufficient~to~recover~$\f$.
        
        We then generate $n=2048$ samples of $(X, Y)$ for each experiment. For ERM we use a closed form linear OLS solution. For DA$+$IV, we make use of linear 2SLS. Finally, DA+$\ivla$ was implemented using a closed form linear OLS solution between empirical versions (see \cref{proposition:ivl-closed-form-solution}) of
        \begin{align*}
            X^{\prime}\coloneqq \sqrt{\alpha}X+ \pr*{\sqrt{1+\alpha} - \sqrt{\alpha}} \EE{}{}{ X\given Z } ,
            &&
            Y^{\prime}\coloneqq \sqrt{\alpha} Y + \pr*{\sqrt{1+\alpha}-\sqrt{\alpha}}\EE{}{}{ Y\given Z } .
        \end{align*}
        Our first experimental result in \cref{fig:kappa-sweep} compares the different estimation methods across varying levels of confounding $\kappa \in [0, 1]$. As expected, ERM performance degrades with increasing confounding. Applying DA alone already brings us closer to the causal solution, while DA+IVL achieves even better performance. DA+IV regression is unstable and generally performs poorly as it~is~under-determined.~

        Next, we fix the confounding and DA strengths at $\kappa = \gamma = 1$, and sweep over the regularization parameter $\alpha \in [10^{-5}, 10^5]$ for DA+$\ivla$. \Cref{fig:alpha-sweep} shows that optimal performance is achieved for intermediate values of $\alpha$, confirming that arbitrarily small values of $\alpha$, while beneficial in the theoretical population setting (as suggested by \cref{eq:ivl-to-erm} in the proof of \cref{theorem:ivl-causal-estimation}), are suboptimal for finite samples.\footnote{We conjecture that this may be due to outcome invariance not holding exactly in practice. A more rigorous investigation is deferred to future work in order to keep the current manuscript more focused.} We also find that both CV and CC strategies effectively select reasonable~values~of~$\alpha$.~
        
        Lastly, \cref{fig:gamma-sweep} examines sensitivity to the DA strength $\gamma \in [10^{-2.5}, 10]$, for fixed confounding strength $\kappa = 1$. As expected, stronger DA results in stronger interventions on $X$, which improves causal effect estimation. However, we also observe diminishing returns; when the variation induced by DA is either too small or too large, DA+$\ivla$ does not yield significant improvements~over~the~DA+ERM~baseline.

        {

\newlength{\clipLeftBaseline}
\newlength{\clipRightBaseline}
\newlength{\clipBottomBaseline}
\newlength{\clipTopBaseline}

\setlength{\clipLeftBaseline}{3.75cm}
\setlength{\clipRightBaseline}{0.35cm}
\setlength{\clipBottomBaseline}{0.5cm}
\setlength{\clipTopBaseline}{0.3cm}

\setlength{\fboxsep}{0pt}

\begin{figure*}[t]
\centerfloat
\begin{subfigure}{.385\linewidth}
\centering
    \includegraphics[width=\linewidth,trim={0.33cm \clipBottomBaseline{} \clipRightBaseline{} \clipTopBaseline{}},clip]{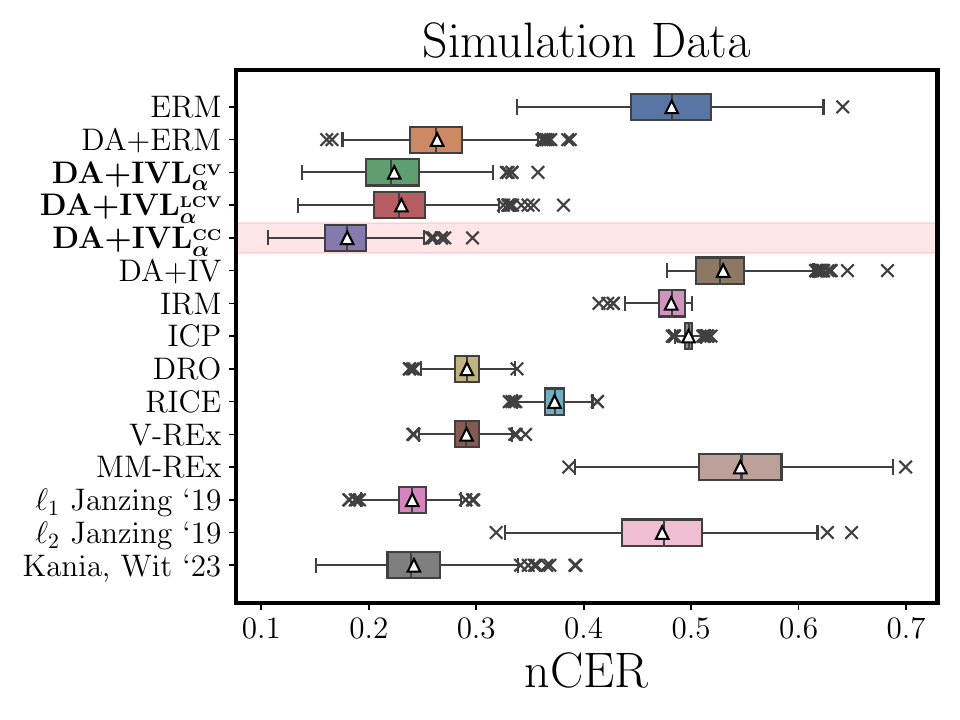}
\end{subfigure}
\hfill
\begin{subfigure}{.3000\linewidth}
\centering
    \includegraphics[width=\linewidth,trim={\clipLeftBaseline{} \clipBottomBaseline{} \clipRightBaseline{} \clipTopBaseline{}},clip]{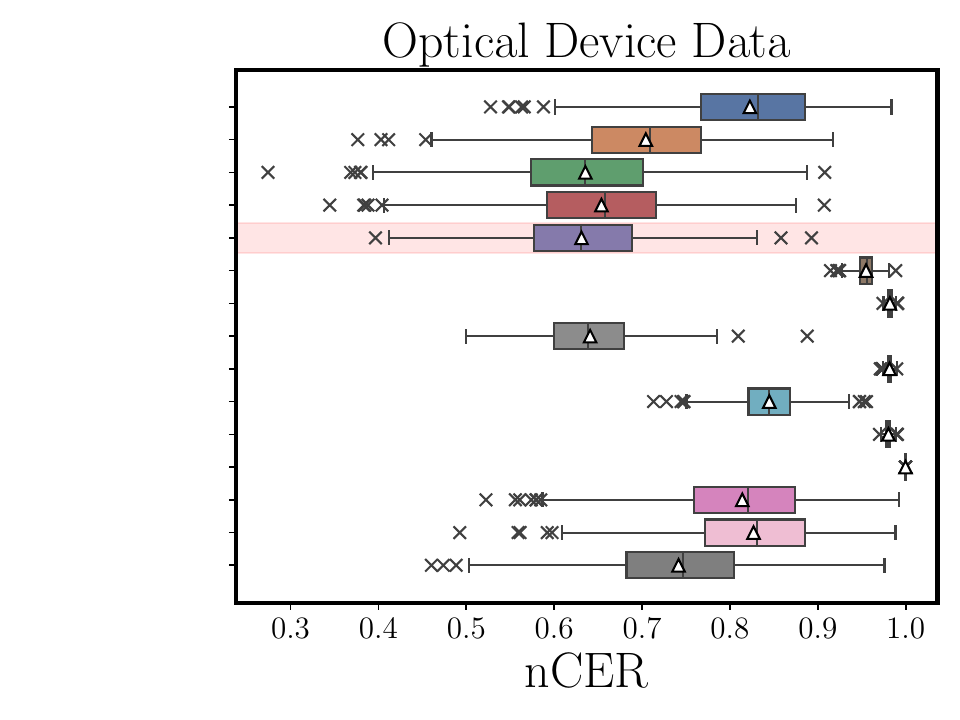}
\end{subfigure}
\hfill
\begin{subfigure}{.3000\linewidth}
\centering
    \includegraphics[width=\linewidth,trim={\clipLeftBaseline{} \clipBottomBaseline{} \clipRightBaseline{} \clipTopBaseline{}},clip]{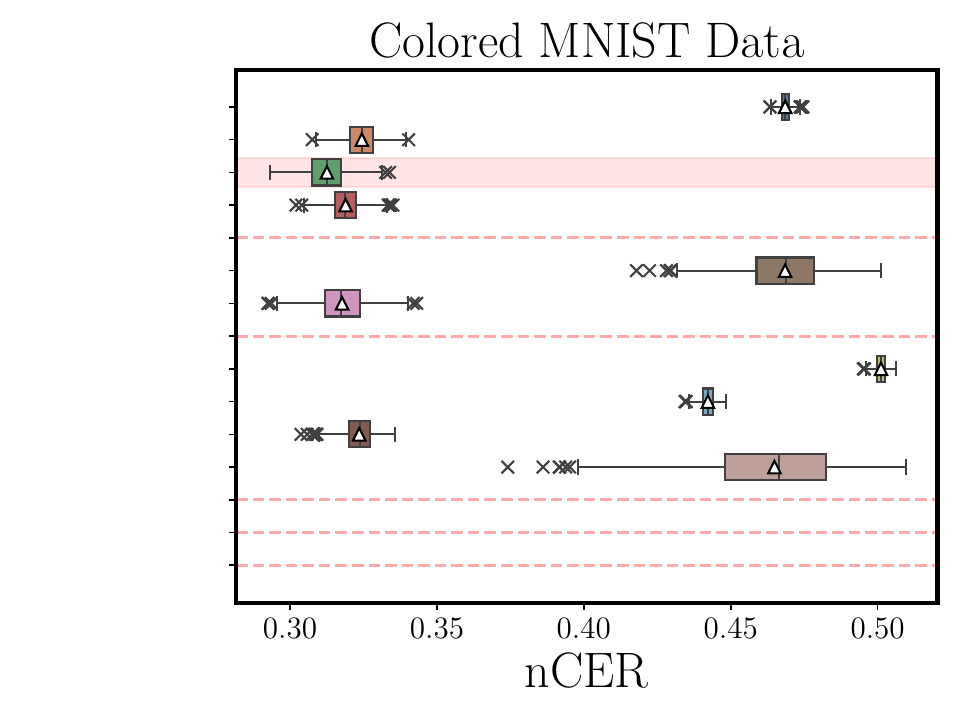}
\end{subfigure}
\caption{Experiment results; common OOD generalization benchmarks compared against the ERM, DA+ERM and DA+IV baselines including DA+{\myMethod}.}
\label{fig:baselines-benchmarks-experiments}
\end{figure*}

}

        For completeness, we also benchmark our approach against other baseline methods on 16 distinct simulation SEMs with 2048 samples each. Aggregated results are presented in \cref{fig:baselines-benchmarks-experiments}~(left~most).~

\subsection{Real data experiments}\label{sec:experiments-real-data}

\paragraph{Optical device dataset.}\label{app:optical-device}
The dataset from \cite{optical-device} consists of $3\times3$ pixel images $X$ displayed on a laptop screen that cause voltage readings $Y$ across a photo-diode. A hidden confounder $C$ controls two LEDs; one affects the webcam capturing $X$, the other affects the photo-diode measuring $Y$. The ground-truth predictor $\f$ is computed by first regressing $Y$ on $(\phi(X), C)$, where $\phi(X)$ are polynomial features of $X$ with degree $d\in\{1, \cdots, 5\}$ that best explains the data (degree 2 in most cases). The component corresponding to $C$ is then removed to recover $\f$. We add Gaussian noise $G\sim\normalDistribution{\mathbf{0}}{\cov{X}{}/10}$ for DA and fit the methods from \cref{sec:sim-exp} on features $\phi(GX)$ for $n = 1000$ samples across 12 datasets. Note that using the same ground-truth polynomial degree for $\phi$ during evaluation is important here so as to avoid introducing statistical bias from model-misspecification as our analysis squarely focuses on confounding bias. \Cref{fig:baselines-benchmarks-experiments} (middle) shows the results, where DA+ERM improves over ERM, and DA+{\myMethod} performs even better, outperforming other baselines.

\paragraph{Colored MNIST.}\label{app:c-mnist}
We evaluate on the colored MNIST dataset \cite{arjovsky2020invariant}, where labels are spuriously correlated with image color during training, but this correlation is flipped at test time. We use the same neural architecture and parameters as \cite{arjovsky2020invariant} across all baselines, training with the IV-based objective described in the \cref{app:iv-supp-implementation}. DA is implemented via small perturbations to hue, brightness, contrast, saturation, and translation, each parameterized by $G \sim \betaDistribution{2}{2}$. Although these do not directly manipulate color, the actual spurious feature, they still help reduce confounding. Results in \cref{fig:baselines-benchmarks-experiments} (rightmost) show that ERM underperforms, DA+ERM provides substantial gains, and DA+$\ivla$ performs competitively with the best DG baselines, with DA+$\ivla[CV]$ achieving the best overall performance. Interested readers may also visit \cref{app:cmnist-experiment}, where we clarify the connection of the colored MNIST model with the cyclic SEM from \cref{eq:problem-setup-sem}.

\section{Limitations}
\label{app:limitations}

\paragraph{Necessity and practicality of prior knowledge.}
As discussed in \cref{sec:causal-data-augmentation}, outcome invariance alone does not suffice to lower confounding bias and practitioners may need domain knowledge to construct DA that targets spurious features as well. Alternatively, one can also take a `carpet bombing' approach by exhausting all available outcome invariant DA in hope that some may align with spurious features. Nevertheless, under outcome invariance, our methods should perform no worse than standard ERM.


Fundamentally, causal estimation from purely observational data is impossible without untestable assumptions. For instance, the IV (or IVL) assumptions of un-confoundedness~and~exclusion restriction are inherently untestable and must be justified through domain knowledge. Moreover, the requirement of alignment with spurious features in \cref{theorem:ivl-causal-estimation} is not an artifact of our IVL relaxation—it is a rephrasing of the exclusion principle that underlies identifiability in IV regression. If an IV does not influence $Y$ through the spurious features of $X$, the corresponding causal components of $f$ cannot be identified \cite{causal-ood-generalizatoin}. IVLs, being relaxations of IVs, inherit these same untestable premises.

Viewed through the lens of IVs/IVLs (\cref{observation:da-as-ivl}), our assumptions on DA are arguably more modest than they may initially seem, especially since a symmetry-based DA model has well-established precedent in the literature \citep{invariance_benefit,group-theory-da-jmlr,7DA,3DA,4DA,5DA,6DA,8DA,9DA,10DA}. This correspondence can be summarized as follows:
\begin{align*}
    \overbrace{
\underbrace{\text{outcome-invariance}}_{\text{popular model for DA}}
+
\underbrace{\text{spurious targets}
}_{\text{benign failure if violated}}
}^{\text{un-testable DA assumptions}}
&&
\Longleftrightarrow
&&
\overbrace{\text{un-confoundedness}+\text{exclusion}}^{\text{un-testable IV/IVL assumptions}}
\end{align*}
In this light, our framework may in fact be quite practical in domains where valid IVs (or other auxiliary variables) are scarce, but plausible outcome-invariances—i.e.,~data~augmentations—are~abundant.

Finally, we recognize the hesitation in committing to strict notions of outcome invariance in practice and leave a more thorough exploration of approximate or even violated invariance to future work.

\paragraph{Choice of $\alpha$.}
Selecting the IVL regularization parameter $\alpha$ in finite-sample settings is not straightforward. As outlined in \cref{app:ivl-supplement}, we propose several strategies that work well empirically, though some may appear less principled since $\alpha$ is tuned via cross-validation within the same distribution, even though the task concerns OOD generalization. This challenge is not unique to IVL, but rather a broader limitation common to DG methods~\cite{lost-domain-generalization}.
\section{Conclusion}
We conclude that our proposed causal framework for data augmentation (DA) enables re-purposing the widely used i.i.d. generalization tool for OOD generalization across treatment interventions. By interpreting outcome-invariant DA as interventions and IV-like variables, our approach reduces confounding bias and consequently improves both causal effect~estimation~and~robust~prediction.~


\section*{Acknowledgments}
To my co-authors for their patience, to Zulfiqar for being my rubber duck and saving the OpenReview submissions minutes before the deadline, and to all of ML pyos for lightening the chaos with comedy. Thank you.

This work was supported by the NSF (ECCS-2401391, IIS-2403240),~and~ONR~(N000142512173).

\newpage


{\small
\bibliography{bibliography_final}

@String(PAMI = {IEEE Transactions on Pattern Analysis and Machine Intelligence})

@String(CVPR = {IEEE/CVF Conference on Computer Vision and Pattern Recognition})

@String(NIPS = {Advances in Neural Information Processing Systems})

@String(BMVC = {British Machine Vision Conference})

@String(ICLR = {International Conference on Learning Representations})

@String(ICML = {International Conference on Machine Learning})

@String(AISTATS = {International Conference on Artificial Intelligence and Statistics})

@String(JMLR = {Journal of Machine Learning Research})

@String(UAI = {Conference on Uncertainty in Artificial Intelligence})

@String(CoRL = {Conference on Robot Learning})

@String(DICTA = {Digital Image Computing: Techniques and Applications})

@String(ICRA = {IEEE International Conference on Robotics and Automation})

@inproceedings{ilse2020selecting,
  title     = {Selecting data augmentation for simulating interventions},
  author    = {Ilse, Maximilian and Tomczak, Jakub M. and Forr{\'e}, Patrick},
  booktitle = ICML,
  year      = {2021},
volume = 	 {139}
}

@article{anchor,
  title     = {Anchor regression: {H}eterogeneous data meet causality},
  author    = {Rothenh{\"a}usler, Dominik and Meinshausen, Nicolai and B{\"u}hlmann, Peter and Peters, Jonas},
  journal   = {Journal of the Royal Statistical Society Series B: Statistical Methodology},
  volume    = {83},
  number    = {2},
  pages     = {215--246},
  year      = {2021},
  publisher = {Oxford University Press}
}

@inproceedings{riskextrap,
  title     = {Out-of-distribution generalization via risk extrapolation ({{REx}})},
  author    = {Krueger, David and Caballero, Ethan and Jacobsen, Joern-Henrik and Zhang, Amy and Binas, Jonathan and Zhang, Dinghuai and Le Priol, Remi and Courville, Aaron},
  booktitle = ICML,
  year      = {2021},
 }

@inproceedings{15,
  title     = {Invariant Causal Representation Learning for Out-of-Distribution Generalization},
  author    = {Lu, Chaochao and Wu, Yuhuai and Hern{\'a}ndez-Lobato, Jos{\'e} Miguel and Sch{\"o}lkopf, Bernhard},
  booktitle = ICLR,
  year      = {2022},
}

@book{pgms,
  title     = {Probabilistic graphical models: principles and techniques},
  author    = {Koller, Daphne and Friedman, Nir},
  year      = {2009},
  publisher = {MIT press}
}

@incollection{gmm,
  title     = {Generalized method of moments},
  author    = {Hall, Alastair R.},
  booktitle = {A Companion to Theoretical Econometrics},
  pages     = {230--255},
chapter = {11},
doi = {10.1002/9780470996249.ch12},
  year      = {2003},
  publisher = {Wiley}
}

@misc{gan-gmm,
  title  = {Adversarial generalized method of moments},
  author = {Lewis, Greg and Syrgkanis, Vasilis},
  year   = {2018},
  note   = {arXiv:1803.07164}
}

@inproceedings{deep-gmm,
  author    = {Bennett, Andrew and Kallus, Nathan and Schnabel, Tobias},
  title     = {Deep generalized method of moments for instrumental variable analysis},
  year      = {2019},
  booktitle = NIPS,
  volume    = {32}
}

@misc{arjovsky2020invariant,
  title  = {Invariant risk minimization},
  author = {Arjovsky, Martin and Bottou, L{\'e}on and Gulrajani, Ishaan and Lopez-Paz, David},
  year   = {2019},
  note   = {arXiv:1907.02893}
}

@article{optical-device,
  title     = {Detecting confounding in multivariate linear models via spectral analysis},
  author    = {Janzing, Dominik and Sch{\"o}lkopf, Bernhard},
  journal   = {Journal of Causal Inference},
  volume    = {6},
  number    = {1},
  year      = {2018},
  publisher = {De Gruyter}
}

@inproceedings{dominik-optical-device-causal-reg2,
  title     = {Detecting non-causal artifacts in multivariate linear regression models},
  author    = {Janzing, Dominik and Sch{\"o}lkopf, Bernhard},
  booktitle = ICML,
  year      = {2018},
  volume    = {80},
}

@inproceedings{causal-regularization,
  title     = {Causal regularization},
  author    = {Janzing, Dominik},
  booktitle = NIPS,
  volume    = {32},
  year      = {2019}
}

@misc{causal-regularization2,
  title  = {Causal Regularization: On the trade-off between in-sample risk and out-of-sample risk guarantees}, 
  author = {Kania, Lucas and Wit, Ernst},
  year   = {2023},
  note   = {arXiv:2205.01593}
}

@inproceedings{causal-interpolation-regularization,
  title     = {Interpolation and Regularization for Causal Learning},
  author    = {Vankadara, Chennuru and Rendsburg, Luca and von Luxburg, Ulrike and Ghoshdastidar, Debarghya},
  booktitle = NIPS,
  volume    = {35},
  year      = {2022},
}

@book{eoci,
  title     = {Elements of causal inference: Foundations and learning algorithms},
  author    = {Peters, Jonas and Janzing, Dominik and Sch{\"o}lkopf, Bernhard},
  year      = {2017},
  publisher = {The MIT Press}
}

@book{econometric-methods,
  title     = {Econometric Methods},
  author    = {Johnston, John},
  edition   = {Second},
  year      = {1971},
  publisher = {McGraw-Hill}
}

@book{econAnalysis,
  title     = {Econometric analysis},
  author    = {Greene, William H.},
  year      = {2003},
  publisher={Prentice Hall},
isbn={9780130661890},
  lccn={20029308}
}

@book{Pearl_2009,
  title     = {Causality},
  author    = {Pearl, Judea},
  year      = {2009},
  publisher = {Cambridge University Press}
}

@book{vapnik,
  title     = {Statistical learning theory},
  author    = {Vapnik, Vladimir Naumovich},
  year      = {1998},
  publisher = {Wiley}
}

@inproceedings{bo,
  title     = {Learning from Conditional Distributions via Dual Embeddings},
  author    = {Dai, Bo and He, Niao and Pan, Yunpeng and Boots, Byron and Song, Le},
  booktitle = AISTATS,
  year      = {2017},
  volume    = {54},
}

@misc{backdoor,
  title  = {A neural mean embedding approach for back-door and front-door adjustment},
  author = {Xu, Liyuan and Gretton, Arthur},
  year   = {2022},
  note   = {arXiv:2210.06610}
}

@article{causal_invariance_nonlinear,
  title     = {Invariant causal prediction for nonlinear models},
  author    = {Heinze-Deml, Christina and Peters, Jonas and Meinshausen, Nicolai},
  journal   = {Journal of Causal Inference},
  volume    = {6},
  number    = {2},
  year      = {2018},
doi = {doi:10.1515/jci-2017-0016},
}

@inproceedings{iv_arthur,
  title     = {Kernel instrumental variable regression},
  author    = {Singh, Rahul and Sahani, Maneesh and Gretton, Arthur},
  booktitle = NIPS,
  volume    = {32},
  year      = {2019}
}

@article{iv_rui,
  title     = {Instrumental variable regression via kernel maximum moment loss},
  author    = {Zhang, Rui and Imaizumi, Masaaki and Sch{\"o}lkopf, Bernhard and Muandet, Krikamol},
  journal   = {Journal of Causal Inference},
  volume    = {11},
  number    = {1},
  year      = {2023},
doi = {10.1515/jci-2022-0073},
}

@inproceedings{iv_niki,
  title     = {A class of algorithms for general instrumental variable models},
  author    = {Kilbertus, Niki and Kusner, Matt J. and Silva, Ricardo},
  booktitle = NIPS,
  volume    = {33},
  year      = {2020}
}

@article{da-reg,
  title   = {A survey on Image Data Augmentation for Deep Learning},
  author  = {Shorten, Connor and Khoshgoftaar, Taghi M.},
  journal={Journal of Big Data},
  year={2019},
  volume={6},
  pages={60},
  doi={10.1186/s40537-019-0197-0}
}

@inproceedings{
deep-iv-arthur,
title={Learning Deep Features in Instrumental Variable Regression},
author={Liyuan Xu and Yutian Chen and Siddarth Srinivasan and Nando de Freitas and Arnaud Doucet and Arthur Gretton},
booktitle={International Conference on Learning Representations},
year={2021},
}

@article{icp,
  author  = {Peters, Jonas and Bühlmann, Peter and Meinshausen, Nicolai},
  title   = {Causal Inference by using Invariant Prediction: Identification and Confidence Intervals},
  journal = {Journal of the Royal Statistical Society Series B: Statistical Methodology},
  volume  = {78},
  number  = {5},
  pages   = {947--1012},
  year    = {2016},
  doi     = {10.1111/rssb.12167}
}

@inproceedings{dro,
  title     = {Distributionally Robust Neural Networks},
  author    = {Sagawa, Shiori and Koh, Pang Wei and Hashimoto, Tatsunori B. and Liang, Percy},
  booktitle = ICLR,
  year      = {2020},
}

@inproceedings{rice,
  author    = {Wang, Ruoyu and Yi, Mingyang and Chen, Zhitang and Zhu, Shengyu},
  booktitle = CVPR, 
  title     = {Out-of-distribution Generalization with Causal Invariant Transformations}, 
  year      = {2022},
  doi       = {10.1109/CVPR52688.2022.00047}
}

@article{chain-graph,
  author  = {Lauritzen, Steffen L. and Richardson, Thomas S.},
  title   = {Chain Graph Models and their Causal Interpretations},
  journal = {Journal of the Royal Statistical Society Series B: Statistical Methodology},
  volume  = {64},
  number  = {3},
  pages   = {321--348},
  year    = {2002},
  doi     = {10.1111/1467-9868.00340}
}

@inproceedings{discovering-cyclic-scm,
  author    = {Lacerda, Gustavo and Spirtes, Peter and Ramsey, Joseph and Hoyer, Patrik O.},
  title     = {Discovering cyclic causal models by independent components analysis},
  year      = {2008},
  publisher = {AUAI Press},
  booktitle = UAI,
  pages     = {366--374}
}

@article{linear-cyclic-sem,
  author  = {Hyttinen, Antti and Eberhardt, Frederick and Hoyer, Patrik O.},
  title   = {Learning Linear Cyclic Causal Models with Latent Variables},
  journal = JMLR,
  year    = {2012},
  volume  = {13},
  pages   = {3387--3439}
}

@inproceedings{additive-nonlinear-cyclic-sem,
  author    = {Mooij, Joris M. and Janzing, Dominik and Heskes, Tom and Sch{\"o}lkopf, Bernhard},
  booktitle = NIPS,
  title     = {On Causal Discovery with Cyclic Additive Noise Models},
  volume    = {24},
  year      = {2011}
}

@article{general-cyclic-sem,
  title   = {Foundations of structural causal models with cycles and latent variables},
  volume  = {49},
  doi     = {10.1214/21-AOS2064},
  number  = {5},
  journal = {The Annals of Statistics},
  author  = {Bongers, Stephan and Forr{\'e}, Patrick and Peters, Jonas and Mooij, Joris M.},
  year    = {2021}
}

@article{cowles-contributions-econometrics,
  author    = {Christ, Carl F.},
  journal   = {Journal of Economic Literature},
  number    = {1},
  pages     = {30--59},
  publisher = {American Economic Association},
  title     = {The {{Cowles Commission's}} Contributions to Econometrics at {{Chicago}}, 1939-1955},
  volume    = {32},
  year      = {1994}
}

@misc{cocycle,
  title  = {Counterfactual Cocycles: A Framework for Robust and Coherent Counterfactual Transports}, 
  author = {Dance, Hugh and Bloem-Reddy, Benjamin},
  year   = {2025},
  note   = {arXiv:2405.13844}
}

@article{cyclic-sem-iv,
  author  = {John Fox},
  title   = {Simultaneous Equation Models and Two-Stage Least Squares},
  journal = {Sociological Methodology},
  volume  = {10},
  year    = {1979},
  pages   = {130--150},
  doi     = {10.2307/270769}
}

@article{cobweb-1,
  title   = {The {Cobweb} Theorem},
  author  = {Ezekiel, Mordecai},
  year    = {1938},
  journal = {The Quarterly Journal of Economics},
  volume  = {52},
  number  = {2},
doi = {10.2307/1881734}
}

@article{cobweb-2,
  author    = {Muth, John F.},
  journal   = {Econometrica},
  number    = {3},
  pages     = {315--335},
  publisher = {The Econometric Society},
  title     = {Rational Expectations and the Theory of Price Movements},
  volume    = {29},
  year      = {1961}
}

@article{2SLS-3SLS,
  title     = {Two-or three-stage least squares?},
  author    = {Belsley, David A.},
  journal   = {Computer Science in Economics and Management},
  volume    = {1},
  pages     = {21--30},
  year      = {1988},
doi = {10.1007/BF00435200}
}

@article{3SLS,
  author    = {Zellner, Arnold and Theil, H.},
  journal   = {Econometrica},
  number    = {1},
  pages     = {54--78},
  title     = {Three-Stage Least Squares: Simultaneous Estimation of Simultaneous Equations},
  volume    = {30},
  year      = {1962},
doi={10.2307/1911287},
}

@article{group-theory-da-jmlr,
  author  = {Chen, Shuxiao and Dobriban, Edgar and Lee, Jane H.},
  title   = {A Group-Theoretic Framework for Data Augmentation},
  journal = JMLR,
  year    = {2020},
  volume  = {21},
  number  = {245},
  pages   = {1--71}
}

@misc{invariance_benefit,
      title={On the Benefits of Invariance in Neural Networks}, 
      author={Clare Lyle and Mark van der Wilk and Marta Kwiatkowska and Yarin Gal and Benjamin Bloem-Reddy},
      year={2020},
      note={arXiv:2005.00178},
}

@inproceedings{fanny,
  title     = {Invariance-inducing regularization using worst-case transformations suffices to boost accuracy and spatial robustness},
  author    = {Yang, Fanny and Wang, Zuowen and Heinze-Deml, Christina},
  booktitle = NIPS,
  volume    = {32},
  year      = {2019}
}

@InProceedings{zz1-droid,
  title = 	 {{{DROID}}: Learning from Offline Heterogeneous Demonstrations via Reward-Policy Distillation},
  author =       {Jayanthi, Sravan and Chen, Letian and Balabanska, Nadya and Duong, Van and Scarlatescu, Erik and Ameperosa, Ezra and Zaidi, Zulfiqar Haider and Martin, Daniel and Matto, Taylor Keith Del and Ono, Masahiro and Gombolay, Matthew},
  booktitle = 	 CoRL,
  year = 	 {2023},
  volume = 	 {229},
  publisher =    {PMLR},
}

@INPROCEEDINGS{owb-comain-rl,
  author={Han, Seunghyup and Bhatti, Osama Waqar and Na, Woo-Jin and Swaminathan, Madhavan},
  booktitle={2023 IEEE MTT-S International Conference on Numerical Electromagnetic and Multiphysics Modeling and Optimization (NEMO)}, 
  title={Reinforcement Learning Applied to the Optimization of Power Delivery Networks with Multiple Voltage Domains},
  year={2023},
  doi={10.1109/NEMO56117.2023.10202224}
}

@article{omitted-variable-bias,
  author  = {Clarke, Kevin A.},
  title   = {The {Phantom Menace}: Omitted Variable Bias in Econometric Research},
  journal = {Conflict Management and Peace Science},
  volume  = {22},
  number  = {4},
  pages   = {341--352},
  year    = {2005},
  doi     = {10.1080/07388940500339183}
}

@book{RO,
  author    = {Ben-Tal, Aharon and El Ghaoui, Laurent and Nemirovski, Arkadi},
  title     = {Robust Optimization},
  year      = {2009},
  publisher = {Princeton University Press},
  doi       = {10.1515/9781400831050}
}

@inproceedings{notJustPrettyPictures,
  title     = {Not {Just Pretty Pictures}: Toward Interventional Data Augmentation Using Text-to-Image Generators},
  author    = {Yuan, Jianhao and Pinto, Francesco and Davies, Adam and Torr, Philip},
  booktitle = ICML,
  year      = {2024},
}

@inproceedings{causalDA4LLM,
  author    = {Feder, Amir and Wald, Yoav and Shi, Claudia and Saria, Suchi and Blei, David},
  booktitle = NIPS,
  title     = {Data Augmentations for Improved (Large) Language Model Generalization},
  volume    = {36},
  year      = {2023}
}

@inproceedings{mocoda,
  title     = {{{MoCoDA}}: Model-based Counterfactual Data Augmentation},
  author    = {Pitis, Silviu and Creager, Elliot and Mandlekar, Ajay and Garg, Animesh},
  booktitle = NIPS,
  volume    = {35},
  year      = {2022}
}

@inproceedings{caiac,
  title = 	 {Causal Action Influence Aware Counterfactual Data Augmentation},
  author =       {Armengol Urp\'{\i}, N\'{u}ria and Bagatella, Marco and Vlastelica, Marin and Martius, Georg},
  booktitle = ICML,
  year      = {2024},
volume = {235},
}

@inproceedings{causal-matching,
  title     = {Domain Generalization using Causal Matching},
  author    = {Mahajan, Divyat and Tople, Shruti and Sharma, Amit},
  booktitle = ICML,
  year      = {2021},
}

@InProceedings{contrastive,
  title = 	 {{CATE} Estimation With Potential Outcome Imputation From Local Regression},
  author =       {Aloui, Ahmed and Dong, Juncheng and Le, Cat Phuoc and Tarokh, Vahid},
  booktitle = 	 UAI,
  year = 	 {2025},
  volume = 	 {286},
}

@article{causal-ood-generalizatoin,
  author  = {Christiansen, Rune and Pfister, Niklas and Jakobsen, Martin Emil and Gnecco, Nicola and Peters, Jonas},
  journal = PAMI, 
  title   = {A Causal Framework for Distribution Generalization}, 
  year    = {2022},
  volume  = {44},
  number  = {10},
  pages   = {6614--6630},
  doi     = {10.1109/TPAMI.2021.3094760}
}

@book{causal-identifiability-ate,
  author    = {Wooldridge, Jefrey M.},
  publisher = {The MIT Press},
  title     = {Econometric Analysis of Cross Section and Panel Data},
  year      = {2010}
}

@inproceedings{prox-causal-learning,
  title     = {Proximal Causal Learning with Kernels: Two-Stage Estimation and Moment Restriction},
  booktitle = ICML,
  volume    = {139},
  year      = {2021},
  author    = {Mastouri, Arash and Zhu, Yuhang and Gultchin, Limor and Korba, Anna and Silva, Ricardo and Kusner, Matt and Gretton, Arthur and Muandet, Krikamol}
}

@book{Horn_Johnson_1985,
  title     = {Matrix Analysis},
  publisher = {Cambridge University Press},
  author    = {Horn, Roger A. and Johnson, Charles R.},
  year      = {1985}
}

@book{matrix-math,
  author    = {Bernstein, Dennis S.},
  publisher = {Princeton University Press},
  title     = {Matrix Mathematics: Theory, Facts, and Formulas},
  edition   = {Second},
  year      = {2009}
}

@incollection{trinity,
  booktitle = {Handbook of the Economics of Finance},
  title     = {Endogeneity in Empirical Corporate Finance},
  publisher = {Elsevier},
  volume    = {2},
chapter = {7},
  pages     = {493--572},
  year      = {2013},
  doi       = {10.1016/B978-0-44-453594-8.00007-0},
  author    = {Roberts, Michael R. and Whited, Toni M.}
}

@misc{bias-variance,
  title        = {Bias-variance decompositions: The exclusive privilege of {Bregman} divergences}, 
  author       = {Tom Heskes},
  year         = {2025},
  eprint       = {2501.18581},
  archivePrefix= {arXiv},
  primaryClass = {cs.LG},
  note          = {arXiv:2501.18581}, 
}

@misc{when-shift-happens,
  title        = {When Shift Happens - Confounding Is to Blame},
  author       = {Abbavaram Gowtham Reddy and Celia Rubio-Madrigal and Rebekka Burkholz and Krikamol Muandet},
  year         = {2025},
  note          = {arXiv:2505.21422}, 
}

@article{ac-vocal,
title = {Characterizing Vocal Hyperfunction Using Ecological Momentary Assessment of Relative Fundamental Frequency},
journal = {Journal of Voice},
year = {2024},
issn = {0892-1997},
doi = {10.1016/j.jvoice.2024.10.025},
author = {Ahsan J. Cheema and Katherine L. Marks and Hamzeh Ghasemzadeh and Jarrad H. {Van Stan} and Robert E. Hillman and Daryush D. Mehta},
}

@inproceedings{lost-domain-generalization,
title={In Search of Lost Domain Generalization},
author={Ishaan Gulrajani and David Lopez-Paz},
booktitle=ICLR,
year={2021},
}

@inproceedings{specIV,
title={Spectral Representation for Causal Estimation with Hidden Confounders},
author={Tongzheng Ren and Haotian Sun and Antoine Moulin and Arthur Gretton and Bo Dai},
booktitle=AISTATS,
year={2025},
}

@InProceedings{regularizing-towards-causal-invariance,
  title = 	 {Regularizing towards Causal Invariance: Linear Models with Proxies},
  author =       {Oberst, Michael and Thams, Nikolaj and Peters, Jonas and Sontag, David},
  booktitle = 	 ICML,
  year = 	 {2021},
  volume = 	 {139},
}

@inproceedings{3DA,
  author    = {H. Shao and others},
  title     = {A theory of {{PAC}} learnability under transformation invariances},
  booktitle = NIPS,
  year      = {2022}
}

@inproceedings{4DA,
  author    = {A. Fawzi and P. Frossard},
  title     = {Manitest: Are classifiers really invariant?},
  booktitle = BMVC,
  year      = {2015}
}

@inproceedings{5DA,
  author    = {Y. Dubois and others},
  title     = {Lossy compression for lossless prediction},
  booktitle = NIPS,
  year      = {2021}
}

@inproceedings{6DA,
  author    = {M. Petrache and S. Trivedi},
  title     = {Approximation-generalization trade-offs under (approximate) group equivariance},
  booktitle = NIPS,
  year      = {2023}
}

@inproceedings{7DA,
  author    = {O. Montasser and others},
  title     = {Transformation-invariant learning and theoretical guarantees for {{OOD}} generalization},
  booktitle = NIPS,
volume = {37},
  year      = {2024}
}

@inproceedings{8DA,
  author    = {D. Romero and S. Lohit},
  title     = {Learning partial equivariances from data},
  booktitle = NIPS,
  year      = {2022}
}

@inproceedings{9DA,
  author    = {S. Zhu and others},
  title     = {Understanding the generalization benefit of model invariance from a data perspective},
  booktitle = NIPS,
  year      = {2021},
volume = {34}
}

@inproceedings{10DA,
  author    = {S. Wong and others},
  title     = {Understanding data augmentation for classification: {When to warp?}},
  booktitle = DICTA,
  year      = {2016}
}

@article{deconfounding-causal-regularization,
  author    = {Peter B{\"u}hlmann and Dominik Cevid},
  title     = {Deconfounding and Causal Regularisation for Stability and External Validity},
  journal   = {International Statistical Review},
  volume    = {88},
  number    = {S1},
  pages     = {S114--S134},
  year      = {2020},
  doi       = {10.1111/insr.12383}
}

@conference{dual-iv,
  title = {Dual Instrumental Variable Regression},
  booktitle = NIPS,
  year = {2020},
  author = {Muandet, K. and Mehrjou, A. and Lee, S. K. and Raj, A.},
  volume = {33}
}

@book{domination,
  title     = {Theory of {Point Estimation}},
  author    = {Lehmann, Erich L. and Casella, George},
  edition   = {2nd},
  year      = {1998},
  publisher = {Springer}
}

@InProceedings{krikamol-domain-generalization,
  title = 	 {Domain Generalization via Invariant Feature Representation},
  author = 	 {Muandet, Krikamol and Balduzzi, David and Schölkopf, Bernhard},
  booktitle = 	 ICML,
  year = 	 {2013},
  volume = 	 {28},
}

@inproceedings{savkin2020adversarial,
  author    = {Savkin, Artem and Lapotre, Thomas and Strauss, Kevin and Akbar, Uzair and Tombari, Federico},
  booktitle = ICRA,
  title     = {Adversarial Appearance Learning in Augmented {Cityscapes} for Pedestrian Recognition in Autonomous Driving},
  year      = {2020},
  pages     = {3305--3311},
  doi       = {10.1109/ICRA40945.2020.9197024}
}
}

\newpage
\appendix

\appendixmaketitle{Appendix---\paperTitle}

{
\hypersetup{linkcolor=blue}
\startcontents[appendices]
\printcontents[appendices]{l}{1}{\section*{Contents}\setcounter{tocdepth}{2}}
}


\newpage

\nomenclature[a02]{$\R^{n\times \ast}$}{$n\times \ast$ Euclidean space; dimension $\ast$ conformal with \& inferred from context.}
\nomenclature[a03]{$x$}{Scalar.}
\nomenclature[a04]{$\vx$}{Vector. When $\vx^\top$ is described as a vector, it means $\vx$ is a flat $1\times\ast$ matrix.}
\nomenclature[a05]{$\mathbf{X}$}{Matrix.}
\nomenclature[a06]{$\sX$}{Set.}
\nomenclature[a06]{$X$}{Random vector.}

\nomenclature[a07]{$\M$}{SEM.}
\nomenclature[a07.5]{$X^{\M}$}{Random vector $X$ with its SEM $\M$ specified when unclear from context.}
\nomenclature[a08]{$\distribution{X}{\M}$}{Distribution of $X$ entailed by $\M$. Superscript dropped if clear~from~context.~}
\nomenclature[a11.75]{$\EE{}{\M}{X}$}{Expected value of $X$ under distribution $\distribution{X}{\M}$.}
\nomenclature[a11]{$\cov{X}{\M}$}{Variance--covariance matrix of $X$ under distribution $\distribution{X}{\M}$.}
\nomenclature[a11.5]{$\cov{X, Y}{\M}$}{Cross--covariance matrix of $X$ and $Y$ under distribution $\distribution{X, Y}{\M}$.}

\nomenclature[a12]{$\intervention{X}{\vx}$}{Intervention---$X$ is set to $\vx$ during data generation.}

\nomenclature[a13]{$\intervention{X}$}{Shorthand for $\intervention{X}{X^{\prime}}$ where $X^{\prime} \sim \distribution{X}{\M}$ is i.i.d. to $X$.}

\nomenclature[a14]{$\doM{X}{\vx}$}{Intervention SEM.}

\nomenclature[a15]{$\M_{X = \vx}$}{SEM with mechanisms of $\M$, but exogenous~noise~distribution~$\distribution{N\mid X=\vx}{\M}$.~}

\nomenclature[a16]{$\counterfactualM{Y = \vy}{X}{\vx}$}{Counterfactual SEM---intervention SEM of $\M_{Y = \vy}$.}

\nomenclature[a17]{$X \indep Y$}{Random vectors $X, Y$ are statistically independent, i.e. $\distribution{Y\mid X}{\M} = \distribution{Y}{\M}$.}

\nomenclature[a18]{$\vx \perp \vy$}{$\vx, \vy$ are perpendicular, i.e. $\vx^\top \vy = 0$. For~random~vectors,~$X^\top Y=0$~a.s.~}

\nomenclature[a19]{$\hat{h}^{\M}$}{Population/ infinite-sample estimate based on distribution $\distribution{}{\M}$.}

\nomenclature[a20]{$\hat{h}^{\sD}$}{Finite-sample estimate based on samples in the dataset $\sD$.}


\printnomenclature

\newpage
\section{Confounding Bias}\label{app:confounding-bias}

\paragraph{Statistical vs. causal inference.} The target estimand for the statistical risk in \cref{eq:erm} is the Bayes optimal predictor $\EE{}{\M}{Y\given X = \vx}$. And the target estimand for the causal risk in \cref{eq:causal-risk} is the average treatment effect (ATE) $\EE{}{ \doM{X}{ \vx } }{ Y\given X = \vx } = \func{f}{\vx}$. As such, \emph{statistical inference} is concerned with \emph{predictions} of outcome $Y$, whereas \emph{causal inference} is concerned with~\emph{estimating}~ $\func{f}{\vx}$.

\paragraph{Statistical vs. confounding bias.} Both types of inference are subject to bias. \emph{Statistical bias} arises due to misspecification of the hypothesis class $\sH$, whereas confounding bias arises due to how the data are generated. The former is therefore a property of the estimator while the later is a property of the data itself. For an estimator $\hat{h}^{\sD}$ with the expected value $\bar{h}(\cdot)=\EE{\sD}{\M}{\hat{h}^{\sD}(\cdot)}$, we define these as
\begin{align*}
    \text{Statistical bias at $\vx$}\coloneqq&\; \EE{}{\M}{Y\given X = \vx} - \bar{h}(\vx),\\
    \text{Confounding bias at $\vx$}\coloneqq&\; f(\vx) - \EE{}{\M}{Y\given X = \vx} .
\end{align*}

\paragraph{Bias-variance decomposition of the causal risk.} Because the treatment $X$ and residual $\xi$ are not correlated under $\doM{X}$ in \cref{eq:canonical-sem}, for any loss function $\ell$ that admits a `clean' or `additive' bias-variance decomposition \cite{bias-variance}, the causal risk in \cref{eq:causal-risk} also admits a bias-variance decomposition. Using squared loss as an example, we have for some hypothesis $\hat{h}^{\sD}$,
\begin{align*}
    &\Rightarrow \Risk{\causal}{\M}{ \hat{h}^{\sD} }
    \\
    &= \EE{}{\doM{X}}{ \sqNorm{ Y - \func{\hat{h}^{\sD}}{X} } } 
    , 
    \\
    \tag{Structural eq. of $Y$.}
    &= \EE{}{\doM{X}}{ \sqNorm{ f(X) + \xi - \func{\hat{h}^{\sD}}{X} } }
    , 
    \\
    \tag{Cross term is 0 as $\xi\indep X^{\doM{X}}$.}
    &= 
    \EE{}{\doM{X}}{ \sqNorm{ \xi } }
    +
    \EE{}{\doM{X}}{ \sqNorm{
    f(X) - \func{\hat{h}^{\sD}}{X}
    } }
    ,
    \\
    \tag{$\distribution{X}{\M}$, $\distribution{X}{\doM{X}}$ identical by construction.}
    &= 
    \underbrace{
    \EE{}{\doM{X}}{ \sqNorm{ \xi } }
    }_{\text{irreducible noise}}
    +
    \underbrace{
    \EE{}{\M}{ \sqNorm{
    f(X) - \func{\hat{h}^{\sD}}{X}
    } }
    }_{\text{estimation error, }\CER{\M}{\hat{h}^{\sD}}=} .
\end{align*}
We can show by following standard procedure that
\begin{align*}
    \EE{\sD}{\M}{\CER{\M}{ \hat{h}^{\sD} }}
    &= 
    \underbrace{
    \EE{X}{\M}{ \sqNorm{ f(X) - \func{\bar{h}}{X} } }
    }_{\text{(average) bias}^2}
    +
    \underbrace{
    \EE{\sD}{\M}{ \EE{X}{\M}{ \sqNorm{ \func{\bar{h}}{X} - \func{\hat{h}^{\sD}}{X} } } }
    }_{\text{variance}}
    .
\end{align*}
Since $\func{\hat{h}^{\M}}{X} = \func{\bar{h}}{X} $ for any population estimate, CER equals the average squared~estimation~bias
\begin{align*}
    \CER{\M}{ \hat{h}^{\M} } 
    =
    \EE{X}{\M}{ \sqNorm{
    f(X) - \func{\hat{h}^{\M}}{X}
    } }
    =
    \EE{X}{\M}{ \sqNorm{
    f(X) - \func{\bar{h}}{X} 
    } }
    .
\end{align*}
For a rich enough hypothesis class, the ERM estimate coincides with the Bayes optimal predictor $\func{\hat{h}^{\M}_{\erm}}{\cdot} = \EE{}{\M}{Y\given X = \cdot}$ and the CER exactly equals the (average squared) confounding bias as we define it above. For a general estimate $\hat{h}^{\sD}$, however, the CER also contains statistical bias. Nevertheless, our claims of ``better causal estimation via reducing confounding bias'' rest on the fact that we are essentially manipulating the data via DA and/or using treatment randomization sources in the form of IVLs. And recall that confounding bias is a property of the data.

\newpage
\section{Simultaneity as Cyclic Structures in Equilibrium}\label{app:cyclic-sem}

\subsection*{Linear cyclic assignments}
SEMs with cyclic structures have been well studied both in the linear case \cite{chain-graph,discovering-cyclic-scm,linear-cyclic-sem}, as well as the non-linear case \cite{additive-nonlinear-cyclic-sem,general-cyclic-sem}. Here we briefly provide a causal interpretation to linear simultaneous equations as SEMs with cyclic assignments.

Consider a square matrix $\T \in \R^{d\times d}$ and the SEM
\begin{align}\label{eq:sem-structural-form}
    W = \T W + N\;,
\end{align}
where random noise vector $N$ is exogenous and $\T$ allows for a cyclic structure. We enforce $\pr*{ \In{d} - \T }$ to be invertible so that the above equation has a unique solution $W$ for any given $N$. Re-writing the \textit{structural form} in \cref{eq:sem-structural-form} into a \textit{reduced form}, the distribution over $W$ is defined by
\begin{align}\label{eq:sem-reduced-form}
    W = \pr*{ \In{d} - \T }^{-1}N\;.
\end{align}
One way we can present a causal interpretation of the above solution is to view it as a stationary point to the following sequence of random vectors $W_t$
\begin{align*}
    W_t = \T W_{t-1} + N\;,
\end{align*}
which converges if $\T$ has a spectral norm strictly smaller than one so that $\T^{t}\rightarrow 0$ as $t\rightarrow \infty$. The structural form \cref{eq:sem-structural-form} essentially describes the iterative application of this operation. And in the limit the distribution of $\lim_{t\rightarrow\infty}W^t$ will be the same as the reduced form \cref{eq:sem-reduced-form}. Although equivalent, reduced form of a cyclic SEM (if one exists) obscures the causal relations in the data generation process.

Furthermore, we restrict our models to not have any ``self-cycles'' (an edge from a vertex to itself). So, e.g., the matrix $\T$ in \cref{eq:sem-structural-form} has all zero diagonal entries. This not only simplifies our analysis by providing a simple and intuitive interpretation for our definition of DA in \cref{sec:data-augmentation}, but it also ensures that non-linear SEMs entail unique, well-defined distributions under mild assumptions \cite{general-cyclic-sem,discovering-cyclic-scm}.

Similarly we can write the example SEM $\M$ from \cref{example:ivl} in this (block matrix) form as
\begin{align*}
    \underbrace{\begin{bmatrix}
        X\\
        Y
    \end{bmatrix}}_{W}
    &=
    \underbrace{\begin{bmatrix}
        \0{m}{m} & \vtau^\top\\
        \f^\top & \0{1}{1}
    \end{bmatrix}}_{\T}
    \underbrace{\begin{bmatrix}
        X\\
        Y
    \end{bmatrix}}_{W}
     +
     \underbrace{
     \begin{bmatrix}
        \K^\top\\
        \0{1}{k}
    \end{bmatrix}
    Z
    +
    \begin{bmatrix}
        \E^\top \\
        \veps^\top
    \end{bmatrix}
    C
     +
     \sigma\cdot
     \begin{bmatrix}
        N_X\\
        N_Y
    \end{bmatrix}
    }_{N}
    ,
\end{align*}
For this simple case, $\pr*{\In{(m+1)} - \T}$ is always invertible so long as $\f^\top\vtau^\top \neq 1$ from \cref{lemma:sem-solvability}. Or we can also restrict $\abs*{ \f^\top\vtau^\top } < 1$ to ensure that the spectral norm of $\T$ is strictly smaller than 1. We sample from this SEM by first sampling all of the exogenous variables $Z, C, N_X, N_Y$ and then solving the above system for each sample of $X, Y$ via the reduced form in \cref{lemma:sem-solvability}.

\subsection*{A motivating example}
 Cyclic SEMs were first discussed in the econometrics literature \cite{cowles-contributions-econometrics} to model various observational phenomena, and often solved via 2SLS based IV regression \cite{cyclic-sem-iv} since it is computationally less costly compared to solving the entire system \cite{2SLS-3SLS}. A classic example from economics \cite{cobweb-1,cobweb-2} is that of a \textit{supply and demand model} $\M$ where the relation of price $P$ of a good with quantity $Q$ of demand can be thought of as a cyclic feed-back loop where producers adjust their price in response to demand of the good and consumers change their demand in response to price of a good. In contrast, a change in consumer tastes or preferences would be an exogenous change on the demand curve and can therefore be used as an IV $Z$.
\begin{align*}
\text{consumer demand:} &\qquad Q = \tau \cdot P + \gamma\cdot Z + N_Q\;,\\
    \text{producer price:} &\qquad P = f\cdot Q + N_P\;.
\end{align*}
Where scalars $f,\tau$ are such that $\abs*{ f\cdot\tau } < 1$ so that the system converges to an equilibrium. We say that the measurements made for $P$ and $Q$ are at the equilibrium state of the market\footnote{In fact, such a feed-back model of supply and demand was initially developed to understand the irregular fluctuations of prices/quantities that are observed in some markets when not at equilibrium \cite{cobweb-1}.} with zero mean measurement noise $N_P, N_Q$ respectively.

\paragraph{Mitigating simultaneity bias for causal effect estimation.} If we now want to \textit{estimate} the effect of demand on price $f$, standard regression will produce a biased estimate $\hat{f}^\M_{\erm} = f + \frac{\text{Cov}(Q, N_P)}{\VV{}{}{Q}}$ because of the simultaneity causing $Q$ and $N_P$ to be correlated (to see this, substitute model of $P$ into the model of $Q$). We can now use IV regression to get an unbiased estimate of the effect of demand on price in the market as $\hat{f}^\M_{\iv}=f$.

\paragraph{Mitigating spurious correlations for robust prediction.} Similarly, if the producer wants to \textit{predict} the effect on demand if price is changed (i.e. intervened on), naive ERM will not be a good choice because it will also capture the spurious correlation from $Q\rightarrow P$. We therefore use three-stage-least-squares (3SLS) \cite{3SLS,2SLS-3SLS} (or similar methods) to estimate the ATE $\hat{\tau}^\M_{\text{3SLS}}=\EE{}{ \doM{P}{.} }{ Q \given P= . }$ where we use the first two stages to estimate $\hat{f}^\M_{\iv}$, followed by ERM to regress from the residuals $\hat{N}_P \coloneqq P - \hat{f}^\M_{\iv}\cdot Q$ to $Q$ in the third stage.



\subsection*{Implications for independence of causal mechanisms}

Here we clarify how the equilibrium assumption/interpretation of cyclic SEMs is not at odds with the classic independent causal mechanism (ICM) principle \cite{eoci}. Note that our SEM formulation in \cref{eq:canonical-sem} is a direct instantiation of the ICM principle as described by Peters et al. \cite{eoci}. The two equations represent the autonomous mechanisms, and their independence is captured by the mutual independence of the exogenous noise terms $N_X, N_Y$. The simultaneity in our model is not a violation of ICM, but rather the equilibrium state resulting from the interaction of these two independent mechanisms. Assuming the existence of this equilibrium is a statement about the scope of systems under analysis, and not about the nature of the mechanisms themselves. Indeed, surgically changing $\tau$ to some $\tau^{\prime}$, for example, does not in itself alter $f$ and vice versa. And precisely because of the ICM, this may or may not make the system unstable depending on the nature of $\tau^{\prime}$. Nevertheless, in our setting, \cref{proposition:da-intervention-distribution-existence} (\cref{app:proposition-da-intervention-distribution-existence}) shows that soft interventions induced by outcome-invariant DA are \emph{always} stable.


\newpage
\section{IV Regression Supplement}\label{app:iv-supplement}
\label{app:iv-supp-implementation}

\paragraph{Two-stage estimators.}\label{app:iv-alternatives-2-stage}

Minimizing the risk in \cref{eq:iv-loss} is known as two-stage IV regression. Another two-stage IV regression approach that we use in our theoretical results is to minimize the risk \cite{prox-causal-learning,anchor}
\begin{align*}
    \Risk{\iv_{\text{LB}}}{\M}{ h } \coloneqq \EE{}{\M}{ 
        \sqNorm{ 
            \EE{}{\M}{ Y \given Z} - \EE{}{\M}{ h\pr*{ X } \given Z} 
        }
    } .
\end{align*}
This can be shown to lower-bound (hence the subscript LB) the risk in \cref{eq:iv-loss} under squared loss~\cite{prox-causal-learning}.
\begin{align}
\nonumber
&\Rightarrow \Risk{\iv}{\M}{h} = \EE{}{}{
    \sqNorm{
        Y - \EE{}{}{
            \func{h}{X} \given Z
        }
    }
} , 
\\
\tag{Adding and subtracting $\EE{}{}{Y\given Z}$.}
&= \EE{}{}{
    \sqNorm{
        \pr*{
            Y - \EE{}{}{ Y\given Z}
        }
        +
        \pr*{
            \EE{}{}{ Y\given Z} - \EE{}{}{
                \func{h}{X} \given Z
            }
        }
    }
} ,
\\
\tag{Expand squared norm.}
&= \EE{}{}{
    \sqNorm{
        Y - \EE{}{}{ Y\given Z}
    }
}
+
\EE{}{}{
    \sqNorm{
        \EE{}{}{ Y\given Z} - \EE{}{}{
                \func{h}{X} \given Z
        }
    }
}
\\
\nonumber
&\phantom{
    {}= \EE{}{}{
        \sqNorm{
            Y - \EE{}{}{ Y\given Z}
        }
    }
}
+ 2 \EE{}{}{
    \pr*{
        Y - \EE{}{}{ Y\given Z}
    }^\top 
    \pr*{
        \EE{}{}{ Y\given Z}
        -
        \EE{}{}{ \func{h}{X} \given Z}
    }
} ,
\\
\label{eq:iv-2sls-expanded}
&= \EE{}{}{
    \sqNorm{
        Y - \EE{}{}{ Y\given Z}
    }
}
+
\EE{}{}{
    \sqNorm{
        \EE{}{}{ Y\given Z} - \EE{}{}{
                \func{h}{X} \given Z
        }
    }
} , 
\\
\tag{Tower rule, scalar $Y$.}
&= \EE{}{}{
    \sqNorm{
        \EE{}{}{ Y\given Z} - \EE{}{}{
                \func{h}{X} \given Z
        }
    }
}
+
\EE{}{}{
    \EE{}{}{
        \pr*{
            Y - \EE{}{}{ Y\given Z}
        }^2
    \given Z }
} ,
\\
\label{eq:iv-2sls-cmr}
&= \EE{}{}{
    \sqNorm{
        \EE{}{}{ Y \given Z} - \EE{}{}{ \func{h}{X} \given Z }
    }
} + \EE{}{}{ \VV{}{}{ Y \given Z} } = \Risk{\iv_{\text{LB}}}{\M}{h} + \EE{}{}{ \VV{}{}{ Y \given Z} }, 
\end{align}
where \cref{eq:iv-2sls-cmr} follows from the definition of conditional variance and we get \cref{eq:iv-2sls-expanded} by setting the cross term to zero since
\begin{align}
\nonumber
&\Rightarrow
\EE{}{}{
    \pr*{
        Y - \EE{}{}{ Y\given Z}
    }^\top 
    \pr*{
        \EE{}{}{ Y\given Z}
        -
        \EE{}{}{ \func{h}{X} \given Z}
    }
} 
\\
\tag{Tower rule.}
&= \EE{}{}{
    \EE{}{}{
        \pr*{
            Y - \EE{}{}{ Y\given Z}
        }^\top 
        \pr*{
            \EE{}{}{ Y\given Z}
            -
            \EE{}{}{ \func{h}{X} \given Z}
        }
    \given Z }
} ,
\\
\label{eq:towik-first-use}
&= \EE{}{}{
    \EE{}{}{
        \pr*{
            Y - \EE{}{}{ Y\given Z}
        }^\top
    \given Z }
        \pr*{
            \EE{}{}{ Y\given Z}
            -
            \EE{}{}{ \func{h}{X} \given Z}
        }
} ,
\\
\nonumber
&= \EE{}{}{
        \pr*{
            \EE{}{}{ Y\given Z} - \EE{}{}{ Y\given Z}
        }^\top
        \pr*{
            \EE{}{}{ Y\given Z}
            -
            \EE{}{}{ \func{h}{X} \given Z}
        }
} ,
\\
\nonumber
&= \EE{}{}{
        \mathbf{0}^\top
        \pr*{
            \EE{}{}{ Y\given Z}
            -
            \EE{}{}{ \func{h}{X} \given Z}
        }
} = 0 ,
\end{align}
where \cref{eq:towik-first-use} follows from the ``taking out what is known'' rule, i.e.,
\begin{align}
\label{eq:towis-rule}
\EE{}{}{ g(B) A \given B} = g(B) \EE{}{}{A \given B} .
\end{align}

\paragraph{Generalized method of moments.}\label{app:iv-alternatives-gmm}

    The IV regression in our colored-MNIST experiment uses the popular \emph{generalized methods of moments (GMM)}~\cite{gmm,deep-gmm,gan-gmm}, or equivalently the \emph{conditional moment restriction (CMR)}~\cite{prox-causal-learning} framework which tries to directly solve for the fact that in \cref{eq:canonical-sem}~with~scalar~$Y$~
    \begin{equation*}
\EE{}{\M}{ \xi \given Z} = \EE{}{\M}{ Y-f\pr*{X} \given Z} = 0 ,
    \end{equation*}
    which holds as a direct consequence of un-confoundedness of $Z$. For any $q: \mathcal{Z} \to \mathbb{R}$, it then follows
    \begin{equation*}
     \mathbb{E}^{\M}\bigl[\bigl(Y-f(X)\bigr)\cdot q(Z)\bigr] = 0\;.
    \end{equation*}
    The GMM-IV estimate of $f$ therefore tries to enforce this condition \cite{gmm,deep-gmm,gan-gmm} by minimizing the risk
    \begin{align*}
    R^{\M}_{\text{IV}_{\text{GMM}}}(h) \coloneqq \sum_{i=1}^{\mu}\mathbb{E}^{\M}\bigl[\bigl(Y - h(X)\bigr)\cdot q_i(Z)\bigr]^2 = \sqNorm{ \EE{}{\M}{ \pr*{Y-h(X)} \cdot \mathbf{q}(Z) } },
    \end{align*}
    where $\mathbf{q}(\cdot)\in \R^\mu$ represents a vector form of the set of $\mu$ arbitrary real-valued functions $q_i$. A more general form of the above GMM based IV risk is to weight the norm by~some~SPD~$\mW$~\cite{econometric-methods,gmm,deep-gmm}
    \begin{align*} 
    R^{\M}_{\text{IV}_{\text{GMM--}\mW}}(h) \coloneqq \sqNorm{ \EE{}{\M}{ \pr*{Y-h(X)} \cdot \mathbf{q}(Z) } }_{\mW} ,
    \end{align*}
    which gives the most statistically efficient estimator, minimizing the asymptotic variance, for $\mW = \cov{Z}{-1}$ \cite{econometric-methods,gmm,deep-gmm}. We use the same for our colored-MNIST experiments, together with the identity function $\mathbf{q}(Z) = Z$. This gives us the final loss of the form
    \begin{align*}
    R^{\M}_{\text{IV}_{\text{GMM--}\cov{Z}{-1}}}(h) &= \sqNorm{  \EE{}{\M}{ Z \cdot \pr*{Y-h(X)} } }_{\cov{Z}{-1}} .
    \end{align*}
And the empirical version of which can be written as follows
\begin{align}\label{eq:iv-loss-k-class-empirical}
    \Risk{\iv_{\text{GMM--}\cov{Z}{-1}}}{\sD}{ h } \coloneqq  
        \pr*{ 
            \hat{\vy} - \h\pr*{ \hat{\mathbf{X}} } 
        }^\top \hat{\mZ} \hat{\mZ}^{\dagger} \pr*{ 
            \hat{\vy} - \h\pr*{ \hat{\mathbf{X}} } 
        } ,
\end{align}
where for dataset samples $(\vx_i, y_i, \vz_i)\in \sD$, we construct the vector $\hat{\vy}\coloneqq [y_0, \cdots, y_n]^\top$, matrices $\hat{\mathbf{X}}\coloneqq [\vx_0^\top, \cdots, \vx_n^\top]^\top$, $\hat{\mZ}\coloneqq \begin{bmatrix}
    \vz_0 &
    \cdots &
    \vz_n
\end{bmatrix}^\top$ with pseudo-inverse $\hat{\mZ}^{\dagger}$ and define $\h\pr*{ \hat{\mathbf{X}} } \coloneqq [h(\vx_0), \cdots, h(\vx_n)]^\top$.

\newpage
\section{{\myMethod} Regression Supplement}
\label{app:ivl-supplement}

\paragraph{Closed form solution in the linear case.}
\label{app:implementation-details}

The following result gives us a way to compute a closed-form solution to the {\myMethod}$_\alpha$ regression problem in the linear Gaussian case. An empirical version of this is used for our linear experiments.

\begin{restatable}[{\myMethod}$_\alpha$ closed form solution]{proposition}{ivlClosedFormSolution}
\label{proposition:ivl-closed-form-solution}
For SEM $\M$ in \cref{example:ivl}, $\hat{\h}^{\M}_{\ivla}$ is the closed form linear OLS solution between
\begin{align*}
X^{\prime} \coloneqq a X + b \EE{}{}{ X\given Z } , 
&&
Y^{\prime} \coloneqq a Y + b \EE{}{}{ Y\given Z } ,
\end{align*}
where
\begin{align*}
a \coloneqq \sqrt{\alpha} , 
&&
b \coloneqq \sqrt{1 + \alpha} - \sqrt{\alpha} .
\end{align*}
\end{restatable}
\textit{Proof.} See \cref{app:proposition-ivl-closed-form-solution} for the proof. \qed

 For the empirical version of \cref{proposition:ivl-closed-form-solution} we fit a closed-form OLS regressor between
\begin{align*}
X^{\prime} \coloneqq \sqrt{\alpha} X + \pr*{\sqrt{1 + \alpha} - \sqrt{\alpha}} \hat{\mZ}\hat{\mZ}^{\dagger} X ,   
&&  
Y^{\prime} \coloneqq \sqrt{\alpha} Y + \pr*{\sqrt{1 + \alpha} - \sqrt{\alpha}} \hat{\mZ}\hat{\mZ}^{\dagger} Y ,
\end{align*}
where $\hat{\mZ}, \hat{\mZ}^{\dagger}$ are as defined in \cref{eq:iv-loss-k-class-empirical}.

\paragraph{Choice of regularization parameter.}\label{sec:alpha}

We try the following approaches to select the~parameter~$\alpha$.

\textit{Cross validation (CV)}, or any variation thereof. We specifically use the following two in our experiments; (\rnum{1}) vanilla CV with $20\%$ samples held-out for validation (\rnum{2}) {\it level cross validation (LCV)} for when $Z$ is discrete, where hold-out data corresponding to $20\%$ of the levels of $Z$ for validation.

\textit{Confounder correction (CC)},
where in a linear setting we follow an approach similar to \cite{causal-regularization} by estimating the length of the true solution $f$ from the observational data $\mathcal{D}$. We then chose $\alpha$ such that the length of $\hat{h}_{\da\ivla}^{\mathcal{D}}$ is closest to the estimated length of the ground truth solution.

\newpage
\section{Experiment Supplement}\label{app:experiments}

For the methods that use {\it stochastic gradient descent (SGD)}, we use a learning rate of $0.01$, batch size of $256$ for $16$ epochs. For baselines that require a discrete domains/environments, we uniformly discretize each dimension of $G$ into $2$ bins. Higher discretization bins renders most baselines ineffective since each domain/environment rarely has more than 1 sample. To keep the comparison fair, however, we also discretize $G$ for {\myMethod}$_\alpha$ regression when using LCV. For the colored MNIST experiment, all CV implementations including baselines use 5-folds for a random search over an exponentially distributed regularization parameter with rate parameter of 1. Same is the case for simulation and optical device experiments, except that DA+IVL methods use a log-uniform distributed regularization parameter over $[10^{-4}, 1]$. Since RICE \cite{rice} grows the dataset size by augmenting each sample $T$ times, we provide it a $1/T$ sub-sample of the original data for fair comparison. Similarly, the causal regularization method by Kania and Wit \cite{causal-regularization2} expects two datasets, a perturbed and an un-perturbed one, which we substitute with 1/2 augmented data and 1/2 original data respectively.

\subsection{Simulation experiment}\label{app:linear-experiment}


{
\setlength{\clipLeftSweep}{1.75cm}
\setlength{\clipRightSweep}{0.375cm}
\setlength{\clipBottomSweep}{0.35cm}
\setlength{\clipTopSweep}{0.3cm}

\setlength{\fboxsep}{0pt}

\begin{figure*}[t]
\centerfloat
\begin{subfigure}{0.35194\linewidth}
\centering
    \raisebox{0.15ex}{\includegraphics[width=\linewidth,trim={0.3cm \clipBottomSweep{} \clipRightSweep{} \clipTopSweep{}},clip]{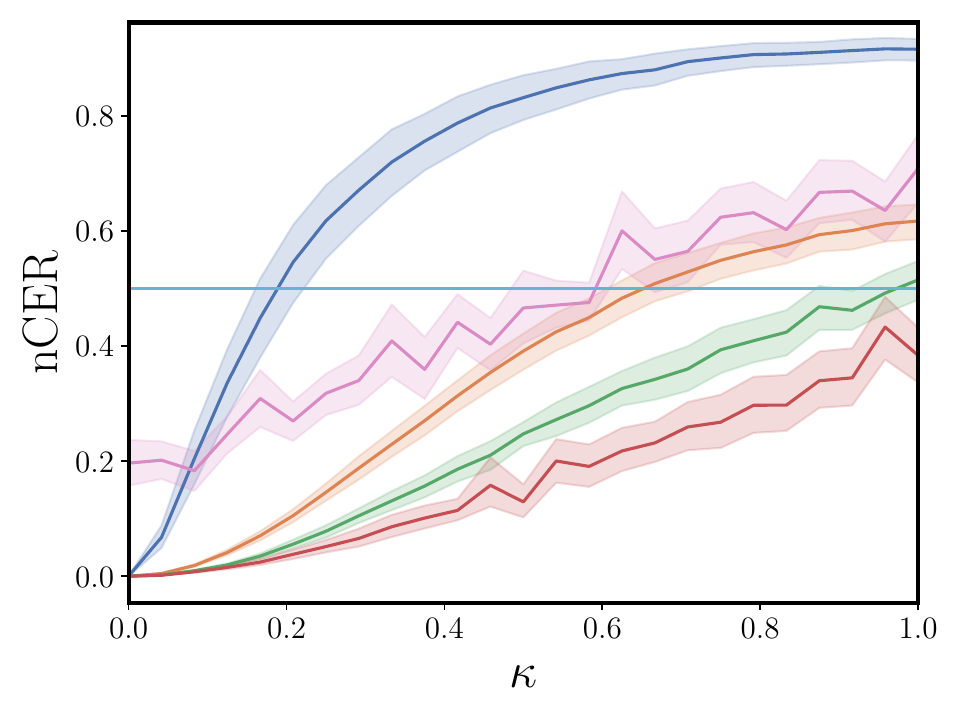}}
    \caption{$\gamma = 1$, $\kappa\in[0, 1]$}
    \label{fig:cyclic-kappa-sweep}
\end{subfigure}
\hfill
\begin{subfigure}{0.31903\linewidth}
\centering
    \includegraphics[width=\linewidth,trim={\clipLeftSweep{} \clipBottomSweep{} \clipRightSweep{} \clipTopSweep{}},clip]{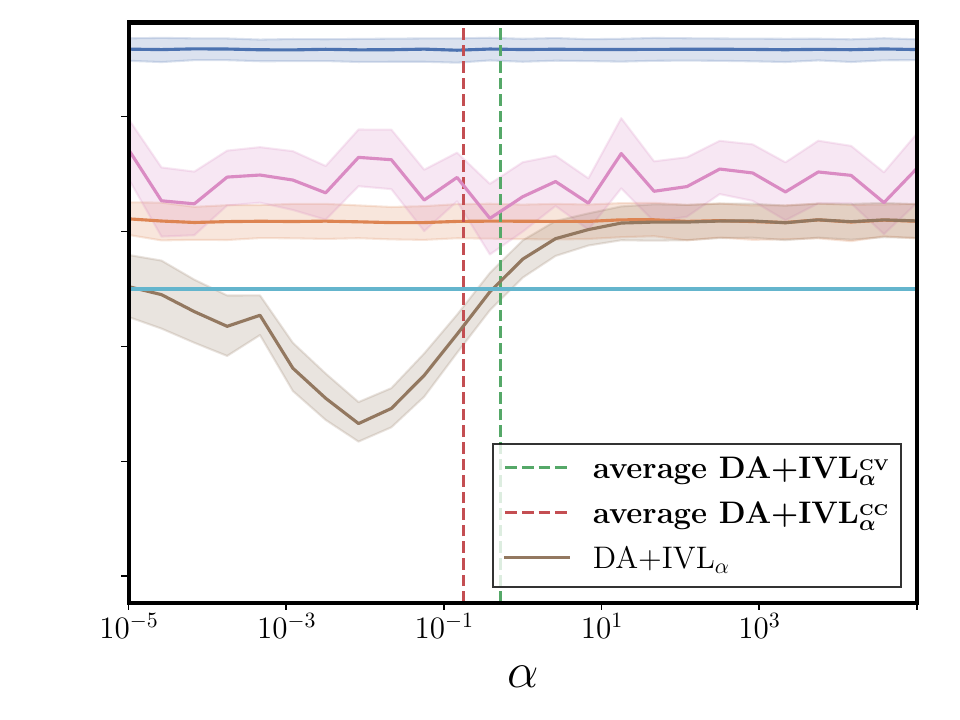}
    \caption{$\gamma, \kappa = 1$, $\alpha\in[10^{-5}, 10^5]$}
    \label{fig:cyclic-alpha-sweep}
\end{subfigure}
\hfill
\begin{subfigure}{0.31903\linewidth}
\centering
    \includegraphics[width=\linewidth,trim={\clipLeftSweep{} \clipBottomSweep{} \clipRightSweep{} \clipTopSweep{}},clip]{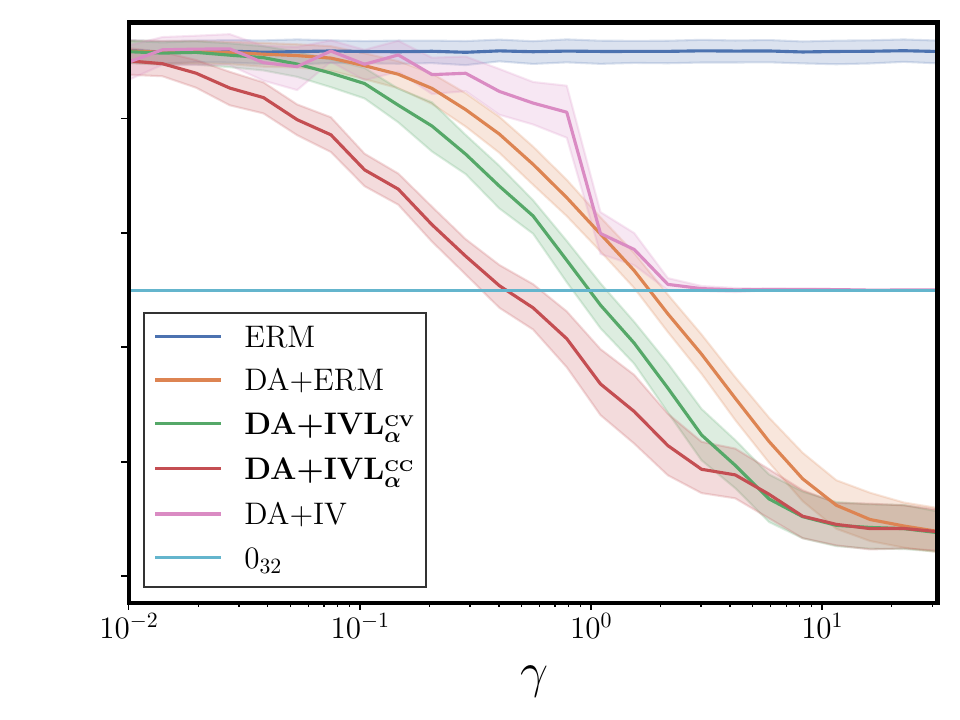}
    \caption{$\gamma \in [10^{-2}, 10^{1.5}]$, $\kappa=1$}
    \label{fig:cyclic-gamma-sweep}
\end{subfigure}
\caption{Simulation of the linear Gaussian SEM of \cref{example:da} with the same setting as \cref{fig:simulation-experiment}, but $\vtau^\top, \f$ sampled uniformly over a unit sphere, representing a cyclic structure. Each data-point represents the average ${\nCER{}}$ over $25$ trials with a $95\%$ CI.}
\label{fig:cyclic-simulation-experiment}
\end{figure*}
}


{

\setlength{\clipLeftSweep}{1.75cm}
\setlength{\clipRightSweep}{0.375cm}
\setlength{\clipBottomSweep}{0.35cm}
\setlength{\clipTopSweep}{0.3cm}

\setlength{\fboxsep}{0pt}

\begin{figure*}[t]
\centerfloat
\begin{subfigure}{0.35194\linewidth}
\centering
    \raisebox{0.15ex}{\includegraphics[width=\linewidth,trim={0.3cm \clipBottomSweep{} \clipRightSweep{} \clipTopSweep{}},clip]{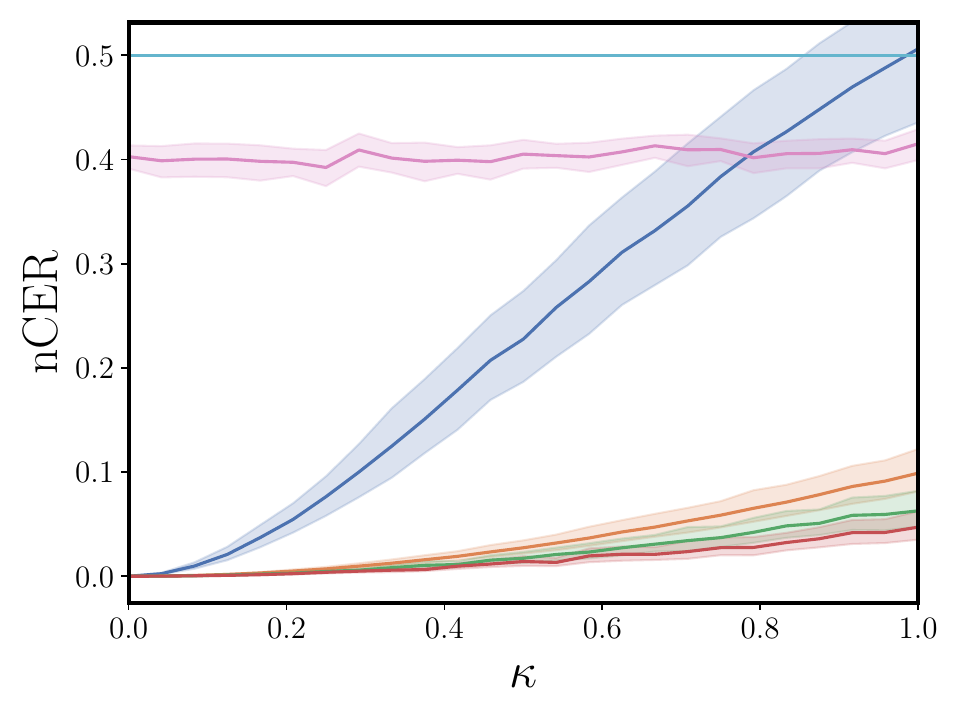}}
    \caption{$\gamma = 1$, $\kappa\in[0, 1]$}
    \label{fig:random-kappa-sweep}
\end{subfigure}
\hfill
\begin{subfigure}{0.31903\linewidth}
\centering
    \includegraphics[width=\linewidth,trim={\clipLeftSweep{} \clipBottomSweep{} \clipRightSweep{} \clipTopSweep{}},clip]{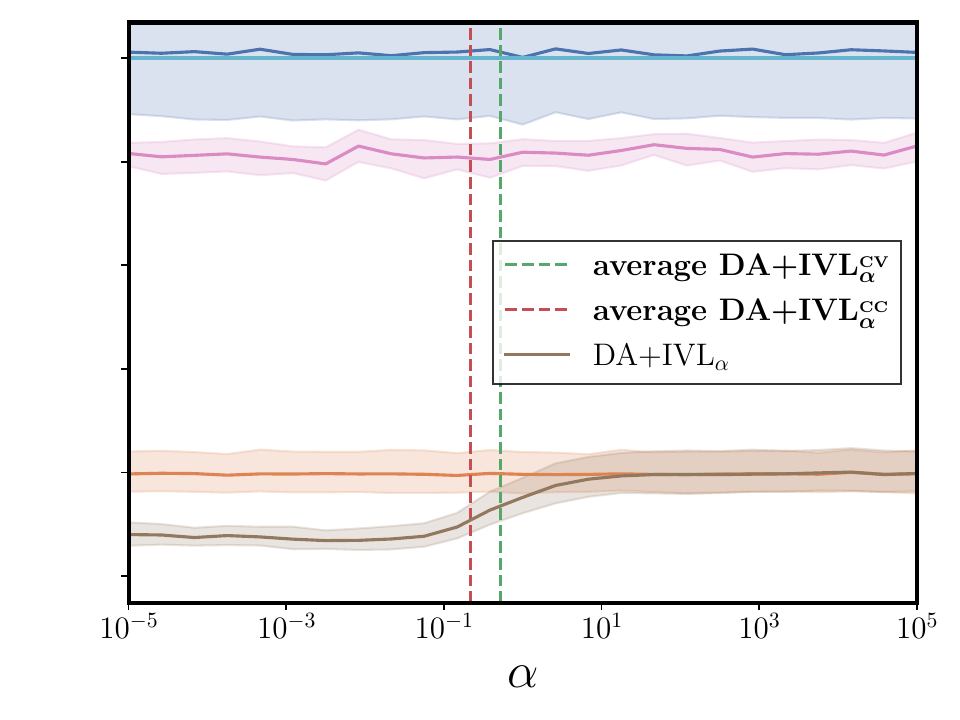}
    \caption{$\gamma, \kappa = 1$, $\alpha\in[10^{-5}, 10^5]$}
    \label{fig:random-alpha-sweep}
\end{subfigure}
\hfill
\begin{subfigure}{0.31903\linewidth}
\centering
    \includegraphics[width=\linewidth,trim={\clipLeftSweep{} \clipBottomSweep{} \clipRightSweep{} \clipTopSweep{}},clip]{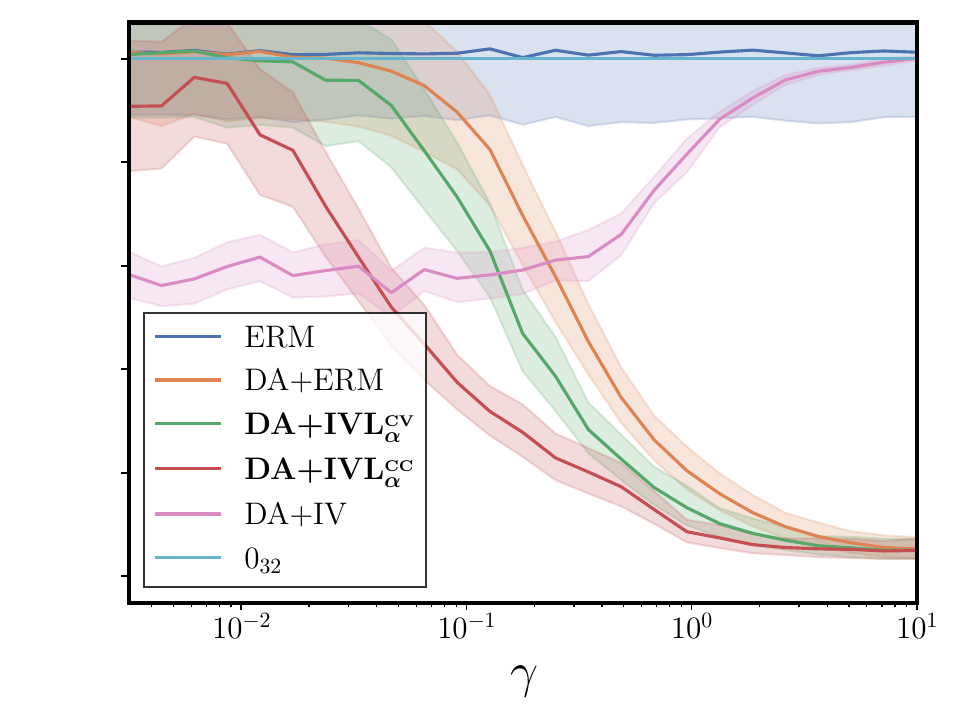}
    \caption{$\gamma \in [10^{-2.5}, 10]$, $\kappa=1$}
    \label{fig:random-gamma-sweep}
\end{subfigure}
\caption{Same experiment as \cref{fig:simulation-experiment}, but with $\K$ constructed by randomly selecting each basis of $\Null{\f^\top}$ with a probability of $2/3$, simulating the effect of knowing only \emph{some} symmetries of $\f$. Each data-point represents the average ${\nCER{}}$ over $25$ trials with a $95\%$ CI.}
\label{fig:random-simulation-experiment}
\end{figure*}

}
For the parameter sweep experiments of \cref{fig:simulation-experiment}, we generate a treatment of dimension $m=32$, but for the OOD baseline comparison experiment in \cref{fig:baselines-benchmarks-experiments} we use $m=16$. Furthermore, for the OOD baseline comparison experiment in \cref{fig:baselines-benchmarks-experiments}, we randomly pick each basis of $\Null{\f}$ with a probability $1/3$ to construct $\K$ (i.e., we know only some, but not all symmetries of $\f$).

We also provide additional linear simulation experiment results in \cref{fig:cyclic-simulation-experiment,fig:random-simulation-experiment}---the former simulates a cyclic structure with a non-zero $\vtau$, and the later simulates a case where only some, but not all symmetries of $\f$ are known. The results of both are consistent with our original experiment in \cref{fig:simulation-experiment}.

\subsection{Optical device experiment}\label{app:optical-device-experiment}

\begin{table}[t]
    \caption{
        $\nCER{}$ $\pm$ one standard error (SE) across the 12 optical-device datasets for various choices of DA. \textbf{Bold} and \emph{italic} denote the lowest and second-lowest average ${\nCER{}}$, respectively. Superscripts $\color{blue}\ast$ and $\color{green}\dagger$ indicate a significant improvement over ERM or \emph{both} ERM \emph{and} DA+ERM, respectively, beyond a margin of SE. Lastly, {\color{red}---} indicates that the method was too expensive for the~value~to~be~computed.
    }\label{table:optical_device}
    \centering
    \begin{tabular}{@{}r l@{}l@{}l@{}l@{}}
        \toprule
        Method & rotate > hflip > vflip\;\;\;\, & random-permutation\;\;\;\, & gaussian-noise\;\;\;\;\;\; & all \\
        \midrule
        ERM & $0.827 \pm 0.079$ & $0.827 \pm 0.079$ & $0.827 \pm 0.079$ & $0.823 \pm 0.083$ \\
        DA+ERM & $\mathit{ 0.617 \pm 0.085^{\color{blue}\ast} }$ & $\mathbf{ 0.513 \pm 0.082^{\color{blue}\ast} }$ & $0.707 \pm 0.090^{\color{blue}\ast}$ & $\mathit{ 0.513 \pm 0.075^{\color{blue}\ast} }$ \\
        \bf DA+IVL$\bm_{\bm\alpha}^{_\text{CV}}$ & $0.623 \pm 0.087^{\color{blue}\ast}$ & $0.540 \pm 0.085^{\color{blue}\ast}$ & $\mathit{ 0.641 \pm 0.092^{\color{blue}\ast} }$ & $0.533 \pm 0.083^{\color{blue}\ast}$ \\
        \bf DA+IVL$\bm_{\bm\alpha}^{_\text{LCV}}$ & $0.619 \pm 0.087^{\color{blue}\ast}$ & $0.534 \pm 0.082^{\color{blue}\ast}$ & $0.662 \pm 0.091^{\color{blue}\ast}$ & $0.574 \pm 0.087^{\color{blue}\ast}$ \\
        \bf DA+IVL$\bm_{\bm\alpha}^{_\text{CC}}$ & $0.623 \pm 0.085^{\color{blue}\ast}$ & $\mathit{ 0.527 \pm 0.082^{\color{blue}\ast} }$ & $\mathbf{ 0.639 \pm 0.076^{\color{blue}\ast} }$ & $\mathbf{ 0.509 \pm 0.078^{\color{blue}\ast} }$ \\
        DA+IV & $0.689 \pm 0.065^{\color{blue}\ast}$ & $0.973 \pm 0.011$ & $0.955 \pm 0.011$ & $0.640 \pm 0.083^{\color{blue}\ast}$ \\
        IRM & $0.972 \pm 0.010$ & $0.960 \pm 0.015$ & $0.970 \pm 0.009$ & $0.953 \pm 0.018$ \\
        ICP & $\mathbf{ 0.544 \pm 0.019^{\color{green}\dagger} }$ & $\mathit{ 0.527 \pm 0.012^{\color{blue}\ast} }$ & $0.646 \pm 0.054^{\color{green}\dagger}$ & {\color{red}---} \\
        DRO & $0.975 \pm 0.005$ & $0.959 \pm 0.012$ & $0.981 \pm 0.003$ & $0.952 \pm 0.014$ \\
        RICE & $0.966 \pm 0.014$ & $0.960 \pm 0.012$ & $0.974 \pm 0.005$ & $0.959 \pm 0.016$ \\
        V-REx & $0.962 \pm 0.024$ & $0.957 \pm 0.013$ & $0.979 \pm 0.005$ & $0.925 \pm 0.037$ \\
        MM-REx & $0.978 \pm 0.013$ & $1.000 \pm 0.000$ & $1.000 \pm 0.000$ & $1.000 \pm 0.000$ \\
        $\ell_1$ Janzing `19 & $0.821 \pm 0.081$ & $0.821 \pm 0.081$ & $0.821 \pm 0.081$ & $0.817 \pm 0.077$ \\
        $\ell_2$ Janzing `19 & $0.823 \pm 0.076$ & $0.823 \pm 0.076$ & $0.823 \pm 0.076$ & $0.828 \pm 0.079$ \\
        Kania, Wit `23 & $0.652 \pm 0.084^{\color{blue}\ast}$ & $0.559 \pm 0.084^{\color{blue}\ast}$ & $0.727 \pm 0.088^{\color{blue}\ast}$ & $0.543 \pm 0.080^{\color{blue}\ast}$ \\
        \bottomrule
    \end{tabular}
\end{table}

In the simulation and optical device experiments, we fit a linear function $h(.)\coloneqq \h\in \R^m$ for a squared loss in all of our risk metrics. For IVL$_\alpha$ regression, we use the closed-form OLS solution from \cref{app:implementation-details}. We also use a closed-form solution for ERM, DA+ERM and DA+IV (2SLS) baselines. The rest of the baselines (other than ICP) use SGD.


In \cref{table:optical_device}, we report further experiments on the optical device dataset with various DA choices. The findings continue to confirm our main hypothesis: DA+IVL dominates DA+ERM, which itself dominates ERM. We never observe an opposite trend with statistical significance.

\subsection{Colored-MNIST experiment}

In the colored MNIST experiment, we use the same 3-layer neural network (NN) architecture for $h$ across all methods comprising of a fully-connected input layer of input dimension $m$, hidden layer of input/output dimension $256$ and output classification layer with a Sigmoid function. Each layer is separated by an intermediary \emph{rectified linear unit} activation function. For the IV risk, we use the empirical version of the GMM based risk from \cref{eq:iv-loss-k-class-empirical}.

\subsubsection*{Colored-MNIST as a cyclic SEM---From invariant prediction to estimating causal effects}\label{app:cmnist-experiment}

\begin{figure}[ht]
\centerfloat
\includegraphics{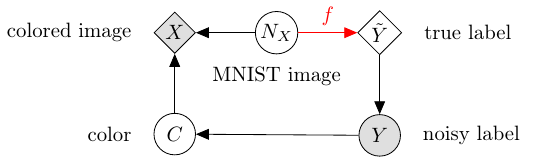}
\caption{The data generation DAG for colored-MNIST as discussed by the original authors \cite{arjovsky2020invariant}. They aim to learn a predictor $h:\sX\rightarrow\sY$ such that it is invariant to changes in $\distribution{X\mid Y}{}$. We argue that this DAG view of colored-MNIST does not make it obvious how the true labeling function $f(\vx)$ is related to the ATE $\EE{}{ \doM{X}{ \vx } }{Y\given X= \vx }$, which we believe is because it is virtually equivalent to the reduced form of our structural form presented in \cref{fig:daiv-color-mnist}.}
\label{fig:c-mnist-original}
\end{figure}
\begin{figure}[ht]
\centerfloat
\begin{subfigure}{.45\linewidth}
\centering
\includegraphics{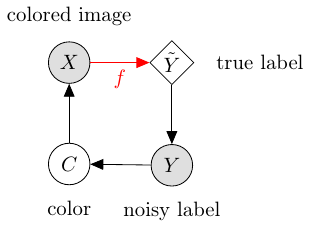}
\caption{Graph for generating colored-MNIST data.}
\end{subfigure}
\hfill
\begin{subfigure}{.55\linewidth}
\centering
\includegraphics{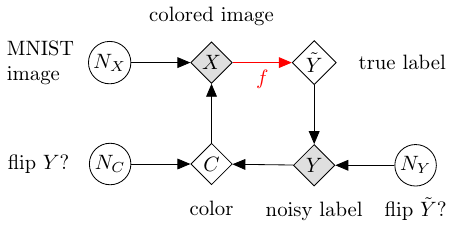}
\caption{Augmented graph---exogenous variables explicitly shown.}
\end{subfigure}
\caption{A cyclic SEM perspective of the colored-MNIST data---an MNIST image $N_X$ is assigned color $C$ to produce a colored-MNIST image $X$. This is then passed through the ground-truth labeling function $f$ to produce the true label $\tilde{Y}$. We flip this with probability 0.25 to produce the observed label $Y$, which in turn is flipped with probability $e$ (at train time $e\in\{0.1, 0.2\}$ and $e=0.9$ at test time) to produce the color $C$. These assignments are iteratively applied for any joint sample of the exogenous variables $N_X, N_Y, N_C$ starting at arbitrary values of endogenous variables until convergence to the unique stationary point $X, Y, C$ (and $\tilde{Y}$).}
\label{fig:daiv-color-mnist}
\end{figure}

In this section we give a cyclic SEM perspective of the colored-MNIST experiment from \cite{arjovsky2020invariant}. The task is binary classification of colored images $X$ from the MNIST dataset into low digits ($y=0$ for digits from $0$ to $4$) and high digits ($y=1$ for digits from $5$ to $9$). The difficulty of the task arises from there being a higher spurious correlation between the color $C$ of the images ($c=0$ for blue and $c=1$ for green) and (noisy) labels $Y$ as compared to the correlation between the digits in the image and the label.

Consider the following cyclic SEM in \cref{fig:daiv-color-mnist}.
\begin{equation}
\begin{aligned}
\nonumber
\vn_X &\sim \distribution{N_X}, n_Y \sim \bernoulliDistribution{0.25}, n_c \sim \bernoulliDistribution{e} && \text{sample all exogenous variables}
\\
\nonumber
X &= \verb!color!(C, \vn_X) && \text{apply color $C$ to the image}
\\
\nonumber
\tilde{Y} &= f(X) && \text{generate ground-truth label with true labeling function}
\\
\nonumber
Y &= \xor{ \tilde{Y}, n_Y } && \text{flip the label with probability $0.25$}
\\
\nonumber
C &= \xor{ Y, n_C } && \text{generate color by flipping $Y$ with probability $e$},
\end{aligned}
\end{equation}
where we first randomly sample an un-colored MNIST image $\vn_X$, and some Bernoulli distributed label noise $n_Y\sim \bernoulliDistribution{0.25}$ and color noise $n_C\sim \bernoulliDistribution{e}$ which is different for each environment $e\in\{0.1, 0.2\}$. Then for some initial arbitrary values $\vx_0$, $\tilde{y}_0$, $y_0$ and $c_0$ respectively for the observed colored image $X$, the ground-truth label $\tilde{Y}$, the observed noisy label $Y$ and the image color $C$, we iteratively apply the following assignments from the SEM
\begin{equation}
\begin{aligned}
\nonumber
\vx_t &= \verb!color!(c_{t-1}, \vn_X) && \text{apply color $C$ to the image}
\\
\nonumber
\tilde{y}_t &= f(\vx_{t-1}) && \text{generate ground-truth label with true labeling function}
\\
\nonumber
y_t &= \xor{ \tilde{y}_{t-1}, n_Y } && \text{flip the label with probability $0.25$}
\\
\nonumber
c_t &= \xor{ y_{t-1}, n_C } && \text{generate color by flipping $Y$ with probability $e$},
\end{aligned}
\end{equation}
until they converge while keeping all sampled exogenous variables $\vn_X, n_Y, n_C$ fixed. It is straightforward to show that this SEM will converge after a maximum of $t=5$ iterations\footnote{Following the mechanisms $c_0 \rightarrow \vx_1\rightarrow \tilde{y}_2 \rightarrow y_3 \rightarrow c_4 \rightarrow \vx_5$, we see that $(\vx_4, y_4, c_4) = (\vx_5, y_5, c_5)$ (same for $\tilde{y}_4 = \tilde{y}_5$).} due to the invariance of $f$ to the color of the image $C$. Furthermore, this stationary-point will be uniquely determined by our exogenous samples $\vn_X, n_Y, n_C$. And this is how we generate one sample $(\vx, y)$ for our colored-MNIST experiment. We repeat this process to generate a sample $(\vx, y)$ for each of $n$ samples $\vn_X, n_Y, n_C$.

Note that the ground-truth labeling function $f$ can only correctly predict the labels $75\%$ of the time. At test time we flip the correlation between the label $Y$ and the image color $C$ by setting $e=0.9$. Also, the above cyclic SEM for colored-MNIST produces the same distribution for $(X, Y)$ as \cite{arjovsky2020invariant}.

The above cyclic SEM perspective of colored-MNIST is interesting because it makes it clear that colored-MNIST is essentially a causal effect estimation task. Specifically, we can estimate the true labeling function $f$ by estimating the ATE $\EE{}{ \doM{X}{\vx} }{ Y\given X= \vx }$ since
\begin{align*}
    \EE{}{ \doM{X}{\vx} }{ Y\given X=\vx } 
    &= \EE{}{ \doM{X}{\vx} }{ \xor{ f(X), N_Y } \given X=\vx },
    \\
    \tag{$N_Y \indep X^{ \doM{X}{\vx} }$.}
    &= \EE{}{\M}{ \xor{ f(\vx), N_Y } },
    \\
    \tag{Definition of \texttt{xor}.}
    &= \EE{}{ \M }{ f(\vx) + N_Y - 2 f(\vx) N_Y },
    \\
    &= f(\vx) + \EE{}{\M}{ N_Y } - 2 f(\vx) \EE{}{\M}{ N_Y },
    \\
    &= \pr*{ 1 - 2\EE{}{\M}{ N_Y } } f(\vx) + \EE{}{\M}{ N_Y },
    \\
    &= 0.5 f(\vx) + 0.25\;. \tag{$N_Y\sim B(0.25)$.}
\end{align*}
Because this is a binary classification task, we have
\begin{align*}
    \texttt{round}\pr*{
    \EE{}{ \doM{X}{\vx} }{ Y\given X=\vx }
    } = f(\vx) .
\end{align*}

This is in contrast to the original DAG perspective of colored-MNIST shown in \cref{fig:c-mnist-original}, where the connection to the estimation of the causal mechanism $f$ is not immediately obvious. We argue that this is because the DAG in \cref{fig:c-mnist-original} is virtually equivalent to the reduced form of our structural form presented in \cref{fig:daiv-color-mnist}.

\newpage
\section{Proofs}\label{app:proofs}



\subsection{Proof of \cref{proposition:ivl-closed-form-solution}---{\myMethod} regression closed form solution in the linear case}
\label{app:proposition-ivl-closed-form-solution}
{
\allowdisplaybreaks
\vspace{\ProofSep}
\ivlClosedFormSolution*
\begin{proof}
The OLS solution for $(X^{\prime}, Y^{\prime})$ minimizes the following ERM risk
\begin{align}
\nonumber
&\Rightarrow \EE{}{}{ 
    \sqNorm{ 
    Y^{\prime} - \h^\top X^{\prime}
    } 
} 
\\
\tag{Substitute in definitions of $X^{\prime}, Y^{\prime}$.}
&= \EE{}{}{
    \sqNorm{
        a Y + b \EE{}{}{ Y\given Z} 
        - \h^\top \pr*{
            a X + b \EE{}{}{ X\given Z} 
        }
    }
} , 
\\
\tag{Distribute the subtraction.}
&= \EE{}{}{
    \sqNorm{
        a \pr*{ 
            Y - \h^\top X 
        }
        +
        b \pr*{ 
            \EE{}{}{ Y\given Z}
            -
            \h^\top \EE{}{}{ X\given Z}
        }
    }
} , 
\\      
\tag{Expand squared norm.}
\nonumber
&= a^2 \EE{}{}{ 
    \sqNorm{ 
        Y - \h^\top X
    }
}
+
b^2 \EE{}{}{ 
    \sqNorm{ 
        \EE{}{}{ Y\given Z}
        -
        \h^\top \EE{}{}{ X\given Z}
    }
}
\\
\label{eq:augmented-ols-expanded}
&\phantom{
    {}= a^2 \EE{}{}{ 
        \sqNorm{ 
            Y - \h^\top X
        }
    }
}
+ 2 a b \EE{}{}{
    \pr*{
        Y - \h^\top X
    }^\top 
    \pr*{
        \EE{}{}{ Y\given Z}
        -
        \h^\top \EE{}{}{ X\given Z}
    }
} .
\end{align}
First we note that from the definitions of $a, b$ we have
\begin{align}
\label{eq:ab2alpha}
a^2 = \sqrt{\alpha} ,
&&
b^2 + 2 a b = \pr*{\sqrt{1+\alpha}-\sqrt{\alpha}}^2 + 2 \sqrt{\alpha} \pr*{\sqrt{1+\alpha}-\sqrt{\alpha}} = 1 .
\end{align}
Now we evaluate the cross term in \cref{eq:augmented-ols-expanded}
\begin{align}
\nonumber
&\Rightarrow \EE{}{}{
    \pr*{
        Y - \h^\top X
    }^\top 
    \pr*{
        \EE{}{}{ Y\given Z}
        -
        \h^\top \EE{}{}{ X\given Z}
    }
}
\\
\tag{Law of iterated expectation.}
&= \EE{}{}{
    \EE{}{}{
        \pr*{
            Y - \h^\top X
        }^\top 
        \pr*{
            \EE{}{}{ Y\given Z}
            -
            \h^\top \EE{}{}{ X\given Z}
        }
    \given Z }
} , 
\\
\tag{Taking out what is known; \cref{eq:towis-rule}.}
&= \EE{}{}{
    \EE{}{}{
        \pr*{
            Y - \h^\top X
        }^\top 
    \given Z
    }
    \pr*{
        \EE{}{}{ Y\given Z}
        -
        \h^\top \EE{}{}{ X\given Z}
    }
}
\\
\nonumber
&= \EE{}{}{
    \pr*{
        \EE{}{}{ Y\given Z}
        -
        \h^\top \EE{}{}{ X\given Z}
    }^\top
    \pr*{
        \EE{}{}{ Y\given Z}
        -
        \h^\top \EE{}{}{ X\given Z}
    }
}
\\
\nonumber
&= \EE{}{}{
    \sqNorm{
        \EE{}{}{ Y\given Z}
        -
        \h^\top \EE{}{}{ X\given Z}
    }
} .
\end{align}
Substituting this back in \cref{eq:augmented-ols-expanded} we get
\begin{align}
\nonumber
&\Rightarrow
\EE{}{}{ 
    \sqNorm{ 
    Y^{\prime} - \h^\top X^{\prime}
    } 
} 
\\
\nonumber
&= a^2 \EE{}{}{ 
    \sqNorm{ 
        Y - \h^\top X
    }
}
+
\pr*{ b^2 + 2 a b } \EE{}{}{ 
    \sqNorm{ 
        \EE{}{}{ Y\given Z}
        -
        \h^\top \EE{}{}{ X\given Z}
    }
} , 
\\
\tag{From \cref{eq:ab2alpha}.}
&= \alpha \EE{}{}{ 
    \sqNorm{ 
        Y - \h^\top X
    }
}
+
\EE{}{}{ 
    \sqNorm{ 
        \EE{}{}{ Y\given Z}
        -
        \h^\top \EE{}{}{ X\given Z}
    }
} , 
\\
\tag{From \cref{eq:iv-2sls-cmr}.}
&= \alpha \Risk{\erm}{\M}{\h}
+
\Risk{\iv}{\M}{\h}
-
\EE{}{}{ \VV{}{}{ Y \given Z} } , 
\\
\nonumber
&= \Risk{\ivla}{\M}{\h}
-
\EE{}{}{ \VV{}{}{ Y \given Z} } .
\end{align}
\end{proof}
}

\myRule

\subsection{Proof of \cref{proposition:da-intervention-distribution-existence}---Existence of an interventional distribution given a DA}
\label{app:proposition-da-intervention-distribution-existence}
{
\allowdisplaybreaks
\begin{proposition}[unique stationary interventional distribution]
\label{proposition:da-intervention-distribution-existence}
In SEM $\A$ from \cref{eq:problem-setup-sem}, given any $(\vg, \vc, \vn_X, \vn_Y)\sim P^{\A}_{G, C, N_X, N_Y}$, if for all $(\vx_0, \vy_0) \in \mathcal{X}\times \mathcal{Y}$ the unique limits
\begin{align*}
    \vx^{\A} &\coloneqq \lim_{t\rightarrow\infty} \vx^{\A}_t = \lim_{t\rightarrow\infty} \tau\pr*{ \vy^{\A} _{t-1}, \vc, \vn_X },
    \\
    \vy^{\A} &\coloneqq \lim_{t\rightarrow\infty} \vy^{\A}_t = \lim_{t\rightarrow\infty} f\pr*{ \vx^{\A} _{t-1} } + \epsilon(\vc) + \vn_Y
\end{align*}
exist, then in $\doA{\tau}{\vg\tau}$ the unique limits
\begin{align*}
    \vx^{ \doA{\tau}{\vg\tau} } &\coloneqq \lim_{t\rightarrow\infty} \vx^{ \doA{\tau}{\vg\tau} }_t = \lim_{t\rightarrow\infty} \vg\tau\pr*{ \vy^{ \doA{\tau}{\vg\tau} }_{t-1}, \vc, \vn_X } = \vg\vx^{\A},
    \\
    \vy^{ \doA{\tau}{\vg\tau} } &\coloneqq \lim_{t\rightarrow\infty} \vy^{ \doA{\tau}{\vg\tau} }_t = \lim_{t\rightarrow\infty} f\pr*{ \vx^{ \doA{\tau}{\vg\tau} }_{t-1} } + \epsilon(\vc) + \vn_Y = \vy^{\A}
\end{align*}
also exist.
\end{proposition}
\begin{proof}
First we try to show that
\begin{align}
    \label{eq:induction-y}
    \vy^{ \doA{\tau}{\vg\tau} }_t &= \vy^{\A}_t .
\end{align}
For the base case, we have by construction
\begin{align}
    \nonumber
    \vy^{ \doA{\tau}{\vg\tau} }_0 \coloneqq \vy_0 \eqqcolon \vy^{\A}_0 .
\end{align}
For the step case, assuming that $\vy^{ \doA{\tau}{\vg\tau} }_t = \vy^{\A}_t$, we have\footnote{Note that here the step size for proof by induction would be $\Delta t=2$ since $\vy_t$ precedes $\vy_{t+2}$. Similar is the case for $\vx_t$ as well.},
\begin{align}
    \nonumber
    \vy^{ \doA{\tau}{\vg\tau} }_{t+2} &= f\pr*{ \vx^{ \doA{\tau}{\vg\tau} }_{t+1} } + \epsilon(\vc) + \vn_Y ,
    \\
    \nonumber
    &= f\pr*{ \vg\tau\pr*{ \vy^{ \doA{\tau}{\vg\tau} }_{t}, \vc, \vn_X } } + \epsilon(\vc) + \vn_Y ,
    \\
    \tag{Invariance of $f$ to $\vg$.}
    &= f\pr{ \tau\pr*{ \vy^{ \doA{\tau}{\vg\tau} }_{t}, \vc, \vn_X } } + \epsilon(\vc) + \vn_Y ,
    \\
    \tag{Assumption $\vy^{\doA{\tau}{\vg\tau}}_t = \vy^{\A}_t$.}
    &= f\pr*{ \tau\pr*{ \vy^{\A}_{t}, \vc, \vn_X } } + \epsilon(\vc) + \vn_Y ,
    \\
    \nonumber
    &= f\pr*{ \vx^{\A}_{t+1} } + \epsilon(\vc) + \vn_Y ,
    \\
    \nonumber
    &= \vy^{\A}_{t+2} .
\end{align}
Hence, we have shown that \cref{eq:induction-y} holds for all even $t$. For odd $t$, we simply replace $t=0$ with $t=1$ in the base case
\begin{align}
    \nonumber
    \vy^{ \doA{\tau}{\vg\tau} }_1 &= f\pr*{ \vx^{ \doA{\tau}{\vg\tau} }_0 } + \epsilon(\vc) + \vn_Y ,
    \\
    \tag{Definitions $\vx^{ \doA{\tau}{\vg\tau} }_0 \coloneqq \vx_0 \eqqcolon \vx^{\A}_0$.}
     &= f\pr*{ \vx^\A_0 } + \epsilon(\vc) + \vn_Y ,
     \\
     \nonumber
     &= \vy^{\A}_1 ,
\end{align}
We have now finally shown that \cref{eq:induction-y} holds for all $t\geq 0$. 

Next, it is now relatively straightforward to show that for any $t> 0$, we have
\begin{align}
    \nonumber
    \vx^{ \doA{\tau}{\vg\tau} }_t &= \vg\tau\pr*{ \vy^{ \doA{\tau}{\vg\tau} }_{t-1}, \vc, \vn_X } ,
    \\
    \tag{Follows from \cref{eq:induction-y}.}
    &= \vg\tau\pr*{ \vy^{\A}_{t-1}, \vc, \vn_X } ,
    \\
    \label{eq:induction-x}
    &= \vg\vx^{\A}_t .
\end{align}

Finally, by applying limit as $t\rightarrow\infty$ to both sides of \cref{eq:induction-y} and \cref{eq:induction-x}, we get
\begin{gather}
    \nonumber
    \vy^{ \doA{\tau}{\vg\tau} } 
    = \lim_{t\rightarrow\infty} \vy^{ \doA{\tau}{\vg\tau} }_t =
    \lim_{t\rightarrow\infty} \vy^{\A}_t
    = \vy^{\A} ,
    \\
    \label{eq:induction-final}
    \vx^{ \doA{\tau}{\vg\tau} } 
    = \lim_{t\rightarrow\infty} \vx^{ \doA{\tau}{\vg\tau} }_t =
    \lim_{t\rightarrow\infty} \vg\vx^{\A}_t
    = \vg\lim_{t\rightarrow\infty} \vx^{\A}_t
    = \vg\vx^{\A} ,
\end{gather}
where the limit can be moved past $\vg$ in \cref{eq:induction-final} because $\vg$ is assumed continuous in its domain.

\end{proof}
}

\myRule

\subsection{Proof of \cref{theorem:ivl-robust-prediction}---Robust prediction with {\myMethod} regression}
\label{app:theorem-ivl-robust-prediction}
{
\allowdisplaybreaks
\vspace{\ProofSep}
\ivlRobustPrediction*
\begin{proof}
Write $X$ in terms of the exogenous variables $C, Z, N_X, N_Y$ using the reduced form from \cref{lemma:sem-solvability} as
\begin{align}
\label{eq:x-reduced-form}
    X = \tZ + \tC + \tN ,
\end{align}
where for readability we represent
\begin{align*}
\tZ \coloneqq \mM_{m\times m} \K^\top Z
,
&&
\tC \coloneqq \mM 
\begin{bmatrix}
    \E^\top \\
    \veps^\top
\end{bmatrix} C
,
&&
\tN \coloneqq 
\sigma\cdot\mM \begin{bmatrix}
    N_X\\
    N_Y
\end{bmatrix}
,
\end{align*}
with
\begin{align*}
    \mM 
\coloneqq 
\begin{bmatrix}
    \mM_{m\times m} & \mM_{m\times 1} \\
    \mM_{1\times m} & 
    \mM_{1\times 1}
\end{bmatrix}
= 
\begin{bmatrix}
    \In{m} & -\vtau^\top\\
    -\f^\top & 1
\end{bmatrix}^{-1}
.
\end{align*}
Now, we start by writing the ERM objective under the intervention $\intervention{\K^\top\pr*{\cdot}}{\vzeta}$ as
\begin{align}
\nonumber 
&\Rightarrow \Risk{\erm}{ \doM{ \K^\top\pr*{\cdot} }{ \vzeta } }{ \h }
\\
\nonumber
&= \EE{}{ \doM{ \K^\top\pr*{\cdot} }{ \vzeta } }{
    \sqNorm{ Y - \h^\top X } 
} ,
\\
\tag{$Y$ structural form \& \cref{eq:x-reduced-form}.}
&= \EE{}{ \doM{ \K^\top\pr*{\cdot} }{ \vzeta } }{
    \sqNorm{
        \xi + \pr*{ \f-\h }^\top \pr*{ \tZ + \tC + \tN }
    }
} ,
\\
\tag{$\tZ$ \& intervention definition.}
&= \EE{}{ \doM{ \K^\top\pr*{\cdot} }{ \vzeta } }{
    \sqNorm{
        \xi + \pr*{ \f-\h }^\top \pr*{ \mM_{m\times m}\vzeta + \tC + \tN }
    }
} ,
\\
\nonumber
&= \EE{}{ \doM{ \K^\top\pr*{\cdot} }{ \vzeta } }{
    \sqNorm{
        \xi + \pr*{ \f-\h }^\top \pr*{ \tC + \tN } + \pr*{ \f-\h }^\top \mM_{m\times m} \vzeta
    }
} ,
\\
\tag{Define ${\h'}^\top\coloneqq \pr*{\f-\h}^\top \mM_{m\times m}$.}
&= \EE{}{ \doM{ \K^\top\pr*{\cdot} }{ \vzeta } }{
    \sqNorm{
        \xi + \pr*{ \f-\h }^\top \pr*{ \tC + \tN } + {\h'}^\top \vzeta
    }
} ,
\\
\tag{Follows from exogeneity of $\vzeta$ under intervention, $\Rightarrow$ cross term zeros-out.}
&= \EE{}{ \doM{ \K^\top\pr*{\cdot} }{ \vzeta } }{
    \sqNorm{
        \xi +  \pr*{ \f-\h }^\top \pr*{ \tC + \tN }
    }
} + \EE{}{ \doM{ \K^\top\pr*{\cdot} }{ \vzeta } }{
    \sqNorm{
        { \h' }^\top \vzeta
    }
} ,
\\
\label{eq:gen-thm:inverse}
&= \EE{}{ \doM{ \K^\top\pr*{\cdot} }{ \0{m} } }{
    \sqNorm{ Y - \h^\top X } 
} + \EE{}{ \doM{ \K^\top\pr*{\cdot} }{ \vzeta } }{
    \sqNorm{ { \h' }^\top \vzeta }
} ,
\\
\nonumber
&= \EE{}{ \doM{ \K^\top\pr*{\cdot} }{ \0{m} } }{
    \sqNorm{ Y - \h^\top X } 
} + \sqNorm{ 
    { \h' }^\top \vzeta 
},
\\
\nonumber
&= \EE{}{ \doM{ \K^\top\pr*{\cdot} }{ \0{m} } }{
    \sqNorm{ Y - \h^\top X }
} + \tr{ \vzeta^\top { \h' } { \h' }^\top \vzeta } , 
\\
\label{eq:ivl-before-max}
&= \EE{}{ \doM{ \K^\top\pr*{\cdot} }{ \0{m} } }{
    \sqNorm{ Y - \h^\top X }
} + \tr{
    \h'^\top \vzeta \vzeta^\top \h' 
}.
\end{align}
Now, note that the maximum of the trace term over $\vzeta\in \sP_{\alpha}$ gives
\begin{align*}
&\Rightarrow \max_{\vzeta\in\mathcal{P}_{\alpha}} \; \tr{ \h'^\top\vzeta\vzeta^\top\h' }
,
\\
\tag{Linearity of trace and definition of $\mathcal{P}_\alpha$.}
&= \pr*{ \frac{1}{\alpha} + 1 } \; \tr{ \h'^\top \pr[\Big]{ \K^\top \EE{}{\M}{Z Z^\top} \K } \h' }
,
\\
\tag{Linearity of expectation.}
&= \pr*{ \frac{1}{\alpha} + 1 } \; \EE{}{\M}{ \tr{ 
\h'^\top \K^\top Z Z^\top \K \h' 
} } , 
\\
\tag{Cyclic property of trace.}
&= \pr*{ \frac{1}{\alpha} + 1 } \; \EE{}{\M}{ \tr{ 
Z^\top \K \h' \h'^\top \K^\top Z 
} } , 
\\
&= \pr*{ \frac{1}{\alpha} + 1 } \; \EE{}{\M}{ \sqNorm{ \h'^\top \K^\top Z 
} } , 
\\
\tag{Substitute in definition of ${\h'}^\top$.}
&= \pr*{ \frac{1}{\alpha} + 1 } \; \EE{}{\M}{ \sqNorm{ \pr*{\f - \h}^\top \mM_{m\times m} \K^\top Z 
} } ,
\\
\tag{Definition of $\tZ$.}
&= \pr*{ \frac{1}{\alpha} + 1 } \; \EE{}{\M}{ \sqNorm{ \pr*{\f - \h}^\top \tZ
} } . 
\end{align*}
We can now substitute this in while maximizing both sides of \cref{eq:ivl-before-max} over interventions $\vzeta\in\sP_{\alpha}$ as
\begin{align}
\nonumber
&\Rightarrow \max_{\vzeta\in\mathcal{P}_{\alpha}}  \Risk{\erm}{ \doM{ \K^\top\pr*{\cdot} }{ \0{m} } }{ \h } 
\\
\tag{First term does not have $\vzeta$.}
&= \EE{}{ \doM{ \K^\top\pr*{\cdot} }{ \0{m} } }{ 
    \sqNorm{ Y - \h^\top X } 
} + \max_{\vzeta\in\mathcal{P}_{\alpha}} \; \tr{ \h'^\top\vzeta\vzeta^\top\h' } ,
\\
\nonumber
&= \EE{}{ \doM{ \K^\top\pr*{\cdot} }{ \0{m} } }{ 
    \sqNorm{ Y - \h^\top X } 
} + \pr*{ \frac{1}{\alpha} + 1 } \; \EE{}{\M}{ \sqNorm{ \pr*{\f - \h}^\top \tZ
} } 
,
\\
\tag{Inverse step of \cref{eq:gen-thm:inverse}.}
&= \EE{}{ \M }{ 
    \sqNorm{ Y - \h^\top X } 
} + \frac{1}{\alpha} \; \EE{}{\M}{ \sqNorm{ \pr*{\f - \h}^\top \tZ 
} } 
,
\\
\tag{From conditional exp. of \cref{eq:x-reduced-form}.}
&= \EE{}{ \M }{ 
    \sqNorm{ Y - \h^\top X } 
} + \frac{1}{\alpha} \; \EE{}{\M}{ \sqNorm{ \pr*{\f - \h}^\top \EE{}{}{ X\given Z } 
} } 
,
\\
\tag{Linearity of expectation.}
&= \EE{}{ \M }{ 
    \sqNorm{ Y - \h^\top X } 
} + \frac{1}{\alpha} \; \EE{}{\M}{ \sqNorm{ \EE{}{}{ \f^\top X\given Z} - \h^\top \EE{}{}{ X\given Z } 
} }
,
\\
\tag{Inverse step of \cref{eq:gen-thm:inverse}.}
&= \EE{}{ \M }{ 
    \sqNorm{ Y - \h^\top X } 
} + \frac{1}{\alpha} \; \EE{}{\M}{ \sqNorm{ \EE{}{}{ Y \given Z} - \h^\top \EE{}{}{ X\given Z } 
} }
,
\\
\tag{From \cref{eq:iv-2sls-cmr}.}
&= \Risk{\erm}{\M}{ \h } 
+ 
\frac{1}{\alpha}
\pr*{
    \Risk{\iv}{\M}{ \h }
    -
    \EE{}{}{ \VV{}{}{ Y \given Z} }
} ,
\\
\nonumber
&=
\frac{1}{\alpha} 
\pr*{
    \Risk{\ivla}{\M}{ \h } 
    -
    \EE{}{}{ \VV{}{}{ Y \given Z} }
} .
\end{align}
\end{proof}
}

\myRule

\subsection{Proof of \cref{theorem:ivl-causal-estimation}---Causal estimation with IVL regression}
\label{app:theorem-ivl-causal-estimation}
{
\allowdisplaybreaks
\vspace{\ProofSep}
\ivlCausalEstimation*
\begin{proof}
For $\hHat^{\M}_{\ivla}$, we have from \cref{proposition:ivl-closed-form-solution}
\begin{align*}
\sqNorm{ \hHat^{\M}_{\ivla} - \f }_{\cov{X}{\M}} &= \sqNorm{ 
    \EE{}{}{ X^{\prime} {X^{\prime}}^\top }^{-1}\EE{}{}{ X^{\prime}  {Y^{\prime}}^\top } - \f 
}_{\cov{X}{\M}} .
\end{align*}
Note that we have
\begin{align}
\nonumber
&\Rightarrow \EE{}{}{ X^{\prime}  {Y^{\prime}}^\top } 
\\
\nonumber
&= \EE{}{}{ 
    X^{\prime} \pr*{ 
        a Y + b \EE{}{}{ Y \given Z } 
    }^\top 
} , 
\\
\nonumber
&= \EE{}{}{ 
    X^{\prime} \pr*{ 
        a Y + b \EE{}{}{ 
            \f^\top X + \xi \given Z 
        } 
    }^\top 
} , 
\\
\tag{By definition $Z\indep \xi$.}
&= \EE{}{}{ 
    X^{\prime} \pr*{ 
        a Y + b \f^\top \EE{}{}{ X \given Z } 
    }^\top 
} , 
\\
\nonumber
&= \EE{}{}{ 
    X^{\prime} \pr*{ 
        a \f^\top X + a \xi + b \f^\top \EE{}{}{ X \given Z } 
    }^\top 
} , 
\\
\tag{Substituting in $X^{\prime}\coloneqq a X + b \EE{}{}{ X\given Z }$.}
&= \EE{}{}{
    X^{\prime} \pr*{ 
        \f^\top X^{\prime} + a \xi 
    }^\top 
} , 
\\
\nonumber
&= \EE{}{}{ 
    X^{\prime} {X^{\prime}}^\top \f + a X^{\prime} \xi^\top 
} , 
\\
\nonumber
&= \EE{}{}{ 
    X^{\prime} {X^{\prime}}^\top 
} \f + a \EE{}{}{ X^{\prime} \xi^\top } , 
\\
\tag{$Z\indep \xi$, therefore $\EE{}{}{X^{\prime} \xi^\top} = a\EE{}{}{X \xi^\top}$.}
&= \EE{}{}{ 
    X^{\prime} {X^{\prime}}^\top 
} \f + a^2 \EE{}{}{ X \xi^\top } , 
\\
\label{eq:uiv-a-rse--covar}
&= \EE{}{}{ 
    X^{\prime} {X^{\prime}}^\top 
} \f + \alpha \EE{}{}{ X \xi^\top } , 
\end{align}
We also see that
\begin{align}
\nonumber
&\Rightarrow \EE{}{}{ X^{\prime}  {X^{\prime}}^\top } 
\\
\nonumber
&= \EE{}{}{
    \pr*{ 
        a X + b \EE{}{}{ X\given Z } 
    } \pr*{ 
        a X + b \EE{}{}{ X\given Z } 
    }^\top 
} , 
\\
\tag{Set $\tZ \coloneqq \mathbb{E}[X\mid Z]$ for brevity.}
&= \EE{}{}{ 
    \pr*{ 
        a X + b \tZ 
    } \pr*{ 
        a X + b \tZ 
    }^\top 
} , 
\\
\nonumber
&= a^2 \EE{}{}{ X X^\top } + b^2 \EE{}{}{ \tZ \tZ^\top } + a b \EE{}{}{ X \tZ^\top } + a b \EE{}{}{ \tZ X^\top } , 
\\
\tag{Because $\EE{}{}{ X \tZ^\top } = \cov{\tZ}$.}
&= a^2 \EE{}{}{ X X^\top } + \pr*{ b^2 + 2 a b } \cov{\tZ} , 
\\
\label{eq:uiv-a-rse--xx-covarXX}
&= \alpha \EE{}{}{ X X^\top } + \cov{\tZ} , 
\end{align}
where we substituted in \cref{eq:ab2alpha} in \cref{eq:uiv-a-rse--xx-covarXX}.

Finally, we now have
\begin{align}
\nonumber
&\Rightarrow \sqNorm{ 
    \hHat^{\M}_{\ivla} - \f 
}_{\cov{X}{\M}} 
\\
\nonumber
&= \sqNorm{ 
    \EE{}{}{ X^{\prime} {X^{\prime}}^\top }^{-1} \EE{}{}{ X^{\prime} {Y^{\prime}}^\top } - \f 
}_{\cov{X}{\M}}
, 
\\
\tag{Substituting in \cref{eq:uiv-a-rse--covar}.}
&= \sqNorm{ 
    \EE{}{}{ X^{\prime} {X^{\prime}}^\top }^{-1} \pr*{ 
        \EE{}{}{ X^{\prime} {X^{\prime}}^\top } \f + \alpha\EE{}{}{X \xi^\top} 
    } - \f 
}_{\cov{X}{\M}} 
,
\\
\nonumber
&= \sqNorm{ 
    \f + \alpha \EE{}{}{ X^{\prime} {X^{\prime}}^\top }^{-1} \EE{}{}{X \xi^\top} - \f 
}_{\cov{X}{\M}} 
,
\\
\nonumber
&= \sqNorm{ 
    \alpha \EE{}{}{ X^{\prime} {X^{\prime}}^\top }^{-1} \EE{}{}{X \xi^\top} 
}_{\cov{X}{\M}} 
,
\\
\tag{Substituting in \cref{eq:uiv-a-rse--xx-covarXX}.}
&= \sqNorm{ 
    \alpha \pr*{ 
        \alpha \EE{}{}{X X^\top} + \cov{\tZ} 
    }^{-1} \EE{}{}{X \xi^\top} 
}_{\cov{X}{\M}} 
,
\\
\tag{Using \cref{lemma:simultaneous-diagonalization}.}
&= \sqNorm{ 
    \pr*{ 
        \mS^\top \mS +  
            \frac{1}{\alpha} \mS^\top \mD \mS 
    }^{-1} \EE{}{}{X \xi^\top}  
}_{\mS^\top \mS}
,
\\
\tag{$\mS$ is invertible.}
&= \sqNorm{ \mS^{-1} 
    \pr*{ 
        \In{m}  +  
            \frac{1}{\alpha} \mD 
    }^{-1} \mS^{-\top} \EE{}{}{X \xi^\top}  
}_{\mS^\top \mS}
,
\\
\tag{Switch to $\ell_2$ norm.}
&= \sqNorm{ 
    \pr*{ 
        \In{m} +  
            \frac{1}{\alpha} \mD 
    }^{-1} \mS^{-\top} \EE{}{}{X \xi^\top}  
}
,
\\
\label{eq:ivl-to-erm}
&\leq \sqNorm{
    \mS^{-\top} \EE{}{}{X \xi^\top}  
}
,
\\
\tag{Substituting $\mathbf{I} = \mS \mS^{-1}$.}
&= \sqNorm{
    \mS \mS^{-1} \mS^{-\top} \EE{}{}{X \xi^\top}  
}
,
\\
\tag{Back to weighted norm.}
&= \sqNorm{
    \mS^{-1} \mS^{-\top} \EE{}{}{X \xi^\top}  
}_{\mS^\top \mS}
,
\\
\tag{Substituting $\cov{X}{\M} \coloneqq \EE{}{\M}{X X^\top} = \mS^\top \mS$.}
&= \sqNorm{
    \EE{}{}{X X^\top}^{-1} \EE{}{}{X \xi^\top}  
}_{\cov{X}{\M}}
,
\\
\tag{Adding and subtracting $\f$.}
&= \sqNorm{
    \f + \EE{}{}{X X^\top}^{-1} \EE{}{}{X \xi^\top}
    - \f
}_{\cov{X}{\M}}
,
\\
\tag{Substitute $\mathbf{I} = \EE{}{}{X X^\top}^{-1}\EE{}{}{X X^\top}$.}
&= \sqNorm{
    \EE{}{}{X X^\top}^{-1} \pr*{ \EE{}{}{X X^\top}\f + \EE{}{}{X \xi^\top} }
    - \f
}_{\cov{X}{\M}}
,
\\
\tag{Linearity of expectation.}
&= \sqNorm{
    \EE{}{}{X X^\top}^{-1} \EE{}{}{ X \pr*{ \f^\top X + \xi}^\top }
    - \f
}_{\cov{X}{\M}}
,
\\
\tag{Substituting $Y = \f^\top X + \xi$.}
&= \sqNorm{
    \EE{}{}{X X^\top}^{-1} \EE{}{}{ X Y^\top }
    - \f
}_{\cov{X}{\M}}
,
\\
\tag{Closed form ERM solution.}
&= \sqNorm{
    \hHat_{\erm}^{\M}
    - \f
}_{\cov{X}{\M}} ,
\end{align}
where inequality \cref{eq:ivl-to-erm} holds because $\mD$ is non-negative diagonal. Furthermore, inequality \cref{eq:ivl-to-erm} only holds with equality iff $\mS^{-\top}\EE{}{}{X \xi^\top}$ is in the kernel of $\mD$. Or equivalently, iff $\EE{}{}{X \xi^\top}$ is in the kernel of $\mS^\top \mD \mS = \cov{\tZ}{}$, which from \cref{lemma:gaussian-conditional-orthogonality} is true iff
\begin{align*}
    \EE{}{\M}{X \given Z} \quad \perp \quad \EE{}{\M}{X \given \xi}
    \qquad
    \text{a.s.}
\end{align*}
\end{proof}
}

\myRule

\subsection{Proof of \cref{theorem:da-causal-estimation}---Causal estimation with DA+ERM}
\label{app:theorem-da-causal-estimation}
{
\allowdisplaybreaks
\vspace{\ProofSep}
\daCausalEstimation*
\begin{proof}
We have
\begin{align}
\nonumber
&\Rightarrow \norm*{ 
    \hHat^{\A}_{\da[G]\erm} - \f 
}_{ \cov{X}{\A} }
\\
\nonumber
&= \norm*{ 
    \EE{}{}{ \pr*{ GX } \pr*{ GX }^\top }^{-1} 
    \EE{}{}{ \pr*{ GX }  Y^\top } - \f  
}_{ \cov{X}{\A} } ,
\\
\tag{Structural eq. of $Y$.}
&= \norm*{ 
    \EE{}{}{ \pr*{ GX } \pr*{ GX }^\top }^{-1} 
    \EE{}{}{ 
        \pr*{ GX } \pr*{\f^\top X + \xi }^\top 
    } - \f 
}_{ \cov{X}{\A} } ,
\\
\tag{Using $\sG$-invariance of $\f$.}
&= \norm*{ 
    \EE{}{}{ \pr*{ GX } \pr*{ GX }^\top }^{-1} 
    \EE{}{}{ 
        \pr*{ GX } \pr*{\f^\top \pr*{ GX } + \xi }^\top 
    } - \f 
}_{ \cov{X}{\A} } ,
\\
\nonumber
&= \norm*{ 
    \pr*{
        \f + \EE{}{}{ \pr*{ GX } \pr*{ GX }^\top }^{-1} 
        \EE{}{}{ \pr*{ GX }  \xi^\top } 
    } - \f 
}_{ \cov{X}{\A} } ,
\\
\nonumber
&= \norm*{ 
    \EE{}{}{
    \pr*{ GX } \pr*{ GX }^\top }^{-1} 
    \EE{}{}{ \pr*{ GX }  \xi^\top } 
}_{ \cov{X}{\A} } ,
\\
\tag{Let $\tG\coloneqq \EE{}{}{GX\given G} = \gamma \cdot \K^\top G$.}
&= \norm*{ 
    \EE{}{}{ \pr*{ X + \tG } \pr*{ X + \tG }^\top }^{-1} 
    \EE{}{}{ \pr*{ X + \tG }  \xi^\top } 
}_{ \cov{X}{\A} } ,
\\
\tag{Using $\tG\indep X, \xi$.}
&= \norm*{ 
    \pr*{ \EE{}{}{ X X^\top } + \EE{}{}{ \tG\tG^\top } }^{-1} 
    \EE{}{}{ X \xi^\top } 
}_{ \cov{X}{\A} } ,
\\
\tag{\cref{lemma:simultaneous-diagonalization}.}
&= \norm*{ 
    \pr*{ \mS^\top \mS + \mS^\top \mD \mS }^{-1} 
    \EE{}{}{ X \xi^\top } 
}_{ \mS^\top \mS } ,
\\
\tag{$\mS, \mS^\top$ invertible.}
&= \norm*{ 
    \mS^{-1} \pr*{ \In{m} + \mD }^{-1} \mS^{-\top} 
    \EE{}{}{ X \xi^\top } 
}_{ \mS^\top \mS } ,
\\
\tag{Switch to $\ell_2$ norm.}
&= \norm*{ 
    \mS \mS^{-1} \pr*{ \In{m} + \mD }^{-1} \mS^{-\top} 
    \EE{}{}{ X \xi^\top } 
} ,
\\
\nonumber
&= \norm*{ 
    \pr*{ \In{m} + \mD }^{-1} \mS^{-\top} 
    \EE{}{}{ X \xi^\top } 
} ,
\\
\label{eq:da-to-erm}
&\leq \norm*{ 
    \mS^{-\top} 
    \EE{}{}{ X \xi^\top } 
} ,
\\
\tag{Substitute in $\In{m} = \mS \mS^{-1}$.}
&= \norm*{ 
    \mS \mS^{-1} \mS^{-\top} 
    \EE{}{}{ X \xi^\top } 
} ,
\\
\tag{Back to weighted norm.}
&= \norm*{ 
    \mS^{-1} \mS^{-\top} 
    \EE{}{}{ X \xi^\top } 
}_{\mS^\top \mS} ,
\\
\tag{Substitute in $\cov{X}{\A}\coloneqq \EE{}{\A}{XX^\top} = \mS^\top \mS$.}
&= \norm*{ 
    \EE{}{}{XX^\top}^{-1} 
    \EE{}{}{ X \xi^\top } 
}_{\cov{X}{\A}} ,
\\
\tag{Add and subtract $\f$.}
&= \norm*{ 
    \f + \EE{}{}{XX^\top}^{-1} 
    \EE{}{}{ X \xi^\top } - \f
}_{\cov{X}{\A}} ,
\\
\tag{Use $\In{m} = \EE{}{}{XX^\top}^{-1}\EE{}{}{XX^\top}$.}
&= \norm*{ 
    \EE{}{}{XX^\top}^{-1} 
    \pr*{ \EE{}{}{XX^\top} \f + \EE{}{}{ X \xi^\top } } - \f
}_{\cov{X}{\A}} ,
\\
\tag{Linearity of expectation.}
&= \norm*{ 
    \EE{}{}{XX^\top}^{-1} 
    \EE{}{}{X \pr*{ \f^\top X  + \xi }^\top } - \f
}_{\cov{X}{\A}} ,
\\
\tag{Structural eq. of $Y$.}
&= \norm*{ 
    \EE{}{}{XX^\top}^{-1} 
    \EE{}{}{X Y^\top } - \f
}_{\cov{X}{\A}} ,
\\
\tag{ERM closed form solution.}
&= \norm*{ 
    \hHat_{\erm}^{\A} - \f
}_{\cov{X}{\A}} ,
\end{align}
where inequality \cref{eq:da-to-erm} holds because $\mD$ is non-negative diagonal. Furthermore, inequality \cref{eq:da-to-erm} only holds with equality iff $\mS^{-\top}\EE{}{}{X \xi^\top}$ is in the kernel of $\mD$. Or equivalently, iff $\EE{}{}{X \xi^\top}$ is in the kernel of $\mS^\top \mD \mS = \cov{\tG}{}$, which from \cref{lemma:gaussian-conditional-orthogonality} is true iff $\EE{}{\A}{GX \given G}\; \perp\; \EE{}{\A}{X \given \xi}$ a.s.
\end{proof}
}

\myRule



\subsection{Miscellaneous supporting lemmas}
\label{app:miscellaneous-supporting-results}
{
\allowdisplaybreaks
\begin{lemma}[Gaussian conditional orthogonality lemma]
\label{lemma:gaussian-conditional-orthogonality}
Let $X, Y, Z \in \mathbb{R}^n$ be zero-mean jointly Gaussian random vectors with covariance matrices 
$\cov{X}{} = \mathbb{E}[X X^\top]$, $\cov{Z}{} = \mathbb{E}[Z Z^\top]$, 
and cross-covariance $\cov{Y,Z}{} = \mathbb{E}[Y Z^\top]$. Define the conditional expectation
\begin{align*}
\mathbb{E}[Y \mid Z] \coloneqq \pr*{ \EE{}{}{Z Z^\top}^{-1} \EE{}{}{Z Y^\top} }^\top Z = \cov{Y,Z}{} \cov{Z}{-1} Z.
\end{align*}
Then the following are equivalent:
\begin{align*}
X \perp \mathbb{E}[Y \mid Z] = 0 \quad \text{a.s.} 
\qquad
\iff 
\qquad
\cov{X}{} \cov{Y,Z}{} = {\0{}}.
\end{align*}
\end{lemma}
\begin{proof}
Since $X, Y, Z$ are jointly Gaussian, $\mathbb{E}[Y \mid Z] = \mathbf{M} Z$ with $\mathbf{M} \coloneqq \cov{Y,Z}{} \cov{Z}{-1}$.
The scalar random variable
\begin{align*}
S \coloneqq X^\top \mathbb{E}[Y \mid Z] = X^\top \mathbf{M} Z
\end{align*}
is Gaussian with mean zero. Hence,
\begin{align*}
S = 0 \quad \text{a.s.}
\qquad
\iff
\qquad
\VV{}{}{S} = 0.
\end{align*}
Compute the variance:
\begin{align*}
\VV{}{}{S} = \EE{}{}{S^2} = \EE{}{}{ (X^\top \mathbf{M} Z)^2 } = \EE{}{}{Z^\top \mathbf{M}^\top X X^\top \mathbf{M} Z}.
\end{align*}
Using independence and zero-mean assumptions,
\begin{align*}
\VV{}{}{S} = \tr{ \mathbf{M}^\top \cov{X}{} \mathbf{M} \cov{Z}{} }.
\end{align*}
Since covariance matrices are positive semidefinite, $\VV{}{}{S} = 0$ iff
\begin{align*}
\cov{X}{1/2} \mathbf{M} \cov{Z}{1/2} = {\0{}} \implies \cov{X}{} \mathbf{M} \cov{Z}{} = {\0{}} .
\end{align*}
Substituting $\mathbf{M} = \cov{Y,Z}{} \cov{Z}{-1}$ gives
\begin{align*}
\cov{X}{} \cov{Y,Z}{} = {\0{}},
\end{align*}
completing the proof.
\end{proof}

\phantom{space}

\phantom{space}

\phantom{space}

\begin{lemma}[SPD and PSD simultaneous diagonalization via congruence]
\label{lemma:simultaneous-diagonalization}
For any $n\times n$ matrices $\mathbf{A} \succ \mathbf{0}$, $\mathbf{B} \succcurlyeq \mathbf{0}$, there exists an invertible $\mathbf{S} \in \mathbb{R}^{n \times n}$ and non-negative diagonal $\mathbf{D} \in \mathbb{R}^{n \times n}$ such that
\begin{align*}
\mathbf{A} = \mathbf{S}^\top \mathbf{S},
&&
\mathbf{B} = \mathbf{S}^\top \mathbf{D} \mathbf{S} .    
\end{align*}
\end{lemma}
\begin{proof}
This is similar to Theorem 7.6.4 in \cite[p. 465]{Horn_Johnson_1985} for two SPD matrices. We proceed similarly; Since $\mathbf{A}$ is SPD, it admits a unique SPD square root $\mathbf{A}^{1/2}$. Define
\begin{align*}
\mathbf{C} := \mathbf{A}^{-1/2} \mathbf{B} \mathbf{A}^{-1/2},
\end{align*}
which is SPD.
By the spectral theorem, there exists an orthogonal matrix $\mathbf{U}$ such that
\begin{align*}
\mathbf{C} = \mathbf{U}^\top \mathbf{D} \mathbf{U},
\end{align*}
where $\mathbf{D}$ is diagonal with non-negative entries (the eigenvalues of $\mathbf{C}$).
Set
\begin{align*}
\mathbf{S} := \mathbf{U} \mathbf{A}^{1/2}.
\end{align*}
Then
\begin{align*}
\mathbf{S}^\top \mathbf{S} = \mathbf{A}^{1/2} \mathbf{U}^\top \mathbf{U} \mathbf{A}^{1/2} = \mathbf{A}^{1/2} \mathbf{I} \mathbf{A}^{1/2} = \mathbf{A},
\end{align*}
and
\begin{align*}
\mathbf{S}^\top \mathbf{D} \mathbf{S} = \mathbf{A}^{1/2} \mathbf{U}^\top \mathbf{D} \mathbf{U} \mathbf{A}^{1/2} = \mathbf{A}^{1/2} \mathbf{C} \mathbf{A}^{1/2} = \mathbf{B}.
\end{align*}
Since $\mathbf{A}^{1/2}$ and $\mathbf{U}$ are invertible, $\mathbf{S}$ is invertible, completing the proof.
\end{proof}


\begin{lemma}[solvability of simultaneous SEM]
\label{lemma:sem-solvability}
The SEM $\M$ in \cref{example:ivl} is solvable iff $ \f^\top \vtau^\top \neq 1 $, in which case the following solution defines the reduced form of the SEM.
\begin{align*}
    \begin{bmatrix}
        X\\
        Y
    \end{bmatrix}
    &=
    \begin{bmatrix}
        \In{m} & -\vtau^\top\\
        -\f^\top & 1
    \end{bmatrix}^{-1}
     \pr*{
     \begin{bmatrix}
        \K^\top\\
        \0{1}{k}
    \end{bmatrix}
    Z
    +
     \begin{bmatrix}
        \E^\top \\
        \veps^\top
    \end{bmatrix}
    C
     + \sigma \cdot
     \begin{bmatrix}
        N_X\\
        N_Y
    \end{bmatrix}
    }
    ,
\end{align*}
Similarly, SEM $\A$ in \cref{example:da} solves for $ \f^\top \vtau^\top \neq \kappa^{-1} $.
\end{lemma}
\begin{proof}
We re-state the SEM $\M$ in the following block form
\begin{align*}
    \begin{bmatrix}
        X\\
        Y
    \end{bmatrix}
    =&
    \begin{bmatrix}
        \0{m}{m} & \vtau^\top\\
        \f^\top & \0{1}{1}
    \end{bmatrix}
    \begin{bmatrix}
        X\\
        Y
    \end{bmatrix}
    +
    \begin{bmatrix}
        \K^\top\\
        \0{1}{k}
    \end{bmatrix}
    Z
     +
     \begin{bmatrix}
        \E^\top \\
        \veps^\top
    \end{bmatrix}
    C
     + \sigma \cdot
     \begin{bmatrix}
        N_X\\
        N_Y
    \end{bmatrix}
    ,
    \\
    \Rightarrow& \begin{bmatrix}
        \In{m} & -\vtau^\top\\
        -\f^\top & 1
    \end{bmatrix}
    \cdot\, \begin{bmatrix}
        X\\
        Y
    \end{bmatrix}
    =
    \begin{bmatrix}
        \K^\top\\
        \0{1}{k}
    \end{bmatrix}
    Z
    +
     \begin{bmatrix}
        \E^\top \\
        \veps^\top
    \end{bmatrix}
    C
     + \sigma \cdot
     \begin{bmatrix}
        N_X\\
        N_Y
    \end{bmatrix}
\end{align*}
solving for $(X, Y)$ involves inverting the block matrix on the LHS. The result immediately follows from Proposition~2.8.7 in \cite[p.~108]{matrix-math}, via the Schur complement formula for block matrix inversion.
\end{proof}

}





\end{document}